\pdfoutput=1
\documentclass[twoside,11pt]{article}

\usepackage{blindtext}

\usepackage{jmlr2e}
\usepackage{subcaption}

\usepackage{lastpage}
\jmlrheading{24}{2023}{1-\pageref{LastPage}}{5/22; Revised
4/23}{7/23}{22-0495}{Marcel Wien\"{o}bst, Max Bannach and Maciej Li\'{s}kiewicz}
\ShortHeadings{Counting Markov Equivalent DAGs with Applications}{Wien\"{o}bst, Bannach and Li\'{s}kiewicz}

\firstpageno{1}

\RequirePackage{amsmath}
\RequirePackage{amssymb}
\usepackage{physics}

\newcommand{\qedhere}{\tag*{\BlackBox}}
\newcommand{\qedheretext}{\hfill\BlackBox}
\newcommand{\noqed}{\renewcommand{\BlackBox}{}\leavevmode\vspace{-\medskipamount}\vspace{-2mm}}

\DeclareMathOperator*{\argmax}{arg\,max}

\usepackage{paralist}

\usepackage[linesnumbered]{algorithm2e}
\SetAlCapSty{}
\newcommand{\sharpP}{\text{\rm\#}P}

\newtheorem{problem}[theorem]{Problem}
\newtheorem{claim}[theorem]{Claim}

\usepackage{booktabs}

\newenvironment{parameterizedproblem}%
{%
  \leavevmode\nobreak\par
  \begin{list}%
    {}%
    {%
      \def\labelstyle{\itshape}
      \setlength{\topsep}{0pt}%
      \settowidth{\labelwidth}{\labelstyle Parameter:}%
      \setlength{\leftmargin}{\labelwidth}%
      \addtolength{\leftmargin}{\labelsep}%
      \setlength{\itemsep}{0pt}%
      \setlength{\parsep}{0pt}%
    }%
      \def\instance{\item[\labelstyle Instance:]}%
      \def\result{\item[\labelstyle Result:]}%
    }%
    {%
  \end{list}%
}

\def\exampleqed{\hfill$\diamond$}

\def\bouquet{\mathcal{B}}
\def\cliques{\Pi}
\def\separators{\Delta}

\def\chordalcomps{\mathcal{C}}
\def\cI{\mathcal I}
\def\cE{\mathcal E}
\def\cR{\mathcal R}

\newcommand{\argmin}{\mathop{\mathrm{arg\,min}}}

\def\balap#1{\hbox to 0pt{\scriptsize$#1$ \hss}}

\def\hamo{\text{\rm\#\kern-0.5ptAMO}}
\def\hext{\text{\rm\#\kern-0.5ptEXT}}
\def\hto{\text{\rm\#\kern-0.5ptTO}}

\def\cC{\mathcal C}

\def\cpp{{C\nolinebreak[4]\hspace{-.05em}\raisebox{.4ex}{\tiny\bf ++}}}

\usepackage{tikz}
\usetikzlibrary{graphs, arrows.meta, shapes.misc, backgrounds, decorations.pathmorphing}

\definecolor{ba.yellow}{RGB}{252,190,18}
\definecolor{ba.gray}{RGB}{153,153,156}
\definecolor{ba.blue}{RGB}{6,123,164}
\definecolor{ba.red}{RGB}{213,96,98}
\definecolor{ba.orange}{RGB}{233,116,81}
\definecolor{ba.pine}{RGB}{67,154,134}
\definecolor{ba.green}{RGB}{196,247,161}
\definecolor{ba.violet}{RGB}{88, 53, 94}

\tikzset{
  pico/.style = {
    every node/.style = {
      draw,
      circle,
      semithick,
      inner sep = 0pt,
      minimum width = 0.7ex,
      fill = white
    },
    semithick
  },
  edge/.style = {
    semithick
  },
  arc/.style = {
    edge,
    ->,
    >={[round,sep]Stealth}
  },
}
\tikzset{
  axis/.style = {
    semithick,
    ->,
    >={[round,sep]Stealth},
  },
  tick/.style = {
    thin,
    font=\small
  },
  timeout/.style = {
    semithick,
    densely dashed,
    color = ba.red,
    font=\small,
  },
  mean_dot/.style = {
    draw, fill,
    circle,
    inner sep = 0pt,
    minimum width = 1mm
  }
}

\def\rhombus{%
  \node (1) at (-1,0) {$1$};%
  \node (2) at (0,1)  {$2$};%
  \node (3) at (0,-1) {$3$};%
  \node (4) at (1,0)  {$4$};%
}

\def\marcelgraph{%
  \node (1) at (0,0)  {$1$};
  \node (2) at (1,0)  {$2$};
  \node (3) at (2,0)  {$3$};
  \node (4) at (0,-1) {$4$};
  \node (5) at (1,-1) {$5$};
  \node (6) at (2,-1) {$6$};
  \graph[use existing nodes, edges = {edge}] {
    1 -- 2 -- 3 --[bend right] 1;
    4 -- 5 -- 6;
    2 -- {4,5,6};
    3 -- {4,5,6};
  };
}

\renewcommand\rightarrow[1][1.4em]{\tikz[baseline=-0.5ex, shorten <=2pt, shorten >=2pt] \draw[-{Stealth[round,sep]}] (0,0) -- (#1,0);}
\renewcommand\leftarrow[1][1.4em]{\tikz[baseline=-0.5ex, shorten  <=2pt, shorten >=2pt] \draw[{Stealth[round,sep]}-] (0,0) -- (#1,0);}

\begin{document}

\title{Polynomial-Time Algorithms for Counting and Sampling Markov Equivalent DAGs
with Applications}

\author{\name Marcel Wien\"{o}bst \email m.wienoebst@uni-luebeck.de \\
       \addr Institute for Theoretical Computer Science\\
       University of L\"{u}beck\\
        23538  L\"{u}beck, Germany
       \AND
       \name Max Bannach \email max.bannach@esa.int \\
       \addr Advanced Concepts Team\\
       European Space Agency\\
       2201 AZ Noordwijk, The Netherlands
       \AND
       \name Maciej Li\'{s}kiewicz \email maciej.liskiewicz@uni-luebeck.de \\
         \addr Institute for Theoretical Computer Science\\
       University of L\"{u}beck\\
        23538  L\"{u}beck, Germany}

\editor{Vanessa Didelez}

\maketitle

\begin{abstract}%
Counting and sampling directed acyclic graphs
from a Markov equivalence class are fundamental tasks in graphical causal 
analysis. In this paper we show that these tasks can be performed in 
polynomial time, solving a long-standing open problem in this 
area. Our algorithms are effective and easily implementable.
As we show in experiments, these breakthroughs make
thought-to-be-infeasible strategies in
active learning of causal structures and causal effect identification
with regard to a Markov equivalence class practically applicable.
\end{abstract}


\section{Introduction}
Graphical modeling plays a key role in causal theory, allowing to
express complex causal phenomena in an elegant, mathematically-sound way.
One of the most popular graphical models are directed acyclic graphs
(DAGs), which represent direct causal influences
between random variables by directed
edges~\citep{spirtes2000causation,pearl2009causality,koller2009probabilistic}.
They are commonly used in empirical sciences to discover and understand 
causal effects. However, in practice, the underlying DAG is often unknown
and cannot be identified unambiguously from observational data.
Instead, the statistical properties of the data are shared by a
number of different DAGs, which constitute a Markov equivalence class
(MEC, for short). Therefore, these DAGs are indistinguishable on the
basis of observations alone
\citep{verma1990equivalence,verma1992algorithm,heckerman1995learning}.

It is of great importance to investigate model learning and to analyze causal 
phenomena using MECs directly rather than the DAGs themselves. 
Consequently, Markov equivalence
classes of DAGs constitute a central part of causal discovery and inference.
Our work contributes to this line of research  by providing 
the first polynomial-time algorithms for \emph{counting} and 
for \emph{uniform sampling} Markov 
equivalent DAGs \textendash\ important primitives in both theory and practice. 

Finding the graphical criterion for two DAGs to be Markov equivalent
\citep{verma1990equivalence} and providing the graph-theoretic 
characterization of MECs as CPDAGs
\citep{Andersson1997} mark key turning  points in this research direction. 
In particular, they have contributed to the progress of computational 
methods in this area. Important advantages of CPDAGs
are demonstrated by algorithms that learn causal structures from observational data
\citep{verma1992algorithm,Meek1995,meek1997graphical,spirtes2000causation,chickering2002learning,chickering2002optimal};
and that analyze causality based on a given MEC, rather than a single DAG
\citep{maathuis2009estimating,van2016separators,perkovic2017complete}.

A key characteristic of an MEC 
is its size, i.e., 
the number of DAGs in the class. It indicates uncertainty 
of the causal model inferred from observational data and
it serves as an indicator for the possibility of recovering true causal effects.
Efficient algorithms for computing the size of an MEC are necessary,
whenever researchers aim to quantify or even reduce the uncertainty
present, which is the case particularly in causal effect identification over MECs~\citep{maathuis2009estimating} as
well as active intervention design with the aim to recover the true
DAG with as few experiments as possible~\citep{He2008,hauser2012characterization,shanmugam2015learning,ghassami2018budgeted,Ghassami2019}.

\begin{table}
  \caption{An overview of the algorithmic improvements in the
    computation of the size of Markov equivalence classes. Here,
$n$ denotes the number of vertices, $k$ indicates the size of the largest undirected clique in the graph,
while $d$ is the maximal number of undirected neighbors of a vertex. The first
polynomial-time algorithm given in~\citep{wienobst2021counting} is
presented, applied and extended in this work. The unexplained abbreviations in the table should be read as follows: DP -- Dynamic Programming;
Dom.\ --   Dominating.
}
\label{table:complexities}
  \begin{tabular}{llll} \toprule 
                                 & Approach                     & Complexity                     \\ \midrule
    \citet{Meek1995}             & Exhaustive search  & $\mathcal{O}(n!)$              \\
    \citet{He2015,he2016formulas}     & Root-Picking (RP)            & $\mathcal{O}(n!)$              \\
    \citet{Talvitie2019}       & RP + Memoization (MM)       & $\mathcal{O}(2^n \cdot n^4)$   \\
    \citet{Talvitie2019}       & DP on Clique-Tree            &
    $\mathcal{O}(k!2^kk^2n)$       \\
    \citet{Ghassami2019}        & RP + MM + Clique-Tree       & $\mathcal{O}(2^n \cdot n^4)$   \\
    \citet{Ganian2020,ganian2022efficient}           & RP + MM + Dom.\ Vertex      & $\mathcal{O}(2^n \cdot n^4)$   \\ 
    \citet{Teshnizi20}         & Intervention Design + MM           & $\mathcal{O}(2^{n+d}(nd + d^3))$ \\ 
    \citet{wienobst2021counting}; This work                    & Clique-Picking + MM         & $\mathcal{O}(n^4)$     \\
    \bottomrule
  \end{tabular}
\end{table}

The first algorithmic approaches to counting the number of Markov equivalent
DAGs date back to the work of~\citet{Meek1995}. Therein, it is already observed
that it suffices to consider the \emph{undirected
components}\footnote{These are the connected components of the graph
after removing all directed edges.} of the CPDAG separately.
Even when using exhaustive search for these components, this can already lead to
reasonably fast algorithms when the CPDAG has few undirected edges,
respectively small undirected components. In the worst-case of large undirected
components, however, this approach requires exponential-time.
Starting with the work of \citet{He2015}, the problem has been readdressed with
better and better worst-case bounds of the run-time (see
Table~\ref{table:complexities} for an overview). Particularly, the
\emph{root-picking} approach has been successively refined in multiple works  
\citep{He2015,he2016formulas,Talvitie2019,Ghassami2019,Ganian2020}
bringing the time complexity down to $\mathcal{O}(2^n \cdot
n^4)$, where $n$ denotes the number of  vertices,
which still amounts to exponential worst-case time.

In this paper, we present the culmination of these algorithmic
efforts, the first algorithm with polynomial time complexity
$\mathcal{O}(n^4)$ for counting the number of Markov equivalent DAGs.
Preliminary results of this work have been presented at the AAAI
Conference on Artificial Intelligence, AAAI~2021 \citep{wienobst2021counting}.
Therein, the polynomial-time algorithm has been proposed and an implementation
has been given, which outperforms the previous methods.  
This algorithm certifies that the problem can be solved efficiently -- in
theory and practice. It also implies a polynomial-time algorithm for
uniformly sampling of DAGs from an MEC. A major focus of this work is
on the further consequences of this
breakthrough and, in particular, to illustrate its application to
downstream tasks in causal discovery and inference, which have been
underexplored thus far. 
To make such applications possible, we
analyze the counting problem particularly for \emph{interventional}
MECs and show how the algorithmic achievements translate to this setting,
which plays an important role in many fundamental tasks in causality, two of which we discuss in depth. Additionally, we provide the
first practical implementation of a uniform sampling algorithm for the
members of an
MEC by extending and improving the results from~\citep{wienobst2021counting}.

%

\subsection*{Our Contributions} 
The main achievement of our paper is Algorithm~\ref{alg:main}, which is the 
first poly\-no\-mial-time algorithm for counting 
Markov equivalent DAGs:

\begin{theorem}[Main]\label{thm:main}
  For an input CPDAG $C$, Algorithm~\ref{alg:main} returns the size of
  the MEC represented by $C$ in polynomial time in the size of the graph.
\end{theorem}

\renewcommand{\figurename}{Algorithm}
\begin{figure}
  \begin{subfigure}{0.59\textwidth}
    \resizebox{\textwidth}{!}{\begin{algorithm}[H] 
      \SetKwInOut{Input}{input}\SetKwInOut{Output}{output}
      \DontPrintSemicolon
      \Input{A CPDAG $C = (V,E)$.}
      \Output{Size of the MEC represented by $C$.}
      $\text{size} \gets 1$ \;
      \ForEach{\text{undirected component} $G$ of $C$}{
        $\text{size} \gets \text{size} \times \text{Clique-Picking}(G)$ \;
      }
      \KwRet $\mathrm{size}$\;
    \end{algorithm}}
  \end{subfigure}
  \begin{subfigure}{0.39\textwidth}
    \resizebox{\textwidth}{!}{\begin{tikzpicture}[baseline=-2.65cm]
      \node (1) at (-0.5,-1.3) {\footnotesize$1$};
      \node (2) at (1,-1.3) {\footnotesize$2$};
      \node (3) at (2.5,-1.3) {\footnotesize$3$};
      \node (4) at (-0.5,-2.3) {\footnotesize$4$};
      \node (5) at (1,-2.3) {\footnotesize$5$};
      \node (6) at (2.5,-2.3) {\footnotesize$6$};
      \graph[use existing nodes, edges = {edge}] {
        1 -- {2,4};
        3 -- 6;
      };
      \graph[use existing nodes, edges = {arc}] {
        {2,4,6} -- 5;
      };

      \node (u) at (1.25,-2.9) {\footnotesize $\text{Clique-Picking}(4-1-2)$};
      \node (l) at (0.79,-3.35) {\footnotesize $\times \; \text{Clique-Picking}(3-6)$};
      \draw (-1.25,-3.6) -- (3.5,-3.6);
      \node (r) at (1, -3.85) {\footnotesize $3 \; \times \; 2 \; = \; 6$};
    \end{tikzpicture}}
  \end{subfigure}
  \caption{The algorithm to compute the size of the MEC represented by
    a CPDAG $C$. Clique-Picking is presented as Algorithm~\ref{alg:cliquepicking}
    in Section~\ref{sec:cliquepicking}. On the right, we show a simple
    example computation. The undirected component $4 - 1 - 2$
    has three possible orientations as $4 \rightarrow 1 \leftarrow 2$
    is disallowed (more on this in the subsequent section). In total,
    there are 6 DAGs in the MEC.}
  \label{alg:main}
\end{figure}
\renewcommand{\figurename}{Figure} 

\setcounter{algocf}{1}
The key component of this algorithm, 
which we coin \emph{Clique-Picking}, computes the number of Markov equivalent
DAGs separately for each undirected component of the CPDAG. In order to do so, it utilizes
the fact that these components are \emph{chordal} by exploring their
clique-tree representation and by evaluating a non-trivial
recursive counting function.\footnote{All technical terms are
  formally introduced in the subsequent section.}

With the recursion used in Clique-Picking, the
problem of uniformly sampling a DAG from a Markov equivalence class can
be solved in polynomial-time as well, by running an
adapted version of Clique-Picking as a preprocessing step, after which sampling is possible in
linear-time.\footnote{The method we propose in this work is an
  improvement over the original algorithm given
  in~\citep{wienobst2021counting}. The algorithm is simpler and the
preprocessing step asymptotically more efficient by a factor $n$.}

\begin{theorem}\label{thm:main:sampling}
  There is an algorithm that, given a CPDAG $C$,
  uniformly samples a DAG in the MEC represented by $C$
  in expected linear time in the size of the graph after an initial
  polynomial-time preprocessing setup.
\end{theorem}

Both algorithms are easy to implement and very fast in practice
-- outperforming previous approaches by a large margin in experimental
evaluations~\citep{wienobst2021counting}. We complement our theoretical findings with optimized implementations
in \cpp{} and Julia to facilitate the application to real-world problems.

Particularly to this end, a special focus in this paper is devoted to two
applications of the Clique-Picking algorithm: Active
learning of causal DAGs~\citep{He2008} and the global-IDA algorithm for causal effect
identification~\citep{maathuis2009estimating}.
We argue that it is
desirable in both cases to compute the size of an MEC (or more
precisely, an \emph{interventional} MEC), a task which, as we show,
can be efficiently solved using Clique-Plicking as a
subroutine.\footnote{Both applications have been considered
  before (see~\cite{Ghassami2019} and other works), but never from the point of view of
  interventional MECs. Instead the more general setting of counting
  with background knowledge has been used, which we show to be computationally
  intractable. In particular,
  utilizing interventional MECs in the global-IDA algorithm connects
  two
disjoint subfields of causality.}
Previously,
the task of counting the number of Markov equivalent DAGs was \emph{avoided}
by researchers in these fields due to its apparent intractability,
leading to the prevalent use
of heuristics at the cost of accuracy.
We demonstrate that through the new methods developed in this paper, such
heuristics are not needed anymore, as the size of an MEC can now be
computed fast in practice and empirically validate this
claim.\footnote{The code for the experiments can be accessed at the address  \url{https://github.com/mwien/counting-with-applications}.}

Finally, we complete the complexity-theoretical study of the counting
problem by investigating the more general problem of counting the
number of DAGs with additional background knowledge. We show that this problem is intractable under common
complexity-theoretical assumptions by connecting it to classical counting problems.

\paragraph*{}
The paper is split into roughly two parts: In the first less-technical
half, we formally introduce the problem of computing the size of an MEC and
give the well-known reduction to a purely graph-theoretical problem
(Section~\ref{sec:preliminaries}), whose solution (Clique-Picking)
will be presented later in Section~\ref{sec:cliquepicking}.
Before this, we derive  and analyze in Section~\ref{sec:applications}
the two applications mentioned above. For these,
Clique-Picking may be viewed as a black-box algorithm, which allows us
to defer the technically demanding introduction and analysis 
of this algorithm to Section~\ref{sec:cliquepicking} in the second half of this paper. Building on these techniques, in Sec.~\ref{sec:uniform:sampling}, we give the polynomial-time algorithm for uniform sampling
of Markov equivalent DAGs and afterwards generalize the results
to the setting of additional background knowledge in
Section~\ref{sec:count:background}. In the subsequent
Section~\ref{sec:conclusion} we give conclusions. To improve readability, we move some 
technical proofs to Section~\ref{sec:missing:proofs}.

\section{Preliminaries}\label{sec:preliminaries}
A graph $G = (V_G, E_G)$ consists of a set of vertices $V_G$ and a set
of edges $E_G \subseteq V_G \times V_G$. Throughout this paper,
whenever the graph $G$ is clear from the context, we will drop the
subscript in this and analogous notations.
An edge $u-v$ is undirected if $(u, v), (v, u) \in E_G$ and directed
$u \rightarrow v$ if $(u,v) \in E_G$ and $(v,u) \not\in E_G$. In the
latter case $u$ is called a parent of $v$. Graphs which contain
undirected and directed edges are called partially directed. Directed
acyclic graphs (DAGs) contain only directed edges and no directed cycle. 
We refer to the neighbors of a vertex~$u$ in $G$ as $N_G(u)$. A
\emph{clique} is a set of pairwise adjacent vertices. We denote the
induced subgraph of $G$ on a set $C\subseteq V$ by $G[C]$. The undirected components of a partially directed
graph $G$ are the connected components in the
undirected graph one obtains after removing all directed
edges from $G$. 

In causality theory, DAGs are used as mathematical models to represent causal 
relations \citep{pearl2009causality}.
For a DAG $D = (V,E)$, the vertices $V=\{1,\ldots, n\}$ represent the random variables
$X=(X_1,\ldots,X_n)$.  A distribution $f$ over $X$ is \emph{Markov} to $D$ if it factorizes as 
$$f(x_1,\ldots, x_n)=\prod_{i\in V} f(x_i \mid \textit{pa}_{i}(D)),$$ 
where $\textit{pa}_i (D)$ denotes the values $x_{j_1},\ldots,x_{j_k}$
of the parents $\textit{Pa}_i (D)=\{{j_1},\ldots,{j_k}\}$ 
 of vertex $i$ in $D$. Two DAGs $D_1$ and $D_2$ are 
\emph{Markov equivalent} 
if for any positive distribution $f$, $f$ is Markov to $D_1$ 
if, and only, if it is Markov to $D_2$.

Due to \citet{verma1990equivalence}, we know a graphical criterion to decide
this relation: Two DAGs are Markov equivalent if, and only if,
they have the same \emph{skeleton} and the same \emph{v-structures}.
The skeleton of a (partially) directed graph $G$ is the undirected 
graph that results from ignoring edge directions. A
v-structure in a (partially) directed graph $G$ is an ordered
triple of vertices $(a,b,c)$ which induce the subgraph $a \rightarrow
b \leftarrow c$.

The Markov equivalence relation partitions the set of all DAGs into
 Markov equivalence classes (MECs), where we denote the MEC of a DAG $D$  as $[D]$.
An MEC can be represented by a CPDAG  $G$ (\emph{completed partially
directed acyclic graph}, also known as an \emph{essential graph}), 
which is the \emph{union graph} of the DAGs in
the equivalence class it represents. When we speak of the union
of a set of graphs $\{G_1 = (V, E_1), \dots, G_k = (V, E_k)\}$,
we think of the graph $G = (V, \bigcup_{i=1}^k E_k)$.
The MEC represented by $G$ is denoted as $[G]$.
The undirected components of a CPDAG are \emph{undirected and connected chordal graphs}
(UCCGs)~\citep{Andersson1997}. In a chordal graph, every undirected
cycle of lengths $\geq 4$ contains a chord, that is an edge between
two vertices of the cycle, which is not part of the cycle.

The problem this paper addresses is, generally speaking, the opposite
direction w.r.t.\ the definition of a CPDAG above: Given a
CPDAG $G$, we aim to compute $| \, [G] \, |$. To do so, we introduce
the following terms\footnote{Most of them are introduced more
  generally for a ``partially directed graph $G$'', you may replace this
  with a ``CPDAG $G$'' during the first read.}: An \emph{extension} of a partially directed
graph $G$ is obtained by replacing
each undirected edge with a directed one.\footnote{We also use the term
  orientation in addition to extension in this paper, which more or less means the same
  thing. In accordance with the literature, we prefer to use orientation for undirected graphs
  and extension for partially directed graphs.} It is called a
\emph{consistent} extension if it is acyclic and does not create a new 
v-structure not present in $G$ (this ensures that for a CPDAG $G$, the
set of consistent extensions constitutes $[G]$). We will denote the number of
consistent extensions of graph $G$ as $\hext(G)$. Hence, if $G$ is a
CPDAG, then $\hext(G)$ is the size of the corresponding Markov
equivalence class. We also refer to the \emph{computational problem}
of counting the number of consistent extensions for a given 
partially directed graph as $\hext$:
\begin{problem}{$\hext$} 
  \begin{parameterizedproblem}
    \instance A partially directed graph $G = (V, E)$.
    \result The number of consistent extensions of $G$. 
  \end{parameterizedproblem}
\end{problem}
By restricting the instances to graphs of a specific graph class, 
we derive the $\hext$ problem for this class.
Naturally, of particular interest is the class of CPDAGs. In this
paper, we also study the problem $\hext$ for interventional essential
graphs, general PDAGs, as well as for MPDAGs. While these graph classes are
formally defined and discussed later, we want to highlight an
important difference: Counting for MPDAGs and PDAGs\footnote{These
  classes are essentially equivalent when it comes to the time complexity
  of the counting task,
  as for every PDAG, there exists an MPDAG with the same consistent
extensions and it can be computed in polynomial time.} is
intractable (more precisely, we show that it is \sharpP-hard). For interventional essential graphs, however, it is
possible to perform this task in polynomial-time, as they share important properties with
CPDAGs (this is formalized in the subsequent section).

\begin{figure} 
  \centering
  \begin{tikzpicture}
    \node (1) at (0,-1.5) {$1$};
    \node (2) at (1.5,-1.5) {$2$};
    \node (3) at (0,-3) {$3$};
    \node (4) at (1.5,-3) {$4$};
    \node (5) at (0,-4.5) {$5$};
    \node (6) at (1.5,-4.5) {$6$};
    \node (7) at (0,-6) {$7$};
    \node (8) at (1.5,-6) {$8$};
    \node (l1) at (0.75, -7) {CPDAG};
    \graph[use existing nodes, edges = {edge}] {
      1 -- {2};
      3 -- {5,6};
      4 -- 6;
      5 -- {6,7};
    };

    \graph[use existing nodes, edges = {arc}] {
      3 -> {1,2};
      4 -> {1,2};
      6 -> 8;
      7 -> 8;
    };
    
    \node (1) at (3,-1.5) {$1$};
    \node (2) at (4.5,-1.5) {$2$};
    \node (3) at (3,-3) {$3$};
    \node (4) at (4.5,-3) {$4$};
    \node (5) at (3,-4.5) {$5$};
    \node (6) at (4.5,-4.5) {$6$};
    \node (7) at (3,-6) {$7$};
    \node (8) at (4.5,-6) {$8$};    
    \node (l2) at (3.75,-7) {UCCGs};

    \graph[use existing nodes, edges = {edge}] {
      1 -- {2};
      3 -- {5,6};
      4 -- 6;
      5 -- {6,7};
    };

    \node (1) at (6,-1.5) {$1$};
    \node (2) at (7.5,-1.5) {$2$};
    \node (3) at (6,-3) {$3$};
    \node (4) at (7.5,-3) {$4$};
    \node (5) at (6,-4.5) {$5$};
    \node (6) at (7.5,-4.5) {$6$};
    \node (7) at (6,-6) {$7$};
    \node (8) at (7.5,-6) {$8$};
    \node (l3) at (6.75, -7) {AMOs};

    \graph[use existing nodes, edges = {arc}] {
      2 -> 1;
      3 -> {5,6};
      6 -> 4;
      5 -> 7;
      6 -> 5;
    };
    
    \node (1) at (6+3,-1.5) {$1$};
    \node (2) at (7.5+3,-1.5) {$2$};
    \node (3) at (6+3,-3) {$3$};
    \node (4) at (7.5+3,-3) {$4$};
    \node (5) at (6+3,-4.5) {$5$};
    \node (6) at (7.5+3,-4.5) {$6$};
    \node (7) at (6+3,-6) {$7$};
    \node (8) at (7.5+3,-6) {$8$};
    \node (l3) at (6.75+3, -7) {Cons.\ ext.\ };

    \graph[use existing nodes, edges = {arc}] {
      2 -> 1;
      3 -> {1,2,5,6};
      4 -> {1,2};
      6 -> 4;
      5 -> 7;
      6 -> {5,8};
      7 -> 8;
    };
    
  \end{tikzpicture}
  \caption{A CPDAG with its UCCGs and possible AMOs. Substituting these into the original CPDAG gives a consistent extension. The AMOs can be constructed independently of each other and the directed part of the CPDAG. Hence, the number of consistent extensions can be computed as the product of the number of AMOs for each UCCG.
    }
  \label{fig:cpdaguccgs}
\end{figure}
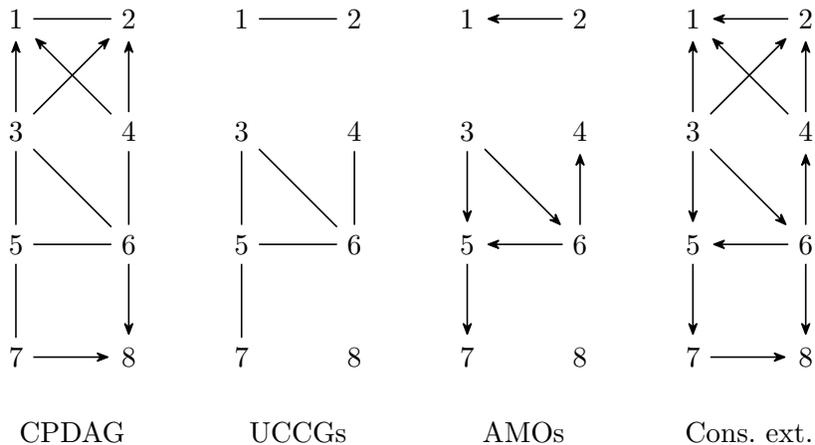

To start with, it is a crucial fact that, for a CPDAG $G$, each UCCG of $G$
can be oriented independently of the other UCCGs 
and the directed part of $G$~\citep{Andersson1997}. This means that to
obtain a consistent extension of $G$, it suffices to orient each UCCG
without creating a cycle or a v-structure. In line with the literature, we term such orientations
AMOs (acyclic moral orientations). 
In accordance with the notation above, we define  $\hamo(H)$ as the
number of AMOs of a connected chordal graph (i.e., UCCG) $H$. We illustrate the introduced terms in
Fig.~\ref{fig:cpdaguccgs}, where the UCCGs of the given CPDAG can be
oriented independently to yield a consistent extension.
Thus, we can conclude that for a CPDAG $G$ the size of $| \, [G] \, |$
is equal to
\begin{equation}\label{eq:hamo:in:cpdag}
  \hext(G) =\prod_{\text{$H$ is UCCG in $G$}} \hamo(H) .
\end{equation}
In other words, the problem $\hext$ of counting the number of DAGs in an
MEC reduces to counting the number of AMOs in a
UCCG~\citep{Gillispie2002,He2008}. We tackle this purely
graph-theoretical problem in Section~\ref{sec:cliquepicking} and derive the first polynomial-time
algorithm for it. For the subsequent Section~\ref{sec:applications}, which focuses on
applications for our methods, it is sufficient to know that such an algorithm exists.

\section{Some Applications of Our Methods} 
\label{sec:applications}
The techniques developed in this paper can be applied to important
tasks in causal discovery and inference. In this section, we highlight
two possible applications: (i) to improve efficiency of learning from 
interventional data and (ii) estimating causal effects from an MEC representation.

\subsection{Incorporating Observational and Interventional Data Efficiently}
As discussed above, when dealing with purely
observational data, a DAG is only
identifiable up to its MEC~\citep{Andersson1997}, which often
makes it impossible to discover the  unique structure. In some cases,
however, additional experimental (also called interventional) data 
may be available or can be produced, in order to resolve the
ambiguities. 
There is a large body of work in the field addressing this problem 
of estimating and explaining a causal structure from both observational \emph{and} 
interventional data.
Analogously to the observational case, all DAGs which satisfy 
the conditional independencies in both observational and interventional 
data form an equivalence class represented by an \emph{interventional 
essential graph}. This graph (as the CPDAG for MECs) is formed by taking the
union of the DAGs in this class.
In the following, we define the concept more formally. 

Let $D=(V,E)$ be a DAG and $f$ be Markov to $D$.
For a set of targets $I \subseteq V$, an intervention with perturbation targets $i\in I$ 
models the effect of replacing the observational distribution 
$f(x_i \mid \textit{pa}_{i}(D))$ by $f^I(x_i)$ 
for all $i\in I$. The intervention graph of $D$ is the DAG $D^I=(V, E^I)$, 
where $E^I =\{ u \to v \in E \mid v\not\in I\}$. Given a family of targets
$\cI \subseteq 2^V$ the pair $(f, \{f^{I}\}_{I \in \cI})$ 
is $\cI$-Markov to $D$ if $f$ is Markov to $D$ and  
for all $I\in \cI$ the interventional distribution $f^I$ factors as
$$
   f^I(x_1,\ldots,x_n)=\prod_{i\not\in I} f(x_i \mid \textit{pa}_{i}(D))\prod_{i\in I} f^I(x_i).$$
Two DAGs $D_1$ and $D_2$ are $\cI$-Markov equivalent if for all positive distributions, 
$(f, \{f^{I}\}_{I \in \cI})$ is $\cI$-Markov for $D_1$ if and only if
it
is $\cI$-Markov to $D_2$. This relation can be expressed in a graphical language as follows:
For a conservative family of targets\footnote{A family of targets $\cI$ is called conservative if for all
$v\in V$, there is some $I \in \cI$ such that $v\not\in I$. Note that, e.g., any 
family containing the empty set $\emptyset$ is conservative.
In this section we assume that the target families are conservative.}  
$\cI$, $D_1$ and $D_2$ are $\cI$-Markov equivalent if for all $I\in \cI$,
$D^I_1$ and $D^I_2$ have the same skeleton and the same v-structures. 
The $\cI$-Markov equivalence class of a DAG $D$ ($\cI$-MEC) is denoted by $[D]_{\cI}$
and can be represented by the $\cI$-essential graph
$\cE_{\cI}(D)=\bigcup_{D'\in [D]_{\cI}} D'$.  
A partially directed graph $G$ is called an $\cI$-essential graph if $G = \cE_{\cI}(D)$ for some DAG $D$.
An example and further explanation regarding interventional MECs and essential
graphs is given in Fig.~\ref{fig:intessgraphs}.

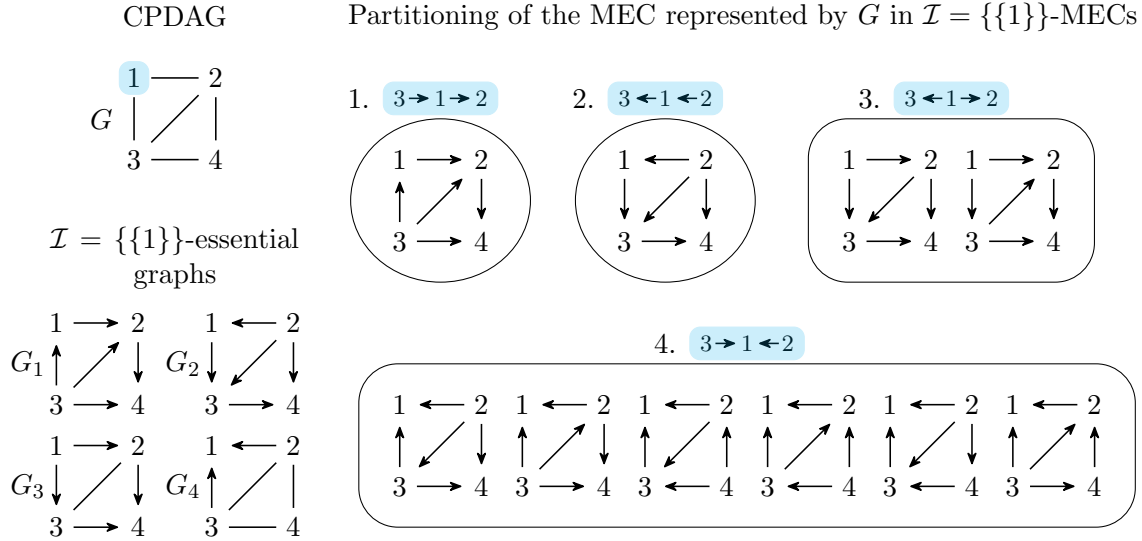
\begin{figure}
  \resizebox{\textwidth}{!}{\begin{tikzpicture}[scale=1.1]
  \node (l1) at (0.25,3.75) {CPDAG};
  \node (l2) at (7.2,3.75) {Partitioning of the MEC represented by $G$  in $\mathcal{I} = \{\{1\}\}$-MECs};
  \node[text width = 4.1cm, align=center] (l3) at (0.25,.8)  {
  $\mathcal{I} = \{\{1\}\}$-essential\\graphs};
  
    \node (G) at (-0.65,2.5) {$G$};
    \node (1) at (-0.25,3) {$1$};
    \node (2) at (0.75,3) {$2$};
    \node (3) at (-0.25,2) {$3$};
    \node (4) at (0.75,2) {$4$};
    \graph[use existing nodes, edges = {edge}] {
      1 -- 2 -- 4 -- 3 -- 1;
      2 -- 3;
    };

    \filldraw[cyan, opacity = 0.2, rounded corners] (-0.25-0.18,3-0.2) rectangle (-0.25+0.18,3+0.2);

    \node (G1) at (-1.55,0-.5) {$G_1$};
    \node (1) at (-1-0.2,0) {$1$};
    \node (2) at (0-0.2,0) {$2$};
    \node (3) at (-1-0.2,-1) {$3$};
    \node (4) at (0-0.2,-1) {$4$};
    \graph[use existing nodes, edges = {arc}] {
      3 -- 1 -- 2 -- 4;
      3 -- 4;
      3 -- 2;
    };

    \node (G2) at (0.5-0.15,0-0.5) {$G_2$};
    \node (1) at (.5+0.2,0) {$1$};
    \node (2) at (1.5+0.2,0) {$2$};
    \node (3) at (.5+0.2,-1) {$3$};
    \node (4) at (1.5+0.2,-1) {$4$};
    \graph[use existing nodes, edges = {arc}] {
      2 -- 1 -- 3 -- 4;
      2 -- 3;
      2 -- 4;
    };

   \node (G3) at (-1.55,-1.5-.5) {$G_3$};
    \node (1) at (-1-0.2,-1.5) {$1$};
    \node (2) at (0-0.2,-1.5) {$2$};
    \node (3) at (-1-0.2,-2.5) {$3$};
    \node (4) at (0-0.2,-2.5) {$4$};
    \graph[use existing nodes, edges = {edge}] {
      2 -- 3;
    };
    \graph[use existing nodes, edges = {arc}] {
      1 -- 2 -- 4;
      1 -- 3 -- 4;
    };    

   \node (G4) at (.5-0.15,-1.5-.5) {$G_4$};
    \node (1) at (.5+0.2,-1.5) {$1$};
    \node (2) at (1.5+0.2,-1.5) {$2$};
    \node (3) at (.5+0.2,-2.5) {$3$};
    \node (4) at (1.5+0.2,-2.5) {$4$};
    \graph[use existing nodes, edges = {edge}] {
      3 -- 2 -- 4 -- 3;
    };   
    \graph[use existing nodes, edges = {arc}] {
      3 -- 1;
      2 -- 1;
    };    

    \node (1) at (3, 2) {$1$};
    \node (2) at (4, 2) {$2$};
    \node (3) at (3, 1) {$3$};
    \node (4) at (4, 1) {$4$};
    \graph[use existing nodes, edges = {arc}] {
      3 -- 1 -- 2 -- 4;
      3 -- 4;
      3 -- 2;
    };

    \node (l1) at (2.5, 2.75) {1.};

    \node[inner sep = 1pt] (3) at (3,2.75) {\footnotesize $3$};
    \node[inner sep = 1pt] (1) at (3.5,2.75) {\footnotesize $1$};
    \node[inner sep = 1pt] (2) at (4,2.75) {\footnotesize $2$};
    \graph[use existing nodes, edges = {arc}] {
      3 -- 1 -- 2;
    };

    \filldraw[cyan, opacity = 0.2, rounded corners] (3-0.2,2.75-0.2) rectangle (4+0.2,2.75+0.2);
    
    \draw (3.5, 1.5) ellipse (1.1cm and 1cm);

    \node (1) at (5.75, 2) {$1$};
    \node (2) at (6.75, 2) {$2$};
    \node (3) at (5.75, 1) {$3$};
    \node (4) at (6.75, 1) {$4$};
    \graph[use existing nodes, edges = {arc}] {
      2 -- 1 -- 3 -- 4;
      2 -- 3;
      2 -- 4;
    };

    \node (l2) at (5.25, 2.75) {2.};
    
    \node[inner sep = 1pt] (3) at (5.75,2.75) {\footnotesize $3$};
    \node[inner sep = 1pt] (1) at (6.25,2.75) {\footnotesize $1$};
    \node[inner sep = 1pt] (2) at (6.75,2.75) {\footnotesize $2$};
    \graph[use existing nodes, edges = {arc}] {
      2 -- 1 -- 3;
    };

    \filldraw[cyan, opacity = 0.2, rounded corners] (5.75-0.2,2.75-0.2) rectangle (6.75+0.2,2.75+0.2);
    
    \draw (6.25, 1.5) ellipse (1.1cm and 1cm);

    \node (1) at (8.5, 2) {$1$};
    \node (2) at (9.5, 2) {$2$};
    \node (3) at (8.5, 1) {$3$};
    \node (4) at (9.5, 1) {$4$};
    \graph[use existing nodes, edges = {arc}] {
      1 -- 2 -- 4;
      1 -- 3 -- 4;
      2 -- 3;
    };

    \node (1) at (10, 2) {$1$};
    \node (2) at (11, 2) {$2$};
    \node (3) at (10, 1) {$3$};
    \node (4) at (11, 1) {$4$};
    \graph[use existing nodes, edges = {arc}] {
      1 -- 2 -- 4;
      1 -- 3 -- 4;
      3 -- 2;
    };

    \node (l3) at (8.75, 2.75) {3.};
    \node[inner sep = 1pt] (3) at (9.25,2.75) {\footnotesize $3$};
    \node[inner sep = 1pt] (1) at (9.75,2.75) {\footnotesize $1$};
    \node[inner sep = 1pt] (2) at (10.25,2.75) {\footnotesize $2$};
    \graph[use existing nodes, edges = {arc}] {
      1 -- 3;
      1 -- 2;
    };

    \filldraw[cyan, opacity = 0.2, rounded corners] (9.25-0.2,2.75-0.2) rectangle (10.25+0.2,2.75+0.2);

    \draw[rounded corners=15pt] (8, 2.5) rectangle (11.5, 0.5);

    \node (1) at (3, -1) {$1$};
    \node (2) at (4, -1) {$2$};
    \node (3) at (3, -2) {$3$};
    \node (4) at (4, -2) {$4$};
    \graph[use existing nodes, edges = {arc}] {
      2 -- 1;
      3 -- 1;
      2 -- 3;
      3 -- 4;
      2 -- 4;
    };

    \node (1) at (4.5, -1) {$1$};
    \node (2) at (5.5, -1) {$2$};
    \node (3) at (4.5, -2) {$3$};
    \node (4) at (5.5, -2) {$4$};
    \graph[use existing nodes, edges = {arc}] {
      2 -- 1;
      3 -- 1;
      3 -- 2;
      3 -- 4;
      2 -- 4;
    };

    \node (1) at (6, -1) {$1$};
    \node (2) at (7, -1) {$2$};
    \node (3) at (6, -2) {$3$};
    \node (4) at (7, -2) {$4$};
    \graph[use existing nodes, edges = {arc}] {
      2 -- 1;
      3 -- 1;
      2 -- 3;
      4 -- 3;
      4 -- 2;
    };

    \node (1) at (7.5, -1) {$1$};
    \node (2) at (8.5, -1) {$2$};
    \node (3) at (7.5, -2) {$3$};
    \node (4) at (8.5, -2) {$4$};
    \graph[use existing nodes, edges = {arc}] {
      2 -- 1;
      3 -- 1;
      3 -- 2;
      4 -- 3;
      4 -- 2;
    };

    \node (1) at (9, -1) {$1$};
    \node (2) at (10, -1) {$2$};
    \node (3) at (9, -2) {$3$};
    \node (4) at (10, -2) {$4$};
    \graph[use existing nodes, edges = {arc}] {
      2 -- 1;
      3 -- 1;
      2 -- 3;
      4 -- 3;
      2 -- 4;
    };

    \node (1) at (10.5, -1) {$1$};
    \node (2) at (11.5, -1) {$2$};
    \node (3) at (10.5, -2) {$3$};
    \node (4) at (11.5, -2) {$4$};
    \graph[use existing nodes, edges = {arc}] {
      2 -- 1;
      3 -- 1;
      3 -- 2;
      3 -- 4;
      4 -- 2;
    };

    \node (l4) at (6.25, -.25) {4.};
    \node[inner sep = 1pt] (3) at (6.75,-.25) {\footnotesize $3$};
    \node[inner sep = 1pt] (1) at (7.25,-.25) {\footnotesize $1$};
    \node[inner sep = 1pt] (2) at (7.75,-.25) {\footnotesize $2$};
    \graph[use existing nodes, edges = {arc}] {
      3 -- 1;
      2 -- 1;
    };
    
    \filldraw[cyan, opacity = 0.2, rounded corners] (6.75-0.2,-.25-0.2) rectangle (7.75+0.2,-.25+0.2);
        
    \draw[rounded corners = 15pt] (2.5,-0.5) rectangle (12, -2.5);
  \end{tikzpicture}}
  \caption{For the CPDAG $G$ on the top left, we show on the right the interventional MECs
  for the family $\mathcal{I} = \{\{1\}\}$, i.e., an intervention is performed on the variable corresponding to vertex $1$ (marked in color). The  possible results for the intervention are shown in the colored regions and the DAGs are partitioned according to those configurations. Each $\{\{1\}\}$-MEC can be represented by the corresponding interventional essential graph on the bottom left, which encodes the still unknown edge orientations as undirected edges.}
  \label{fig:intessgraphs}
\end{figure}

The key property of interventional essential graphs, for our purposes,
is that their undirected components are chordal and induced
subgraphs, just as in CPDAGs:
\begin{proposition}[\citet{hauser2012characterization}] 
\label{prop:hauser}
Let $G$ be an $\cI$-essential graph representing an $\cI$-MEC $[D]_{\cI}$ for 
a target family $\cI$. Then, the undirected components of $G$ are
chordal (we will refer to them as UCCGs, just as for CPDAGs). 
Moreover, a DAG $D'$ is in $[D]_{\cI}$ if and only if $D'$ can be obtained from 
$G$ by acyclic moral orientations of the UCCGs of $G$ independently of each other.
\end{proposition}

For example, in the $\cI$-essential graph $G_4$ in Fig.~\ref{fig:intessgraphs}
representing the $\cI$-MEC determined by the intervention result $3 \rightarrow 1 \leftarrow 2$, 
the undirected component is the triangle 
 \tikz[baseline={(0,-0.12)}]{ \node (2) at (0,0)  {$2$};
  \node (3) at (0.8,0)  {$3$};
  \node (4) at (1.6,0)  {$4$};
  \graph[use existing nodes, edges = {edge}] {
    2 -- 3 -- 4 --[bend right] 2;
  }}. 
The $\cI$-MEC consists of all six DAGs which can be obtained from 
$G$ by acyclic moral orientations of this triangle. In general, this
statement implies that for interventional essential graph $G$, the number of
DAGs in the corresponding equivalence class $[D]_{\cI}$ is
\[
 |[D]_{\cI}|\ =\  \prod_{H \text{ is UCCG in } G} \hamo(H)
\]
and, thus, this can be efficiently computed with the Clique-Picking
algorithm. This leads to the following, main theorem of this section,
which can be proved in  the same manner as Theorem~\ref{thm:main}
and~\ref{thm:main:sampling}.

\begin{theorem} \label{thm:main:interv}
   For a given interventional essential graph representing an $\cI$-MEC
   $[D]_{\cI}$, the number of DAGs in $[D]_{\cI}$ can be computed 
   in polynomial time. Moreover, sampling uniformly a DAG in $[D]_{\cI}$ 
   can be done in linear time, after preprocessing. 
\end{theorem}

The fact that the size of interventional essential graphs can be
computed efficiently can be utilized in the context of \emph{active learning} 
of the underlying causal DAG. It describes the process of designing experiments
(i.e., interventions) in order to recover the DAG. 
A natural approach is to start estimating the essential graph
(CPDAG) with observational data
and afterwards, through experimentation, inferring the direction of
beforehand unorientable edges to reduce the number of indistinguishable DAGs.
Usually the objective is to find the underlying causal DAG with as few
experiments as possible.  Active learning has been the subject of a considerable amount 
of research,~see \citet{EberhardtGS05,Eberhardt08,He2008,hauser2012characterization,hauser2014two,shanmugam2015learning,ghassami2018budgeted,GreenewaldKSMKA19,activelearningdct2020} 
and the references therein.

One way of designing experiments is to use the  following  approach: 
Consider (for simplicity) only noiseless adaptive single-target
interventions, i.e.,
each experiment manipulates a single variable of interest and the
intervention informs us correctly about the resulting 
$\mathcal{I}$-MEC.
In this setting, every intervention reveals the orientations of all edges
adjacent to the intervened vertex and further edge orientations
may be inferred by the Meek rules~\citep{Meek1995}
(see Fig.~\ref{fig:intessgraphs} for an illustration).
Additionally, we assume variables are manipulated sequentially, i.e., one can use intervention results
obtained by manipulating the previous variables to select a current
variable to intervene on. To choose the best intervention target,
usually an objective function w.r.t.\ the current interventional
essential graph is computed for each variable, often based on every
possible intervention result. 

Below we discuss three algorithms following this approach: 
MinMaxMEC and MaxEntropy by \citet{He2008} and 
OptSingle by \citet{hauser2014two}.
The first two, particularly, use the sizes of the
$\mathcal{I}$-MECs resulting from such hypothetical interventions, in
order to compute the objective function. Hence, our methods are 
vital for the computational feasibility of those approaches. Moreover,
we show that even the third approach can be sped up significantly.

The algorithms start with the (observational)
MEC $[D]_{\cI}$, i.e., with $\cI=\{\emptyset \}$ for the true DAG $D$,
which is represented as an $\cI$-essential graph $\cE_{\cI}(D)$. Afterwards, 
while $|[D]_{\cI}|>1$, the current target family 
$\cI$ 
and $G=\cE_{\cI}(D)$ are 
updated as follows: MinMaxMEC selects the variable  to intervene on such that 
\begin{equation}\label{eq:MinMaxMEC}
  v^* = \argmin_{v\in V} \max_{D'\in [D]_{\cI}}   | [D']_{\cI\cup  \{\{v\}\}}| .
\end{equation}
MaxEntropy chooses
\begin{equation}\label{eq:MaxEntropy}
  v^* = \argmax_{v\in V}  H_v,
\end{equation}
where $H_v$ is the entropy defined as follows:
Let $D_1,\ldots,D_k \in [D]_{\cI}$ be  DAGs such that  
$[D_1]_{\cI\cup \{\{v\}\}}\dot{\cup} \ldots \dot{\cup} [D_k]_{\cI\cup  \{\{v\}\}}$
is a partition of $[D]_{\cI}$. Then
$
  H_v = - \sum_{j=1}^k \frac{l_j}{L}\log  \frac{l_j}{L},
$
with $l_j = | [D_j]_{\cI\cup  \{\{v\}\}}| $ and $L= |[D]_{\cI}| $.
Algorithm OptSingle  computes a vertex
\begin{equation}\label{eq:OptSingle}
  v^* = \argmin_{v\in V} \max_{D'\in [D]_{\cI}}  \xi(\cE_{\cI\cup \{\{v\}\}}(D')),
\end{equation}
where  $\xi(H)$ denotes the number of undirected edges in a graph $H$.
Next, the intervention on $v^*$ is realized and the algorithm updates 
$G:=\cE_{\cI\cup\{\{ v^*\}\}}(D)$ and $\cI:=\cI\cup\{\{ v^*\}\}$ completing  
the iteration step. Note that none of these three strategies lead to
an optimal algorithm (in a worst-case
or average-case sense), but are effective greedy heuristics.

\begin{example}\label{example:active:learning}
For the CPDAG $G$ in Fig.~\ref{fig:intessgraphs}, the algorithms 
MinMaxMEC, MaxEntropy and OptSingle partition, for every vertex $v$, 
the MEC represented by $G$ into $\{\{v\}\}$-MECs according to all possible 
results for the intervention on $v$. The partitioning for $v=1$ is shown in 
Fig.~\ref{fig:intessgraphs}. The values needed to select $v^*=2$ or  $v^*=3$ solving 
the Eq.~\eqref{eq:MinMaxMEC}, Eq.~\eqref{eq:MaxEntropy}, resp.~Eq.~\eqref{eq:OptSingle}, 
are given in the table below.

\begin{center}\normalfont\begin{tabular}{r c c c c c}
\toprule
	& cardinalities of & & &  number of  undir. edges  & \\
  	$v$  &  $\{\{v\}\}$-essential MECs &  $\max$ card. & $H_v$ & in $\{\{v\}\}$-essential graphs & $\max$  number\\ 
	   \cmidrule(rl){1-6} 
	$1$  & 6,2,1,1      & 6 & 1.57 & 3,1,0,0       & 3\\
	$2$  & 3,2,2,1,1,1& 3 & 2.45 & 2,1,1,0,0,0 & 2\\
	$3$  & 3,2,2,1,1,1& 3 & 2.45 & 2,1,1,0,0,0 & 2\\
	$4$  & 6,2,1,1      & 6 & 1.57 & 3,1,0,0    & 3\\
 \bottomrule
\end{tabular}
\end{center} \exampleqed
 \end{example}
  
Clearly, the most costly part 
of implementing MinMaxMEC and MaxEntropy is
the counting of Markov equivalent DAGs. As this was previously thought
infeasible, these methods were often avoided~\citep{activelearningdct2020}.
However, one can easily see that, based on
Theorem~\ref{thm:main:interv},
the sizes of MECs needed to choose a vertex w.r.t.\ Eq.~\eqref{eq:MinMaxMEC}, 
resp. Eq.~\eqref{eq:MaxEntropy}, can be computed in polynomial time,
assuming the $(\cI\cup \{\{v\}\})-$MECs are represented 
as interventional essential graphs.

Another efficiency issue of the algorithms,
including OptSingle, concerns the computation of the interventional essential graphs
for each possible intervention results (as there may be exponentially many
such results and the algorithms consider every hypothetical result in
advance, this step is crucial). Interestingly, using the ideas
from Section~\ref{sec:cliquepicking}
we can show that, given a current $\cI$-interventional essential graph $G$
and an interventional result on a vertex~$v$, we can compute the 
new interventional essential graph in linear time.
This is possible using an algorithm based on Maximum Label Search~\citep{berry2009maximal}, which is also used in Clique-Picking, 
as we state in the theorem below.

The only thing left to be explained is 
how to enumerate the possible interventional results 
on $v$: To see this, let $H$ be a UCCG of $G$ containing $v$. Then
the resulting orientations of the incident undirected 
edges $u - v$ in $G$ can be represented as a clique $K\subseteq N_H(v)$,
which contains the incident vertices $u$ of edges oriented as $u\to
v$;\footnote{The parents of $v$ have to form a clique, else a new
  v-structure would be created, which would be in violation with the
  definition of $\mathcal{I}$-MECs.}
the edges  $u - v$, with $u\in N_H(v) \setminus K$, are oriented as $u \gets v$.
Note, that $K$ can be empty.

\begin{theorem} \label{thm:I-MEC:enum}
	Assume $D$ is a DAG, $\cI$ is a  target family, $v$ is a vertex, and  $H$
	is a UCCG of $G=\cE_{\cI}(D)$ containing $v$. Let a clique $K\subseteq N_H(v)$ 
	represent orientations of edges $u - v$ in $G$ as described above
	and  let $D_K \in [D]_{\cI}$ be a DAG with the edges oriented according to $K$.	
	Then, given $G$, $v$, and $K$, the essential graph $G'=\cE_{\cI\cup \{\{v\}\}}(D_K)$ 
	can be computed in time $\mathcal{O}(|V_G| + |E_G|)$. 
\end{theorem}

Notably, this theorem improves upon previous work
by~\cite{Teshnizi20}, which gave an $O(d\cdot m)$ algorithm for this
task (with $d$ being the maximum degree of the graph). For example, computing a vertex solving Eq.~\eqref{eq:MinMaxMEC} in MinMaxMEC can be 
implemented as shown in Algorithm~\ref{alg:MinMaxMEC} below. The other two approaches  
can be implemented similarly. The time complexity is  
$\mathcal{O}(\text{Val}(G)\cdot p(|V_G|, |E_G|))$, where $\text{Val(G)}$ denotes the total number of 
intervention results and $p$ is the polynomial bounding the time complexity 
of the Clique-Picking algorithm used to compute the size of an MEC.

\begin{algorithm}
  \caption{An efficient implementation of MinMaxMEC: 
  The  algorithm computes a vertex solving Eq.~\eqref{eq:MinMaxMEC}.}
  \label{alg:MinMaxMEC}
  \SetKwInOut{Input}{input}\SetKwInOut{Output}{output}
  \DontPrintSemicolon
  \Input{An $\cI$-essential graph $G = (V,E)$.}
  \Output{Vertex $v^*$ solving Eq.~\eqref{eq:MinMaxMEC}.}
  {\bf if} $G$ is a DAG {\bf then return} $\emptyset$ {\bf and stop} \;
  $\text{MinMax} \gets \infty$ \;
  \ForEach{\text{undirected component} $H$ of $G$}{
	  \ForEach{$v\in H$}{
	  	$\text{Max} \gets 0$ \;
	  	\ForEach{clique $K\subseteq N_H(v)$}{
		for $G,v,K$ compute $G'=\cE_{\cI \cup \{ \{v\}\}}(D_K)$  using Theorem~\ref{thm:I-MEC:enum}\;
		$c \gets \hext(G')$ using Algorithm~\ref{alg:main}\;
		{\bf if }  $c  > \text{\rm Max}$ {\bf then}  $\text{Max}  \gets c$\;
  		}
		\uIf{$\text{\rm Max} < \text{\rm MinMax} $}{
     			 $\text{MinMax}  \gets \text{Max}$; $v^* \gets v$
    		}
  }}
  \KwRet $v^*$\;
\end{algorithm}

These improvements can make the difference between infeasibility and
practical applicability. We replicate the recent experimental results
from~\citet{activelearningdct2020}, which includes a comparison of the most
popular algorithms for single-target adaptive active
learning. For their experiments, a single
chordal component was generated in two ways: Large and
sparse graphs were sampled by adding edges to randomly
generated trees, and small and very
dense graphs by ``chordalizing'' Erdös-Renyi graphs.  The three
strategies we discussed above were only
included in the experiments on small graphs (between 8 and 14
vertices) due to their apparent infeasibility. We show that, using implementations\footnote{The experiments in this
  paper record the results using implementations in
  Julia. Additionally, in the code appendix (\url{https://github.com/mwien/counting-with-applications}), a \cpp implementation of Clique-Picking is provided.} of our
methods, these approaches scale to much larger graphs and that they deliver superior results compared to other
algorithms.

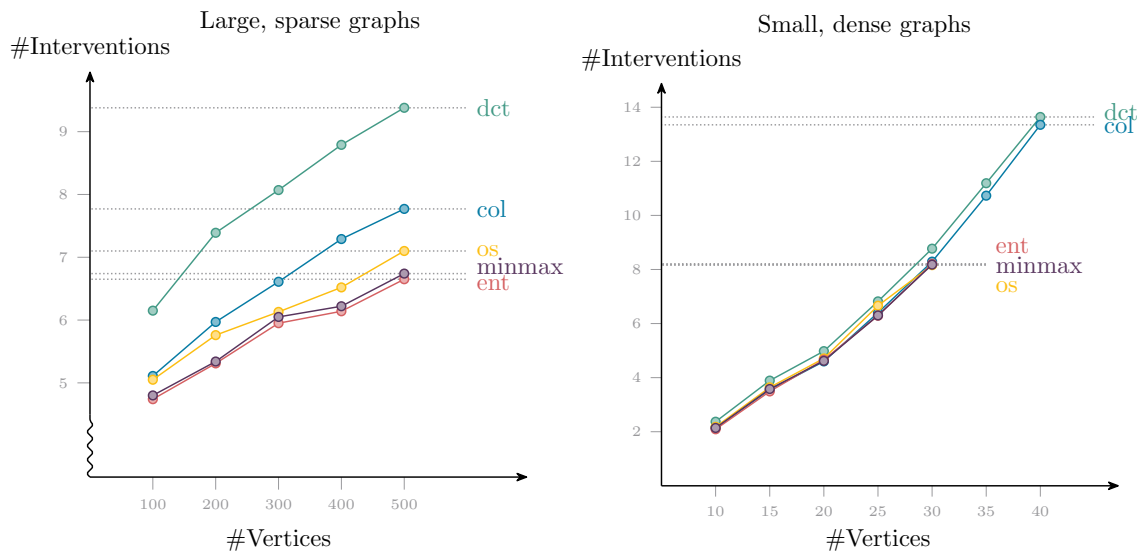
\begin{figure}
  \newcounter{plotlen}
  \newcommand{\plot}[3][black]{%
    \setcounter{plotlen}{0} \foreach \dump in {#2}{\stepcounter{plotlen}}    
    \draw[color=#1, semithick]
    \foreach \y [count=\x] in {#2}{ \ifnum\x=1 (\x,\y) node[point, color=#1, fill=#1!50] {} \else to (\x,\y) node[point, color=#1, fill=#1!50] (tmp) {} \fi };
    \begin{scope}[on background layer]
      \draw[color=ba.gray, semithick, densely dotted] (tmp) to ++(1,0) node[right] {\small\textcolor{#1}{#3}};
      \draw[color=ba.gray, semithick, densely dotted] (tmp) to ++(-\theplotlen,0);
    \end{scope}
  }
  \resizebox{\textwidth}{!}{\begin{tikzpicture}[
    axis/.style     = { semithick, draw, >={[round,sep]Stealth}, -> },
    axisskip/.style = { semithick, draw, rounded corners=1, decorate, decoration = {zigzag, segment length=2mm,amplitude=0.5mm} },
    point/.style    = { draw, circle, inner sep = 0pt, minimum width =
      1.25mm },
    scale = 0.9,
    ]

    \plot[ba.pine]{6.15, 7.39, 8.07, 8.79, 9.38}{dct}
    \plot[ba.blue]{5.11, 5.97, 6.61, 7.29, 7.77}{col}
    \plot[ba.yellow]{5.05, 5.76, 6.13, 6.52, 7.1}{os}
    \plot[ba.red]{4.74, 5.31, 5.95, 6.14, 6.65}{\raisebox{-2ex}{ent}} 
    \plot[ba.violet]{4.8, 5.34, 6.05, 6.22, 6.74}{\raisebox{1.5ex}{minmax}} 

    \foreach \tick [count=\x] in {100, 200, 300, 400, 500}{
      \draw[ba.gray] (\x,3.5) -- ++(0,-0.2) node[below] {\tiny\tick};
    }
    \node at (3, 2.5) {\small \#Vertices};

    \foreach \tick in {5,6,7,8,9}{
      \draw[ba.gray] (0,\tick) -- ++(-0.2,0) node[left] {\tiny\tick};
    }
    \node at (3.5, 10.75) {\small Large, sparse graphs};
    
    \draw[axisskip] (0,3.5) -- (0,4.5);    
    \draw[axis]     (0,4.5) -- (0,10) node[above] {\small\centering \#Interventions};
    \draw[axis]     (0,3.5) -- (7,3.5);            
  \end{tikzpicture}%
  \begin{tikzpicture}[
    axis/.style     = { semithick, draw, >={[round,sep]Stealth}, -> },
    axisskip/.style = { semithick, draw, rounded corners=1, decorate, decoration = {zigzag, segment length=2mm,amplitude=0.5mm} },
    point/.style    = { draw, circle, inner sep = 0pt, minimum width =
      1.25mm },
    scale=0.775
    ]

    \begin{scope}[yscale=0.5]
      \plot[ba.pine]{2.37, 3.89, 4.98, 6.82, 8.77, 11.19, 13.64}{\raisebox{1ex}{dct}}      
      \plot[ba.blue]{2.15, 3.55, 4.6, 6.4, 8.29, 10.73, 13.35}{\raisebox{-1ex}{col}}
      \plot[ba.yellow]{2.18, 3.64, 4.71, 6.66, 8.16}{\raisebox{-5ex}{os}}
      \plot[ba.red]{2.09, 3.49, 4.66, 6.29, 8.21}{\raisebox{3.5ex}{ent}}
      \plot[ba.violet]{2.14, 3.58, 4.62, 6.3, 8.18}{minmax}

      \foreach \tick in {2,4,...,14}{
        \draw[ba.gray] (0,\tick) -- ++(-0.2,0) node[left] {\tiny\tick};
      }
      \draw[axis] (0,0) -- (0,15) node[above] {\small\centering \#Interventions};
    \end{scope}

    \node at (3.75, 8.5) {\small Small, dense graphs};

    \foreach \tick [count=\x] in {10, 15, 20, 25, 30, 35, 40}{
      \draw[ba.gray] (\x,0) -- ++(0,-0.2) node[below] {\tiny\tick};
    }
    \draw[axis]  (0,0) -- (8.5,0);                  
    \node at (4, -1) {\small \#Vertices};
  \end{tikzpicture}}
  \caption{We compare the performance of the three discussed strategies MinMaxMEC ({\color{ba.violet}minmax}), MaxEntropy ({\color{ba.red}ent}) and OptSingle ({\color{ba.yellow}os}), with the coloring-based strategy ({\color{ba.blue}col}) of \cite{shanmugam2015learning}, and directed-clique-tree approach ({\color{ba.pine}dct}) of \cite{activelearningdct2020}.
On the left, the average number of interventions for large, sparse graphs; on the right, for small dense graphs. For each choice of parameters, 100 CPDAGs (chordal graphs) were generated as described in~\cite{activelearningdct2020}. Afterwards, for each CPDAG, a DAG was uniformly chosen from the MEC (the true DAG, used as oracle for the intervention results).}
  \label{fig:actplot}
\end{figure}

In Fig.~\ref{fig:actplot}, the plot on the left shows the number of
performed interventions of five active learning strategies for
the large and sparse graphs and the right one for the small and very
dense graphs. The three strategies we discussed above,
MinMaxMEC, MaxEntropy and OptSingle, clearly outperform the other
methods and, in particular, the two methods which utilize the sizes of
the $\mathcal{I}$-MECs perform the best. In case of the sparse graphs
the differences are larger due to the fact that the structure of the
graph can be utilized to a higher degree; the dense graphs are close to
fully connected graphs.
Importantly, for the sparse graphs, where a lot of performance may be
gained and which occur frequently in practice, the implementation using
Clique-Picking is able to scale up to graphs with 500 vertices (taking
about 10-15 minutes on a desktop computer for the largest graphs\footnote{The
  experiments were run on an Intel(R) Core(TM) i7-8565U
CPU with 16GBs of RAM.}). In
case of the dense graphs, the implementations of MinMaxMEC, MaxEntropy
and OptSingle can handle up to 30
vertices, compared to 14 previously (the bottleneck here is the exponential number of hypothetical intervention results).

Hence, we argue that these methods are feasible in most practical
settings  as, on the one hand, sparse graphs are more prevalent and,
on the other hand, this experiment considers
fully undirected graphs and usually many edge directions are already
detected during the estimation of the CPDAG. Finally, the tradeoff between possibly saved computation time and
finding better intervention targets, should lean, in our view, towards
the latter, as the cost of experimentation exceeds the cost of
beforehand-computation by a large margin.

\subsection{Estimating Causal Effects from CPDAGs and Observed Data}
A second application 
concerns calculating the \emph{total causal effects} to measure 
the effects of interventions  in observational studies using Pearl's 
{\it do}-calculus~\citep{pearl2009causality}. 
For a given DAG $D=(V,E)$, with vertices $V=\{1,\ldots,n\}$ representing random 
variables $X_i$, for $i\in V$, and a distribution $f$ over $X=(X_1,\ldots,X_n)$, 
which is Markov to $D$, 
the distribution generated by an intervention on $X_i$, written $\textit{do}(X_i=x'_i)$,
can be expressed in a truncated factorization formula
$f(x_1,\ldots,x_n\mid \textit{do}(X_i=x'_i))=\prod_{j=1, j\not= i}^n f(x_i \mid \textit{pa}_{i}(D))|_{x_i=x'_i}$ if $x_i=x'_i$ and $0$ otherwise.
By integrating out all variables, except $x_i, x_j$, we get that the distribution 
of $X_j$ after an intervention $\textit{do}(X_i=x'_i)$ 
can be computed by \emph{adjustment for direct causes} of~$X_i$ 
(represented as parents of $i$ in $D$):   
\begin{equation}\label{eq:adj}
  f(x_j \mid \textit{do}(X_i=x'_i))\ =\
 \left\{
    \begin{array}{ll}
       f(x_j) &\mbox{if $j\in \textit{Pa}_{i}$}\\[1mm]
       \int f(x_j \mid x'_i,\textit{pa}_{i})  f(\textit{pa}_{i}) \ d\,\textit{pa}_{i} & 
       \mbox{if $j\not\in \textit{Pa}_{i}$,}
    \end{array}
  \right.
\end{equation}
where $\textit{Pa}_{i}$ stands for  $\textit{Pa}_{i}(D)$, for short, 
and $f(\cdot)$ and  $f(\cdot \mid x'_i, \textit{pa}_{i})$ represent 
preinter\-vention distributions  \citep[Theorem~3.2.2]{pearl2009causality}.
The expected value 
$$
  \mathbb{E}(X_j \mid \textit{do}(X_i=x'_i))\ =\
 \left\{
    \begin{array}{ll}
       \mathbb{E}(X_j) &\mbox{if $j\in \textit{Pa}_{i}$}\\[1mm]
       \int \mathbb{E}(X_j \mid x'_i,\textit{pa}_{i})  f(\textit{pa}_{i}) \  d\,\textit{pa}_{i} & 
       \mbox{if $j\not\in \textit{Pa}_{i}$,}
    \end{array}
  \right.
$$
which summarizes the intervention distribution, can be used to define 
the total causal effect of $X_i$ on $X_j$ as follows:
For non-parametric models the formula 
 \begin{equation}\label{eq:causal:effect:nonparam}
    \theta_{ji}(D)\ =\  \mathbb{E}(X_j \mid \textit{do}(X_i=x'_i)) -  
     				\mathbb{E}(X_j \mid \textit{do}(X_i=x''_i))
\end{equation}
can be taken as a definition of the effect, 
that expresses the difference between expected value of $X_j$
under each of the following two actions: $X_i$ is forced to take value $x'_i$
or $X_i$ is forced to take value $x''_i$ \citep[page~70]{pearl2009causality};
For parametric linear models, the causal effect can be defined as
\begin{equation}\label{eq:causal:effect:param}
	\begin{array}{lll}
           \theta_{ji}(D) &= & \frac{\partial}{\partial x}{\mathbb{E}(X_j | \textit{do}(X_i=x))} \\[2mm]
           &= & \mathbb{E}(X_j \mid \textit{do}(X_i=x'_i+1)) -  \mathbb{E}(X_j \mid \textit{do}(X_i=x'_i))
       \end{array}
\end{equation}
for any chosen value of $x_i'$. 
It shows the change in the expected value of $X_j$ when changing in interventions 
the value of $X_i$ by one unit.

Thus, when the true causal DAG $D$ is known, the outcomes of interventions 
can be estimated using the formulas above and there is an extensive literature 
providing techniques for calculating the causal effects when the 
formula~\eqref{eq:adj}  
is not applicable, e.g., due  to unobservable variables (see, e.g., \citep{pearl2009causality,shpitser2006identification,ShpitserVR2010,van2019separators}).
On the other hand, if a DAG is only identifiable up to its MEC, the situation changes significantly: while for some CPDAGs one can compute the causal effects 
from the graph and the observed data \citep{van2016separators,perkovic2017complete},
in general, for a given CPDAG $G$, the true causal effect of $X_i$ on $X_j$ 
may differ across the DAGs in the  MEC $[G]$. In such cases 
we can at best determine a \emph{multiset} of possible causal effects 
$\theta_{ji}(D)$, one for each DAG $D$ in $[G]$.

Based on this idea, \citet{maathuis2009estimating}~propose 
algorithms, called IDA (\underline{I}ntervention Calculus 
when the \underline{D}AG is \underline{A}bsent), to extract useful causal 
information, e.g.,  to estimate bounds on causal effects. The basic algorithm (Algorithm~1 in 
\citep{maathuis2009estimating}), also called global-IDA in the literature,
starts with an empty multiset $\Theta$ and, for a given CPDAG $G$,  adds 
$\theta_{ji}(D)$ to $\Theta$  for all
DAGs $D$ in $[G]$. To avoid the unnecessary 
enumeration of all DAGs, the authors propose a natural 
modification, which computes the same output $\Theta$  and works as 
follows: Let $D_1,\ldots,D_k \in [G]$ 
be  DAGs such that  $[D_1]_{\{\emptyset,\{i\}\}}\dot{\cup} \ldots \dot{\cup} [D_k]_{\{\emptyset,\{i\}\}}$
is the partition of $[G]$ into the possible interventional-MECs for
$\cI = \{\emptyset,\{i\}\}$. Then, for all $\ell=1,\ldots,k$, the algorithm adds
$c_{\ell}$ copies of $\theta_{ji}(D_{\ell})$, where the multiplicity $c_{\ell}$
denotes the number of DAGs in $[D_{\ell}]_{\{\emptyset,\{i\}\}}$.
The correctness is based on the fact that for all  DAGs in $[D_{\ell}]_{\{\emptyset,\{i\}\}}$,
the causal effect of $X_i$ on $X_j$ is the same.\footnote{In the
original work as well as later papers~\cite{Ghassami2019} the
connection to $\cI$-MECs was not drawn. This is
crucial for obtaining efficient algorithms for computing the
multiplicities through Theorem~\ref{thm:main:interv} and connects the
two applications discussed in this paper.} 

\citet{maathuis2009estimating} notice, that global-IDA
\emph{``works well if the number of covariates is small, say less than
  10 or so''} and  that the bottleneck is the computation of the
multiplicities $c_{\ell}$, which quickly becomes infeasible if the number of covariates $n$ increases. Therefore, the authors developed a
``localized'' version, called local-IDA, which computes the multiset
$\Theta^L=\{\theta_{ji}(D_{1}),\ldots,\theta_{ji}(D_{k})\}$
instead of $\Theta$. This, however, does not reflect the true
multiplicities assuming each DAG in the MEC is equally likely to be
the ground truth. In the context of Fig.~\ref{fig:intessgraphs},
assuming we perform IDA on CPDAG $G$, this means that
the information, that the causal effect $\theta_{41}$ corresponding to
the configuration
$3 \rightarrow 1 \leftarrow 2$ appears in 6 of the possible 10 DAGs, is discarded.

In this paper we show that using our new approaches we 
can effectively implement global-IDA, meaning, in particular, we can efficiently
compute for each possible parent set (i.e., each of the partitions
of the MEC) the multiplicity, meaning the number of DAGs in the partition. We present this 
implementation as  Algorithm~\ref{alg:globalIDA}. Relying on the
formulation of the
problem in terms of $\cI$-MECs from above and utilizing Theorem~\ref{thm:main:interv}
and~\ref{thm:I-MEC:enum}, we can conclude:
\begin{theorem}\label{thm:ida:global}
  For a given CPDAG $G=(V,E)$ and $i,j\in V$, Algorithm~\ref{alg:globalIDA} computes 
  the multiset  $\Theta=\{\theta_{ji}(D)\mid D\in[G] \}$ including causal effects of $X_i$ on $X_j$ for
  all DAGs $D\in[G]$.  It runs in time $\mathcal{O}(k\cdot p(|V|, |E|))$,
  where $k=|\Theta^L|$ and $p$ is a polynomial bounding 
  the time complexity of the Clique-Picking algorithm.
\end{theorem}

\begin{algorithm}
  \caption{An efficient implementation of the global-IDA algorithm.}
  \label{alg:globalIDA}
  \SetKwInOut{Input}{input}\SetKwInOut{Output}{output}
  \DontPrintSemicolon
  \Input{A CPDAG $G = (V,E)$ and vertices $i,j\in V$.}
  \Output{Multiset $\Theta$ of possible causal effects of $X_i$ on $X_j$}
	  	$\Theta \gets \emptyset$ \;
		$H\gets $ undirected component of $G$ containing vertex $i$\;
	  	\ForEach{clique $K\subseteq N_H(i)$}{
		for $G,i,K$ compute $G'=\cE_{\{\emptyset,  \{i\}\}}(D_K)$  using Theorem~\ref{thm:I-MEC:enum}\;
		$c \gets \hext(G')$ using Algorithm~\ref{alg:main}\;
		add $c$ copies of $\theta_{ji}(D_K)$ to $\Theta$
  }
  \KwRet $\Theta$ \;
\end{algorithm}

Hence, the additional effort of Algorithm~\ref{alg:globalIDA} compared
to local-IDA is only a polynomial factor. In practice, this factor
will likely not matter, as we show by replicating the experiments on
linear models originally performed by~\citet{maathuis2009estimating}.

In a causal \emph{linear model}, where every edge represents a linear
direct causal effect and
under the assumption that the distribution of variables is multivariate normal,
one can compute the causal effects $\theta_{ji}(D)$,
defined in Eq.~\eqref{eq:causal:effect:param}, as the regression
coefficient $\beta_{ji\mid \textit{Pa}_{i}(D)}$
of $X_i$ in the linear regression of $X_j$ on $X_i$ and  $\textit{Pa}_{i}(D)$ 
(for details, see e.g., \citep{maathuis2009estimating}).
\citet{maathuis2009estimating} use \emph{sample versions} of global-
and local-IDA, particularly relying on the PC-algorithm and
conditional independence tests for the estimation of CPDAG $G$ from data. 
In their studies, they consider variables 
$X_1, \ldots , X_n, X_{n+1}$ and, as described above, they compute the multisets 
$\hat{\Theta}_i$ and  $\hat{\Theta}^L_i$ for total effects of a
randomly chosen covariate $X_i$ on a response variable $Y = X_n+1$. 

They report, that in simulation studies over sparse DAGs,
already for  $n = 14$, at least one of the 10 replicates\footnote{In the original experiments, the algorithms run over
  10 replicates with sample size 1000. We perform 100 replicates for
  more stable results.} of the global-IDA algorithm
took more than 48 hours to compute, so that the computation was aborted.
Likewise, for the \emph{riboflavin data} with $n=4088$ covariates in the data set,
the global-IDA algorithm is stated as infeasible.
We perform the same experiments reporting (i) the run time of
computing $\hat{\Theta}_i^L$, i.e., the causal effect for each
possible parent set of $X_i$ and (ii) the
run time of computing the multiplicity for each parent set, i.e., the
number of DAGs in the MEC with $X_i$ having the specified parents, as
proposed in Algorithm~\ref{alg:globalIDA}. Hence, local-IDA
would be identical to performing step (i), whereas global-IDA would
consist of (i) and (ii). Table~\ref{tab:IDA:times} shows the
results.

\begin{table}[htbp]  
  \centering\small
  \caption{Mean runtime in seconds of computing the causal effects and
    the multiplicities over 100 replicates with sample
    size 1000 and the specified
    number of covariates. The case $n=4088$ is the real-world
    riboflavin dataset and is hence performed only once. Here, we
    average over \emph{all} covariates $X_i$ instead of choosing a
    random one.}
  \label{tab:IDA:times}
  \begin{tabular}{rlccccccc}
    \toprule
                                           &          &          & \hss\hbox to 0pt{number of covariates \ \hss}     &        &       &      &       &                    \\
                                           &          & $4$      & $9$                                               & $14$   & $29$  & $49$ & $99$  & $4088$             \\
    \cmidrule(rl){1-9}
    \smash{\raisebox{-1.75ex}{Effects}}        & Time in $s$  & 0.09440   & 0.14986  & 0.18639   & 0.24875  & 0.24214    & 0.27393   & 0.48184            \\
                                               & Std.\ dev.\  & 0.11958   & 0.11310  & 0.12029   & 0.16941  & 0.11084    & 0.14278   & 0.22053            \\ [1ex]
    \smash{\raisebox{-1.75ex}{Multipl.}} & Time in $s$  & 0.00014   & 0.00023  & 0.00023   & 0.00086  & 0.00113    & 0.00172   & 0.07042            \\
                                               & Std.\ dev.\  & 0.00028   & 0.00059  & 0.00006   & 0.00453  & 0.00484    & 0.00545   & 0.02758            \\   
    \bottomrule
  \end{tabular}
\end{table}

Clearly, the extra effort of computing the multiplicities is
negligible in this setting. The regression tasks, which both
local- and global-IDA perform for the causal effect estimation, have
significantly larger computational effort. This is due to the fact
that the counting tasks for computing the multiplicities is often
\emph{extremely simple} for the given
graphs. The CPDAGs learned from the PC-algorithm are usually very
sparse. Moreover, the input to Clique-Picking (and other counting
algorithms) consists only of the
undirected components of the CPDAG and even for large graphs, these are often
quite small. In this sense, the graphs used in the experiments in the
previous
subsection were worst-case inputs when it comes to
computational cost, as they were completely undirected. Therefore,
we reemphasize that for most practical problems,
there is no reason to avoid the counting task as the Clique-Picking
algorithm should
be fast enough to handle almost all imaginable cases.

\section{The Clique-Picking Algorithm}
\label{sec:cliquepicking}

In this section, we will show how the problem $\hamo$ defined on
undirected chordal graphs can be solved in polynomial-time through the
novel Clique-Picking algorithm. As
discussed in Section~\ref{sec:preliminaries}, such a polynomial-time algorithm
for $\hamo$ means that the problem of computing the size of an MEC can
be solved in polynomial-time as well.

The algorithm Clique-Picking, which we develop in the following,
heavily relies on the special properties of chordal graphs and their
close connection to AMOs. Hence, we
start by introducing the necessary graphical terms and give important facts.

\subsection{Further Definitions and Known Properties}
\paragraph*{Chordal graphs.}
The set of all maximal cliques of undirected graph $G$ is denoted  by $\cliques(G)$. A vertex is
\emph{simplicial} if its neighbors form a clique.
In a connected graph, we call a set $S\subseteq V$ an
\emph{$a$-$b$-separator} for two nonadjacent vertices $a,b\in V$ if
$a$ and $b$ are in different connected components in
$G[V\setminus S]$. If no proper subset of $S$ separates $a$ and $b$ we
call $S$ a \emph{minimal $a$-$b$-separator}. We say a set
$S$ is a \emph{minimal separator} if it is a minimal $a$-$b$-separator
for any two vertices\footnote{Observe that a minimal separator can be a proper
subset of another minimal separator (for different vertex pairs $a$-$b$).}. We denote the set of all minimal
separators of a graph $G$ by $\separators(G)$.
An undirected graph is called
\emph{chordal} if no subset of four or more vertices induces an
undirected cycle. For every chordal graph on $n$ vertices we have
$|\cliques(G)|\leq n$~\citep{dirac61}. Furthermore, it is well-known
that a graph $G$ is chordal if, and only if, all its minimal
separators are cliques.

\paragraph*{AMOs.}
Just like any other DAG, an AMO $\alpha$ of a UCCG $G$ can be represented by a (not
necessarily unique) linear ordering of the vertices. Such a
\emph{topological} ordering $\tau$ represents $\alpha$ if for each edge $u
\rightarrow v$ in $\alpha$, $u$ precedes $v$ in $\tau$. 
Note that every AMO of a UCCG contains exactly
one source vertex, i.e., a vertex with no incoming
edges\cite{He2015}. 
Based on this observation, one may define the $s$-orientation $G^s$ of
a UCCG $G$ to be the union of all AMOs
of $G$ with unique source vertex~$s$. 
We view $s$-orientations from the equivalent
perspective of being the union of all AMOs that can be represented by
a topological ordering starting with~$s$. The undirected components of~$G^s$
are UCCGs and can be oriented independently~\citep{He2015}.
This observation enables recursive strategies for counting 
AMOs: the ``root-picking''
approaches~\citep{He2015,Ghassami2019,Talvitie2019,Ganian2020} that
pick each vertex~$s$ as source and recurse on the UCCGs of the $s$-orientation.
Because these UCCGs can be oriented independently, the number of AMOs
is obtained by alternately summing over the number of AMOs for each
source vertex $s$ and multiplying the number of AMOs for each independent UCCG.

\subsection{Basics of the Clique-Picking Algorithm}
To count the number of Markov equivalent DAGs, we use
the association between an AMO and its topological
orderings. In accordance with the algorithm we develop,
it is helpful to consider only topological orderings, which are
well-behaved in the following sense:

\begin{definition}
  A topological ordering $\tau$ of an AMO $\alpha$ is called \emph{clique-starting} if it has a maximal clique as a prefix.
\end{definition}

We denote all clique-starting topological orderings representing an
AMO $\alpha$ of a graph~$G$ by $\mathrm{top}_G(\alpha) = \{\tau_1,
\dots, \tau_\ell\}$ and will only consider such topological orderings in the following. It is sound to restrict ourselves in this way due to the following result:

\begin{lemma}
  \label{lemma:startclique}
  Every AMO can be represented by a clique-starting topological ordering.
\end{lemma}

Based on these observations, we generalize the definition of
$s$-orientations with the goal of handling whole cliques at once:
For this, we consider permutations $\pi$ of a clique $K$, as
each $\pi(K)$ represents a distinct AMO of the subgraph induced by $K$.
   
\begin{definition}
  Let $G$ be a UCCG, $K$ be a
  clique in $G$, and let $\pi(K)$ be a permutation of~$K$.
  \begin{enumerate}
  \item The $\pi(K)$-orientation of $G$, also denoted $G^{\pi(K)}$,
    is the union of all AMOs of $G$ that can be represented by a
    topological ordering beginning with $\pi(K)$.
  \item Let $G^{K}$ be the union of $\pi(K)$-orientations of $G$
    over all $\pi$, i.e.,\ let $G^{K} = \bigcup_{\pi} G^{\pi(K)}$.
  \item 
  Denote by $\cC_G(\pi(K))$ the undirected components 
  of $G^{\pi(K)}[V  \setminus K]$ and let 
  $\cC_G(K)$ denote the undirected components 
  of $G^{K}[V  \setminus K]$.
  \end{enumerate}
\end{definition}

Figure~\ref{fig:pikorients} shows an example $\pi(K)$-orientation of $G$:
For a graph $G$ in \textbf{(a)}, a clique $K = \{1,2,3,4\}$,
and a permutation $(4, 3, 2,1)$,  graph $G^{(4,3,2,1)}$ is presented in \textbf{(c)}.
It is the union of two DAGs which are AMOs of $G$, whose
topological orderings begin with $4,3,2,1$. The first DAG 
can be represented by topological ordering $4, 3, 2,1,5,6,7$ and 
the second one by $4, 3, 2,1,6,5,7$.  In Fig.~\ref{fig:pikorients},
we also compare the  $(4, 3, 2,1)$-orientation with an $s$-orientation, for 
$s=4$, shown in \textbf{(b)}. The undirected components of the orientations are
indicated by the colored regions. By orienting whole cliques at once, we get
significantly smaller undirected components in the resulting $\pi(K)$-orientation
than in the $s$-orientation (e.g., $\{5,6\}$ compared to $\{1, 2, 3, 5, 6\}$).
Finally, \textbf{(d)} illustrates graph $G^{\{1,2,3,4\}}$.
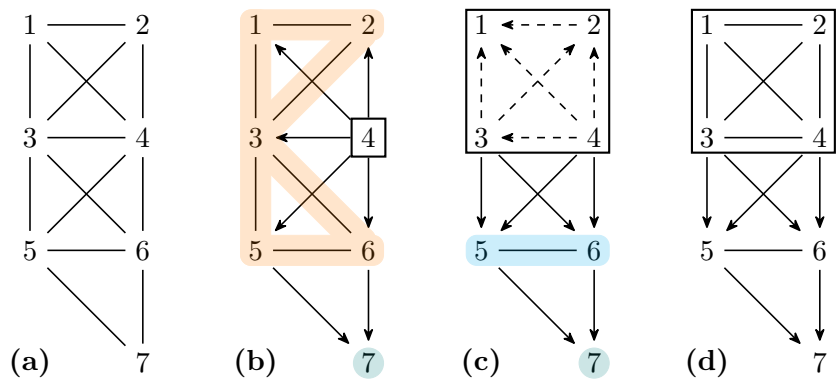
\begin{figure} 
  \centering
  \begin{tikzpicture}
    \node (1) at (0,-1.5) {$1$};
    \node (2) at (1.5,-1.5) {$2$};
    \node (3) at (0,-3) {$3$};
    \node (4) at (1.5,-3) {$4$};
    \node (5) at (0,-4.5) {$5$};
    \node (6) at (1.5,-4.5) {$6$};
    \node (7) at (1.5,-6) {$7$};
    \node (l1) at (0, -6) {\textbf{(a)}};
    \graph[use existing nodes, edges = {edge}] {
      1 -- {2,3,4};
      2 -- {3,4};
      3 -- {4,5,6};
      4 -- {5,6};
      5 -- {6,7};
      6 -- {7};
    };
    
    \node (1) at (3,-1.5) {$1$};
    \node (2) at (4.5,-1.5) {$2$};
    \node (3) at (3,-3) {$3$};
    \node[rectangle,draw, thick] (4) at (4.5,-3) {$4$};
    \node (5) at (3,-4.5) {$5$};
    \node (6) at (4.5,-4.5) {$6$};
    \node (7) at (4.5,-6) {$7$};
    \node (l2) at (3,-6) {\textbf{(b)}};

    \graph[use existing nodes, edges = {edge}] {
      1 -- {2,3};
      2 -- {3};
      3 -- {5,6};
      5 -- 6;
    };

    \graph[use existing nodes, edges = {arc}] {
      4 -> {1,2,3,5,6};
      5 -> 7;
      6 -> 7;
    };

    \begin{scope}[opacity = 0.2, transparency group]
      \filldraw[orange,rounded corners] (2.8,-4.7) rectangle
      (4.7,-4.3);
      \filldraw[orange,rounded corners] (2.8,-1.3) rectangle
      (3.2,-4.7);
      \filldraw[orange,rounded corners] (2.8,-1.3) rectangle (4.7,-1.7);
      \filldraw[orange,rounded corners] (2.85,-3.15) -- (3.15,-2.85) -- (4.65,-4.35) -- (4.35,-4.65);
      \filldraw[orange,rounded corners] (2.85,-2.85) -- (3.15,-3.15) -- (4.65,-1.65) -- (4.35,-1.35);
    \end{scope}
    \filldraw[teal, opacity = 0.2] (4.5,-6) circle (0.2);

    \node (1) at (6,-1.5) {$1$};
    \node (2) at (7.5,-1.5) {$2$};
    \node (3) at (6,-3) {$3$};
    \node (4) at (7.5,-3) {$4$};
    \node (5) at (6,-4.5) {$5$};
    \node (6) at (7.5,-4.5) {$6$};
    \node (7) at (7.5,-6) {$7$};
    \node (l3) at (6, -6) {\textbf{(c)}};

    \graph[use existing nodes, edges = {edge}] {
      5 -- 6;
    };

    \graph[use existing nodes, edges = {arc}] {
      2 ->[dashed] 1;
      3 ->[dashed] {1,2};
      3 -> {5,6};
      4 ->[dashed] {1,2,3};
      4 -> {,5,6};
      5 -> 7;
      6 -> 7;
    };

    \draw[thick] (5.8,-1.3) rectangle (7.7,-3.2);

    \filldraw[cyan, opacity = 0.2, rounded corners] (5.8,-4.7) rectangle (7.7,-4.3);
    \filldraw[teal, opacity = 0.2] (7.5,-6) circle (0.2);
    
    \node (1) at (6+3,-1.5) {$1$};
    \node (2) at (7.5+3,-1.5) {$2$};
    \node (3) at (6+3,-3) {$3$};
    \node (4) at (7.5+3,-3) {$4$};
    \node (5) at (6+3,-4.5) {$5$};
    \node (6) at (7.5+3,-4.5) {$6$};
    \node (7) at (7.5+3,-6) {$7$};
    \node (l3) at (6+3, -6) {\textbf{(d)}};

    \graph[use existing nodes, edges = {edge}] {
      5 -- 6;
      2 -- 1;
      3 -- {1,2};
      4 -- {1,2,3};
    };

    \graph[use existing nodes, edges = {arc}] {
      3 -> {5,6};
      4 -> {,5,6};
      5 -> 7;
      6 -> 7;
    };

    \draw[thick] (5.8+3,-1.3) rectangle (7.7+3,-3.2);

    
  \end{tikzpicture}
  \caption{For a UCCG $G$ in \textbf{(a)}, the figure shows $G^{(4)}$ in 
    \textbf{(b)},  $G^{(4,3,2,1)}$ in \textbf{(c)}, and $G^{\{1,2,3,4\}}$ in \textbf{(d)}.
    The undirected components in $G^{(4)}$ and $G^{(4,3,2,1)}$ are indicated by
    the colored regions and the vertices put at the beginning of the
    topological ordering by a rectangle (all edges from the rectangle
    point outwards). Edges inside the rectangle in \textbf{(c)} are dashed, as they
    have no influence on the further edge directions outside the rectangle. 
    }
  \label{fig:pikorients}
\end{figure}

The crucial observation is that the undirected components
$\cC_G(\pi(K))$ are independent of the permutation
$\pi$. This means no matter how the vertices $\{1, 2, 3, 4\}$ are
permuted, if the whole clique is put at the beginning of the
topological ordering, no further edge orientations will be
influenced. Informally, this is because all edges from the clique $K$
to other vertices are directed outwards no matter the permutation
$\pi$. We formalize this observation as:

\begin{lemma} \label{lemma:perminvariance}
  Let $G$ be a UCCG and $K$ be a clique of $G$.
  For each permutation $\pi(K)$, all edges of $G^{\pi(K)}$
  coincide with the edges of $G^K$, excluding the edges connecting 
  the vertices in $K$, and therefore, in particular, $\cC_G(\pi(K)) = \cC_G(K)$.
\end{lemma}

In our journey towards a polynomial-time algorithm for $\hamo$, 
we start by developing a linear-time algorithm for computing
$\cC_G(K)$. This algorithm will yield structural insights we will
later use for deriving a recursive formula for counting Markov equivalent DAGs.

The computation of $\cC_G(K)$ can be performed efficiently
through adaptions of well-known graph traversal algorithms used
most prominently in chordality testing.
While in~\citep{wienobst2021counting}
specifically the Lexicographic BFS algorithm has been used to
compute $\cC_G(K)$, we now propose a more
general framework which allows ``plugging in'' various linear-time
chordality testing algorithms.

To do this, we need to introduce
further terms from the chordal graph theory and connect them to the
problem of counting AMOs.

\begin{definition}
  A linear ordering $\rho = (x_1, \dots, x_n)$ 
  of the vertices of chordal graph $G$ is called a perfect elimination
  ordering (PEO) if for each $i \in \{1, \dots, n\}$ 
  the vertex $x_i$ is simplicial in $G[\{x_i, \dots, x_n\}]$.
\end{definition}

\begin{lemma}
  \label{lemma:peomao}
  A topological ordering $\tau$ of the vertices of a UCCG $G$ represents an AMO if, and only if, it is the reverse of a PEO.
\end{lemma}

 There are various linear-time graph traversal algorithms designed for chordality
 testing, such as the Lexicographic
 BFS, Maximum Cardinality Search, Lexicographic DFS and Maximal Neighborhood Search~\citep{Rose1976,tarjan1984simple,corneil2008unified}. All of these algorithms are based on the same principle: the
 vertices are visited in reverse order of a PEO if, and only if, the
 graph is chordal. Hence, by Lemma~\ref{lemma:peomao}
 these algorithms will traverse the graph in the topological order of
 an AMO. A fact we exploit in the following.

To make our main algorithm independent on the specific 
method of graph traversal, we present Algorithm~\ref{alg:cgk}
for computing the set $\chordalcomps_G(K)$ in a general form.
It is based on a generic algorithm, called Maximum Label Search
(MLS), to compute a PEO for a given graph $G$~\citep{berry2009maximal}.
MLS traverses $G$ using the following labeling structure:

 \begin{definition}[\citet{berry2009maximal}]
   A labeling structure 
   $\mathcal{L} = (L, \preceq, l_0, \text{Inc})$ consists of:
   \begin{itemize}
   \item $L$ is a set (the set of labels),
   \item $\preceq$ is a partial order on $L$ (which may be total or
     not),
   \item $l_0$ is an element of $L$ (the initial label),
   \item \text{Inc} (increase) is a mapping from $L \times
     \mathbb{N}^+$ 
     to $L$ satisfying the following IC (Inclusion Condition):
     for any subsets $I$ and $I'$ of $\mathbb{N}^+$, if $I \subset
     I'$, then $\text{lab}_{\mathcal{L}}(I) \prec
     \text{lab}_{\mathcal{L}}(I')$. Here, $\mathbb{N}^+$ denotes  the set of positive natural numbers
     and for a set $J= \{i_1, i_2, \dots, i_k\}$ with $i_1 > \dots > i_k$, the labeling function is defined as 
      $\text{lab}_{\mathcal{L}}(J)
     = \text{Inc}(\dots (\text{Inc} (l_0, i_1), \dots), i_k)$,
   \end{itemize}
 \end{definition}

Based on this, the MLS algorithm takes as input a graph $G$ and for a specific  labeling structure
 $\mathcal{L}$ it returns a PEO if $G$ is chordal (see Algorithm~MLS in \citep{berry2009maximal}).
 In this framework, for example, Maximum Cardinality Search is a special case of MLS with
 labeling set $L = \mathbb{N}^+ \cup \{0\}$, the total order $\preceq$ to be $\leq$,
 $l_0 = 0$, and $\text{Inc}(l, i) = l+1$. Our adaption of MLS, which
 computes $\chordalcomps_G(K)$, is presented as
 Algorithm~\ref{alg:cgk}.
 
\begin{algorithm}
  \caption{A generic algorithm
    for computing the set $\chordalcomps_G(K)$.}
  \label{alg:cgk}
  \DontPrintSemicolon
  \SetKwInOut{Input}{input}\SetKwInOut{Output}{output}\SetKwInOut{Aux}{framework}
  \SetKwFor{Rep}{repeat}{}{end}
  \Input{A UCCG $G = (V,E)$, a clique $K\subseteq V$. }
  \Aux{A labeling structure $\mathcal{L} = (L,\preceq,l_0,\textit{Inc})$.}
  \Output{$\chordalcomps_G(K)$.}
  $V' \gets \emptyset$; $P \gets \emptyset$; For all $x\in V$ initialize labels $L(x)$ as $l_0$ \;
  \For{$i = 1 \text{ to } n$}{
    \uIf{$i \leq |K|$}{
      $x \gets$ any vertex in $K \setminus V'$ \label{line:clique}
    }
    \Else{
      $X \gets$ set of vertices in $V \setminus V'$ with maximal label
      \; \label{line:ml}
      Append undirected components of $G[X \setminus P]$ to the output \;
      $P \gets P \cup X$ \;
      $x \gets$ any vertex in $X$ \; \label{line:getsx}
    }
    \ForEach{$y$ in $N(x) \setminus V'$}{\label{line:foreachloop1}
      $L(y) \gets \textit{Inc(}L(y),n-i+1)$
    }\label{line:foreachloop3}
    $V' \gets V' \cup \{x\}$ \;
  }
\end{algorithm}

For convenience, we introduce the following terms:
\begin{definition}
  In the execution of Algorithm~\ref{alg:cgk}
  on input $G, \mathcal{L}$, let $P_i(y)$ be the set of previously
  visited neighbors of vertex $y \in G$ (a vertex is visited if it was
  chosen as $x$ in line~\ref{line:clique} or~\ref{line:getsx})
  before the start of the $i$-th iteration. Moreover, let $i(x)$ be the iteration in which vertex $x \in V
  \setminus K$ was output.
\end{definition}

Clearly, $P_i(y)$ are exactly the vertices, which contributed to
$y$'s label up to iteration $i$.

\begin{lemma}\label{lemma:samechoice}
  Algorithm~\ref{alg:cgk} always chooses vertex $x$ with maximal label.
\end{lemma}

\begin{proof}
  This is the case by construction in line~\ref{line:getsx}.
  We have to show that it also holds in line~\ref{line:clique}.
  Observe that the first $|K|$ chosen vertices are all from clique
  $K$. Hence, when a vertex $x$ from $K$ is chosen all previously chosen
  vertices are neighbors of $x$. This means that any other label is
  equal or smaller.  
\end{proof}
This lemma implies that Algorithm~\ref{alg:cgk}, as the Maximum Label
Search, visits the vertices in reverse PEO order, i.e., in an order
representing an AMO.

\begin{theorem} \label{theorem:cgk}
  Algorithm~\ref{alg:cgk} computes $\chordalcomps_G(K)$. Moreover, it  can be implemented to run in
  time $\mathcal{O}(|V| + |E|)$.
\end{theorem}

\begin{proof}
  Consider two adjacent vertices $a$ and $b$ in $G$. We show that $a$
  and $b$ are in the same subgraph in the output iff we have $a -
  b$ in $G^K$. By transitivity it follows that two vertices are in the same
  subgraph iff there is an undirected path between them, which implies
  the first part of the statement (we will analyze the run time afterwards).
  
  If $a$ and $b$ are in the same connected component
  output by Algorithm~\ref{alg:cgk} then there was a point in the
  algorithm at which $a$ and $b$ 
  had a maximal label and, hence, either one
  could have been chosen as vertex $x$. 
  In both cases the algorithm would have produced a topological
  ordering representing an AMO starting with clique $K$ (following
  from Lemma~\ref{lemma:samechoice} and Lemma~\ref{lemma:peomao}), one time with
  $a \rightarrow b$, the other with $a \leftarrow b$. Hence, we have
  $a - b$ in $G^K[V \setminus K]$ by definition.

  Consider that $a$ and $b$ are not in the same connected component output by the
  algorithm.  Let $a$ be w.l.o.g.\ the vertex which is output
  earlier, i.e., $i(a) < i(b)$.  We show by induction over the order the vertices were visited that 
  $a \rightarrow b$ in $G^K[V
  \setminus K]$. For the start of the induction, observe that the
  vertices in $K$ are not output at all and all
  edges from $K$ to vertices in $V \setminus K$ are oriented towards those vertices.
  
  At the iteration $i = i(a)$ when $a$ was output, $b$ had a
  strictly smaller
  label. It follows that $P_i(a) \neq P_i(b)$. With $P_i(b) \setminus P_i(a)
  = \emptyset$ as the algorithm produces a PEO by Lemma~\ref{lemma:samechoice},
  it follows $P_i(b) \subset P_i(a)$. Let $c$ be in $P_i(a) \setminus
  P_i(b)$. By induction
  hypothesis, we have $c \rightarrow a$ in $G^K$ as $c$
  is not output together with $a$ (recall that $i$ is the iteration
  when $a$ is output, $c$ has already been visited previously). Then,
  $a \rightarrow b$ follows from the first Meek rule.

  Common choices of labeling structure (such as for Maximum Cardinality Search or Lexicographic BFS) lead
  to a linear-time implementation.
\end{proof}
Theorem~\ref{theorem:cgk} is an important result in its own
right. Algorithm~\ref{alg:cgk} may be used not only for computing
$\chordalcomps_G(K)$, but also for computing the $s$-orientations of a
chordal graph $G$ in linear time as well as the interventional
essential graph based on given intervention results (see
Section~\ref{sec:applications} for a discussion).

But for now, we focus on the structural properties regarding AMOs and
chordal graphs revealed by Algorithm~\ref{alg:cgk}, which allow us to
conclude that the undirected
components of $G^K$ (i.e., the graph which occurs when fixing
clique $K$ as ``source'') are chordal and can be oriented
independently. The first fact can be easily seen as, by
Algorithm~\ref{alg:cgk}, the undirected components are induced
subgraphs, which preserve the chordality of the graph. The second fact
is more technical and due to the observation that vertices in the same
undirected component have the same parent set in $G^K$, which ensures
that any AMO of the component will not create a new v-structure in
$G^K$. Crucially, this paves the way towards a recursive formulation
of $\hamo$ based on picking a clique as source.

\begin{corollary}\label{cor:properties:gk}
  Let $G$ be a chordal graph and $K$ a clique.
  \begin{enumerate}
  \item The undirected components of $G^K[V \setminus K]$ are
    induced subgraphs and hence chordal graphs.
  \item Let adjacent $x,y$ in $G$ be in different undirected
    connected components of $G^K[V \setminus K]$ and $i(x) <
    i(y)$. Then, $x \rightarrow y$ is an edge in $G^K[V \setminus K]$. 
  \item $P_{i(v)}(v) =  \textit{Pa}_v (G^K)$.
  \item For adjacent $a,b$ in the same undirected component
    of $G^K[V \setminus K]$, we have that $P_{i(a)} = P_{i(b)}$.
  \item The number  $\hext(G^{\pi(K)})$ can be factorized as 
   \[
      \hext(G^{\pi(K)}) = \prod_{H \in \cC_G(\pi(K))} \hamo(H).
    \]
  \item \label{it:des:formula} The number of AMOs represented by some topological ordering
    with clique $K$ at the beginning (in any permutation) is
    \[
      |K|! \times \hext(G^K) = |K|! \times \prod_{H \in
        \chordalcomps_G(K)} \hamo(H).
    \]
  \end{enumerate}
\end{corollary}

In line with our notation, we write $\hext(G^K)$ as this graph is
partially oriented and $\hamo(H)$ with $H \in \chordalcomps_G(K)$ as
it is an undirected chordal graph.
Based on item~\ref{it:des:formula} of
Corollary~\ref{cor:properties:gk}, we would like to count
the AMOs of a chordal graph $G$ with the following
recursive procedure: Pick a
maximal clique $K$, consider all its
permutations at once (i.e., multiply by $|K|!$), and take the product of the recursively computed number of
AMOs of the UCCGs of $\cC_G(K)$. By
Lemma~\ref{lemma:startclique}, we will count every AMO in
this way, if we compute the sum over all maximal
cliques. Unfortunately, we will count some orientations multiple
times, as a single AMO can be represented by multiple topological
orderings starting with different maximal cliques. For instance,
assume we have two maximal cliques $K_1$ and
$K_2$ with $K_1\cap K_2=S$ such that $K_1\setminus S$ is separated
from $K_2\setminus S$ in $G[V\setminus S]$. A topological ordering
that starts with $S$ can proceed with either $K_1\setminus S$ or
$K_2\setminus S$ and result in the same AMO.

\begin{example}\label{example:coreSplit}
  Consider the following chordal graph (left) with maximal cliques
  $K_1=\{1,2,3\}$ and $K_2=\{2,3,4\}$. A possible
  AMO of the graph is shown on the right.
  
  
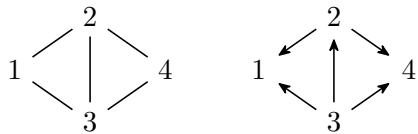
\begin{center}
    \tikz[yscale=0.7]{
      \rhombus
      \graph[use existing nodes, edges = {edge}] {
        1 -- 2 -- 4 -- 3 -- 1; 2 -- 3;
      };
    }\qquad
    \tikz[yscale=0.7]{
      \rhombus
      \graph[use existing nodes, edges = {arc}] {
        2 -> {1, 4};
        3 -> {1, 2, 4};
      };
    }
  \end{center}

  The AMO has two topological orderings: $\tau_1=(3,2,1,4)$ and
  $\tau_2=(3,2,4,1)$ starting with $K_1$ and $K_2$, respectively.
  Hence, if we count all topological orderings starting with
  $K_1$ and all topological orderings starting with $K_2$, we will count the
  AMO twice. However, $\tau_1$ and $\tau_2$ have $(3,2)$ as common
  prefix and $K_1\cap K_2=\{2,3\}$ is
  a minimal separator of the graph~--~a fact that we will use
  in the following.\exampleqed
\end{example}

\begin{lemma}
  \label{lemma:minimalSeparator}
  Let $\alpha$ be an AMO of a chordal graph~$G$ and 
  let $\tau_1, \tau_2\in \mathrm{top}(\alpha)$ be two clique-starting topological orderings that represent $\alpha$. Then
  $\tau_1$ and $\tau_2$ have a common prefix $S\in\separators(G)\cup\cliques(G)$.
\end{lemma}

Note that this lemma implies that \emph{all} topological orderings
that correspond to an AMO have a common prefix, which is a minimal
separator or maximal clique.

The combinatorial function $\phi$, defined below, 
plays a crucial role to avoid overcounting.

\begin{definition}\label{def:phi}
  For a set $S$ and a collection $\cR$ of subsets of~$S$, 
we define $\phi(S,\cR)$ as the number of
all permutations of $S$ that do not have 
a set $S'\in \cR$ as prefix. 
\end{definition}
\begin{example}\label{example:phi}
Consider the set $S= \{2,3,4,5\}$ and the collection $\cR=\big\{\{2,3\},$ $\{2,3,5\}\big\}$. Then 
$\phi(S,\cR) = 16$ since  there are 16 permutations of
$\{2,3,4,5\}$ that neither start with $\{2,3\}$ nor
$\{2,3,5\}$~--~e.g., $(3,2,4,5)$ and $(2,5,3,4)$ are forbidden as
they start with $\{2,3\}$ and $\{2,3,5\}$, respectively; but $(3,5,4,2)$ is allowed.\exampleqed
\end{example}

In this paper, we always consider sets $S \in \separators(G) \cup \cliques(G)$
and collections $\cR \subseteq \separators(G)$. Therefore, we can use the abbreviation
$\phi(S)=\phi\big(S,\{\,S'\mid S'\in\Delta(G)\wedge S'\subsetneq
S\,\}\big)$.

\begin{proposition}\label{proposition:countingFormula}
Let $G$ be a UCCG. Then:
  \[\hamo(G)=\sum_{S\in\Delta(G)\cup\Pi(G)}
  \phi(S)
  \times 
  \prod_{\balap{H\in\chordalcomps_G(S)}}\hamo(H).
  \]
\end{proposition}

\begin{proof}
  By the choice of $S$ and the definition of $\chordalcomps_G(S)$,
  everything counted by the formula is a topological ordering
  representing an AMO.  We argue that every AMO~$\alpha$ is counted
  exactly once. Let $S\in\separators(G)\cup\cliques(G)$ be the
  smallest common prefix of all topological orderings in
  $\mathrm{top}(\alpha)$~--~which is well-defined by
  Lemma~\ref{lemma:minimalSeparator}. First observe that, by the
  minimality of $S$, $\alpha$ is counted at the term for $S$: There is
  no other prefix $\tilde S\subsetneq S$ of the topological orderings
  with $\tilde S\in\separators(G) \cup\Pi(G)$. 

  On the other hand, $S$ is the only term in the sum at which we can
  count $\alpha$, as for any larger $\tilde S$ with $S\subsetneq \tilde
  S$ that is a prefix of some $\tau\in\mathrm{top}(\alpha)$, we have
  $S$ is considered in $\phi(\tilde S)$.
\end{proof}
\begin{example}\label{example:separators}
We consider the following chordal graph with
two minimal separators 
and three maximal cliques:
\begin{center}
  $G=$\tikz[baseline={(0,-0.5)}]{\marcelgraph}%
  \quad%
  \tikz[baseline={(0,-0.5)}]{
    \node[anchor=west, baseline] at (0,0)  {$\separators(G)=\big\{\{2,3\}, \{2,3,5\}\big\}$};
    \node[anchor=west, baseline, text width=7.0cm] at (0,-1) {$\cliques(G)=\big\{\{1,2,3\}, \{2,3,4,5\}, 
    \{2,3,5,6\}\big\}$};
  }
\end{center}
To compute $\hamo(G)$ using Proposition~\ref{proposition:countingFormula}, we need the following
values. Note that the resulting subgraphs $H$ are trivial, except for the
case $S=\{2,3\}$ and $S=\{1,2,3\}$. In these cases, we obtain the
induced path on $\{4,5,6\}$, which has three possible AMOs.
\begin{center}
  \begin{tabular}{ccc}
    \emph{$S\in\separators(G)\cup\cliques(G)$} & \emph{$\phi(S)$} & $\prod\limits_{H\in\cC_G(S)} \hamo(H)$\\[1.5ex]
    $\{2,3\}$      & $2$  & $3$  \\
    $\{2,3,5\}$    & $4$  & $1$  \\
    $\{1,2,3\}$    & $4$  & $3$  \\
    $\{2,3,4,5\}$  & $16$ & $1$  \\
    $\{2,3,5,6\}$  & $16$ & $1$  \\
  \end{tabular}
\end{center}
Using Proposition~\ref{proposition:countingFormula} we can compute $\hamo(G)$ as follows:
\[
  \hamo(G) = 2\cdot 3 + 4\cdot 1 + 4\cdot 3 + 16\cdot 1 + 16\cdot 1 = 54.
\]
We remark that we do \emph{not} have discussed how to
compute~$\phi(S)$ yet~--~for this example, this can be done by na\"ive
enumeration. In general, however, this is a non-trivial task. We tackle this
issue below.  \exampleqed
\end{example}

\subsection{The Algorithm}
From Proposition~\ref{proposition:countingFormula} we know how to count
AMOs by using minimal separators in
order to avoid overcounting and it is rather easy to check that we can
compute $\phi(S,\cR)$ in time \emph{exponential} in $|\cR|$ using the
inclusion-exclusion principle. However, our goal is \emph{polynomial
  time} and, thus, we have to restrict the collection $\cR$.
\begin{lemma}\label{lemma:efficientPhi}
  Let $S$ be a set and $\cR=\{X_1,\dots,X_{\ell}\}$ be a collection of
  subsets of $S$ with $X_1\subsetneq X_2\subsetneq\dots\subsetneq
  X_{\ell}$. Then:
  \[
    \phi(S,\cR) = |S|!
    -\sum_{i=1}^{\ell}|S\setminus X_i|!\cdot\phi(X_i,\{X_1,\dots,X_{i-1}\}).    
  \]
\end{lemma}

\begin{proof}
  We prove the statement by induction over $\ell$ with the base case
  $\phi(S,\emptyset)=|S|!$. Consider a set $S$ and a collection
  $\cR=\{X_1,\dots,X_{\ell}\}$ of subsets of $S$. We can compute
  $\phi(S,\cR)$ by taking $\phi(S,\{X_1,\dots,X_{\ell-1}\})$ (the number
  of permutations of $S$ that do not start with
  $X_1,\dots,X_{\ell-1}$) and by subtracting the number of
  permutations that start with $X_{\ell}$ but none of the other $X_i$, i.e.,
  \begin{align*}
    \phi(S,\cR) &= \phi(S,\{X_1,\dots,X_{\ell-1}\}) 
    - |S\setminus X_{\ell}|!\cdot\phi(X_{\ell},\{X_1,\dots,X_{\ell-1}\}).
  \end{align*}
  Inserting the induction hypothesis, we obtain:
  \begin{align*}
    \phi(S,\cR) &= |S|! -\sum_{i=1}^{\ell-1}|S\setminus X_i|!\cdot\phi(X_i,\{X_1,\dots,X_{i-1}\}) 
    - |S\setminus
       X_{\ell}|!\cdot\phi(X_{\ell},\{X_1,\dots,X_{\ell-1}\})\\
    &=|S|!-\sum_{i=1}^{\ell}|S\setminus X_i|!\cdot\phi(X_i,\{X_1,\dots,X_{i-1}\}).\qedhere
  \end{align*}
  \noqed
\end{proof}
Observe that this formula can be evaluated in polynomial time with
respect to $|S|$ and $\ell$, as all recursive calls have the
form $\phi(X_i,\{X_1,\dots,X_{i-1}\})$ and, thus, there are at most
$\ell$ distinct ones. The goal of this section is to develop a version of
Proposition~\ref{proposition:countingFormula} based on this lemma.
This will allow us to obtain the Clique-Picking algorithm.

To achieve this goal, we rely on the strong structural properties that
chordal graphs entail: A
\emph{rooted clique tree} of a UCCG~$G$ is a triple $(T,r,\iota)$ such
that $(T,r)$ is a rooted tree and
$\iota\colon V_T\rightarrow \cliques(G)$ a bijection between the
nodes of $T$ and the maximal cliques of $G$ such that
$\{\,x\mid v\in\iota(x)\,\}$ is connected in $T$ for all $v\in V_G$.
In slight abuse of notation, we denote, for a set $C\subseteq V_G$, by
$\iota^{-1}(C)$ the subtree $\{\,x\mid C\subseteq\iota(x)\,\}$. We
denote the children of a node $v$ in a tree~$T$ by
$\text{children}_T(v)$.  It is well-known that (i)~every chordal graph
has a rooted clique tree $(T,r,\iota)$ that can be computed in linear
time, and (ii)~a set $S\subseteq V_G$ is a minimal separator if, and only
if, there are two adjacent nodes $x,y\in V_T$ with
$\iota(x)\cap\iota(y)=S$~\citep{Blair1993}.

We wish to interleave the structure provided by the clique tree with a
formula for computing $\hamo$. For this sake, let us define the
\emph{forbidden prefixes} for a node $v$ in a clique tree.

\begin{definition}
  Let $G$ be a UCCG, $\mathcal{T} = (T, r, \iota)$ a rooted clique tree of
  $G$, $v$ a node in $T$ and $r = x_1, x_2, \dots, x_p = v$ the unique
  $r$-$v$-path. We define the set $\mathrm{FP}(v, \mathcal{T})$ to consist of
  all intersections $\iota(x_i) \cap \iota(x_{i+1})$ that are contained in $\iota
  (v)$, for $1 \leq i < p$. 
\end{definition}

\begin{lemma}\label{lemma:inputPhi}
  We can order the elements of the set $\mathrm{FP}(v, \mathcal{T})$ as $X_1
  \subsetneq X_2 \subsetneq \dots \subsetneq X_\ell$.
\end{lemma}

\begin{proof}
  The ordering of the sets is given by the natural order along the
  path from the root $r$ to node $v$.
  The sets in $\mathrm{FP}(v, \mathcal{T})$
  satisfy $\iota(x_i)\cap \iota(x_{i+1})\subseteq \iota(v)$. By the
  definition of a clique tree, we have $\iota(x_i)\cap \iota(x_{i+1})\subseteq \iota(y)$ for
  each $y$ that lies on the $x_i$-$v$-path in $T$. Hence, each such $y$
  can only add supersets of $\iota(x_i) \cap \iota(x_{i+1})$ to $\mathrm{FP}(v, \mathcal{T})$.
\end{proof}
By combining the 
lemma with Lemma~\ref{lemma:efficientPhi}, we
deduce that  $\phi(\iota(v),\mathrm{FP}(v,\mathcal{T}))$ can be evaluated in
polynomial time for nodes $v$ of the clique tree. We are left with the task of developing a formula for
$\hamo$ in which all occurrences of $\phi$ are of this form. It is
quite easy to come up with such formulas that count every AMO at least
once~--~but, of course, we have to ensure that we count every AMO
\emph{exactly} once.
The formula given in proposition below achieves
this goal.

\begin{proposition}\label{proposition:fpFormula}
  Let $G$ be a UCCG and $\mathcal{T} = (T, r, \iota)$ be a rooted
  clique tree of $G$. Then
  \[
    \hamo(G) = \sum_{\balap{v\in V_T}} \phi(\iota(v), \mathrm{FP}(v, \mathcal{T})) 
    \times \prod_{\balap{H \in \chordalcomps_G(\iota(v))}} \hamo(H).
  \]
\end{proposition}

Crucially, evaluating this formula can be done efficiently.
As implementation, we give
Algorithm~\ref{alg:cliquepicking}, which utilizes
memoization to avoid recomputations. Traversing the clique tree with a
BFS allows for a straightforward computation of $\mathrm{FP}$. 

\begin{algorithm}
  \caption{Clique-Picking}
  \label{alg:cliquepicking}
  \SetKwInOut{Input}{input}\SetKwInOut{Output}{output}
  \DontPrintSemicolon
  \Input{A UCCG $G = (V,E)$.}
  \Output{$\hamo(G)$.}
  \SetKwFunction{FCount}{count}
  \SetKwFunction{FNewors}{neworients}
  \SetKwFunction{FPop}{pop}
  \SetKwFunction{FAppend}{append}
  \SetKwFunction{FPush}{push}
  \SetKwProg{Fn}{function}{}{end}
    \KwRet $\FCount(G, \emptyset)$\;
  \Fn{\FCount{$G$, $\mathrm{memo}$}}{
    \If{$G \in \mathrm{memo}$}{\KwRet $\mathrm{memo}[G]$}
    $\mathcal{T} = (T,r,\iota) \gets$ a rooted clique tree of $G$ \;
    $\mathrm{sum} \gets 0$ \;
     $Q \gets$ queue with single element $r$ \;
    \While{$Q$ is not empty}{
      $v \gets \FPop(Q)$ \;\label{line:pop}
      $\FPush(Q, \mathrm{children}(v))$ \;
      $\mathrm{prod} \gets 1$ \;
      \ForEach{$H \in \chordalcomps_G(\iota(v))$}{
        $\mathrm{prod} \gets \mathrm{prod} \cdot \FCount(H, \mathrm{memo})$ \; \label{line:mult}
      }
      $\mathrm{sum} \gets \mathrm{sum} +
      \phi(\iota(v),\mathrm{FP}(v,\mathcal{T})) \cdot \mathrm{prod}$
      \; \label{line:weightedsum}
    }
    $\mathrm{memo}[G] = \mathrm{sum}$ \;
    \KwRet $\mathrm{sum}$\;
  }
\end{algorithm}

\begin{theorem}\label{theorem:cliquepicking}
  For an input UCCG $G$, Algorithm~\ref{alg:cliquepicking} returns
  the number of AMOs of $G$.
\end{theorem}

We defer the rather involved proof of this Theorem to
Section~\ref{sec:missing:proofs}.

\begin{example}\label{example:cliquePicking}
  We consider a rooted clique tree~$(T,r,\iota)$ for the
  graph $G$ from Example~\ref{example:separators}. The root is labeled
  with $r$ and the function $\iota$ is visualized in
  \textcolor{ba.blue}{blue}. The edges of the clique tree are
  labeled with the corresponding minimal separators.

  \begin{center}
    \tikzset{
      tree node/.style = {
        draw,
        fill,
        circle,
        inner sep     = 0pt,
        minimum width = 1mm,
        outer sep     = 5pt,
      }
    }
    
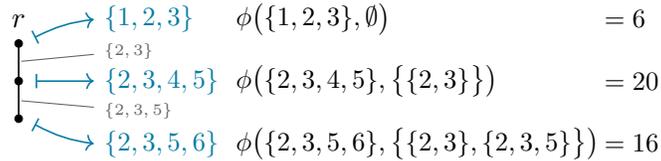
\tikz{
      \node[inner sep=0pt, minimum width=1mm, circle, label={[label distance=0.2mm]$r$}] at (0,0) {};
      \node[tree node] (r) at (0,  0) {};
      \node[tree node] (a) at (0,-.5) {};
      \node[tree node] (b) at (0, -1) {};
      \draw[edge] (0,0) -- (0, -.5) -- (0,-1);

      \node[anchor=west, color=ba.blue, font=\small] (c1) at (1,  0.33)   {$\{1,2,3\}$};
      \node[anchor=west, color=ba.blue, font=\small] (c2) at (1, -0.5)    {$\{2,3,4,5\}$};
      \node[anchor=west, color=ba.blue, font=\small] (c3) at (1, -1.33)   {$\{2,3,5,6\}$};
      \graph[use existing nodes, edges = {edge, |->, color=ba.blue}] {
        r ->[bend left=10]  c1;
        a -> c2;
        b ->[bend right=10] c3;
      };

      \node[anchor=west, color=ba.gray!65!black, font=\tiny] (s1) at (1, -0.1)   {$\{2,3\}$};
      \node[anchor=west, color=ba.gray!65!black, font=\tiny] (s2) at (1, -0.9)  {$\{2,3,5\}$};
      \node (e1) at (-0.1, -0.25) {};
      \node (e2) at (-0.1, -0.75) {};
      \graph[use existing nodes, edges = {thin, color=ba.gray!65!black}] {
        s1 -- e1;
        s2 -- e2;
      };

      \node[anchor=west, font=\small] at (2.75,  0.33) {$\phi\big(\{1,2,3\},   \emptyset\big)$};
      \node[anchor=west, font=\small] at (2.75, -0.5)  {$\phi\big(\{2,3,4,5\}, \big\{\{2,3\}\big\}\big)$};
      \node[anchor=west, font=\small] at (2.75, -1.33) {$\phi\big(\{2,3,5,6\}, \big\{\{2,3\},\{2,3,5\}\big\}\big)$};
      \node[anchor=west, font=\small] at (7.4+0.25,  0.33) {$=6$};
      \node[anchor=west, font=\small] at (7.4+0.25, -0.5)  {$=20$};
      \node[anchor=west, font=\small] at (7.4+0.25, -1.33) {$=16$};
    }
  \end{center}
  Algorithm~\ref{alg:cliquepicking} traverses the tree $T$ from the root
  $r$ to the bottom and computes the values shown at the right. 
  The only case in which we
  obtain a non-trivial subgraph is for $S=\{1,2,3\}$ (an
  induced path on $\{4,5,6\}$). Therefore:
  \begin{equation}
  \hamo(G)=6\cdot 3 + 20\cdot 1 + 16\cdot 1 = 54.\tag*{\exampleqed}
  \end{equation}
\end{example}

Since clique trees can be computed in linear time~\citep{Blair1993}, an
iteration of the algorithm runs in polynomial time due to
Lemma~\ref{lemma:efficientPhi} and~\ref{lemma:inputPhi}.  We prove 
next  that Algorithm~\ref{alg:cliquepicking} performs at most
$2\cdot|\cliques(G)|-1$ recursive calls, which implies overall
polynomial run time.

We analyze the run time of the Clique-Picking algorithm by bounding
the number of connected chordal subgraphs that we encounter. The
following proposition shows that this number can be bounded by
$\mathcal{O}(|\cliques(G)|)$. Recall that we have $|\cliques(G)| \leq |V|$ in
chordal graphs and, thus, we only have to handle a linear number of
recursive calls.

\begin{proposition} \label{prop:subpbound}
  Let $G$ be a UCCG. The number of distinct UCCGs explored by \texttt{count} is bounded by $2|\cliques(G)|-1$. 
\end{proposition}

We can, hence, conclude the following:

\begin{theorem}\label{theorem:runtime}
  The Clique-Picking algorithm runs in time
  $\mathcal{O}\big(\,|\cliques(G)|^2 \cdot (|V| + |E|)\,\big)$.
\end{theorem}

\begin{proof}
  By Proposition~\ref{prop:subpbound}, \texttt{count} explores
  $\mathcal{O}(|\cliques(G)|)$ distinct UCCGs. For each of them, the clique tree is computed in
  time $\mathcal{O}(|V| + |E|)$. Afterwards, for each maximal clique,
  the subproblems are computed by Algorithm~\ref{alg:cgk} 
  in time
  $\mathcal{O}(|V| + |E|)$ by Theorem~\ref{theorem:cgk}.

  For the computation of $\phi(S,\mathrm{FP}(v,\mathcal{T}))$, note
  that $\mathrm{FP}$ can be
  obtained straightforwardly: Traverse the clique tree with a
  BFS, keep track of the nodes on the path from root $r$ to any
  visited node, compute $\mathrm{FP}$ with its definition.

  The function $\phi$ can be evaluated using dynamic programming and
  the recursive formula from
  Lemma~\ref{lemma:efficientPhi}. There are $\mathcal{O}(|S|)$ distinct
  recursive calls and for each a sum over $\mathcal{O}(|S|)$ terms has to
  be computed (as $l$ is always smaller than $|S|$). Because $S$ is a
  clique, the effort is in $\mathcal{O}(|E|)$. 
\end{proof}
We summarize the findings of this section in the following theorem 
which restates our main Theorem~\ref{thm:main}.

\begin{theorem}\label{theorem:main:restate}
  Algorithm~\ref{alg:main} solves the problem  $\hext$ for CPDAGs
  (i.e., the computation of the size of an MEC)
  in polynomial time.
\end{theorem}

 \begin{proof}
   The correctness of Algorithm~\ref{alg:main} follows 
   from Equation~\eqref{eq:hamo:in:cpdag} and Theorem~\ref{theorem:cliquepicking}.
   By Theorem~\ref{theorem:runtime}, it performs polynomially many arithmetic operations. Since $\hamo(G)$ is bounded above
   by $n!$, with $n$ denoting $|V_G|$, all operations run in polynomial time
   because the involved
   numbers can be represented by polynomially many bits.
 \end{proof}
 \null
\section{Uniform Sampling of Markov Equivalent DAGs}\label{sec:uniform:sampling}
In this section, we investigate the problem of uniformly sampling a
DAG from a Markov equivalence class. This problem is closely related
to the counting problem and we show how to solve it efficiently using
the Clique-Picking approach.

\begin{algorithm}
  \caption{The recursive function \texttt{sample} uniformly samples an AMO (represented through its topological ordering) from a UCCG $G$.}
  \label{alg:sampling}
  \SetKwInOut{Input}{input}\SetKwInOut{Output}{output}
  \DontPrintSemicolon
  \Input{A UCCG $G$.}
  \Output{Topological ordering of uniformly drawn AMO of $G$.}
  \SetKwFunction{FSample}{sample}
  \SetKwFunction{FClique}{drawclique}
  \SetKwFunction{FPerm}{drawperm}
  \SetKwFunction{FCount}{precount}
  \SetKwFunction{FConcat}{concat}
  \SetKwProg{Fn}{function}{}{end}
  \Fn{\FSample{$G$}}{
    $\mathcal{T} = (T,r,\iota) \gets \text{ rooted clique tree of } G$
    \; 
    $v \gets \text{drawn with probability proportional to } \phi(\iota(v), \mathrm{FP}(v, \mathcal{T})) 
  \times \prod_{H \in \chordalcomps_G(\iota(v))} \hamo(H)$ \; \label{line:drawclique}
    $\tau \gets \text{uniformly drawn permutation of } \iota(v)
    \text{ without prefix in } \mathrm{FP}(v, \mathcal{T})$ \; \label{line:drawto}
    \ForEach{$H \in \chordalcomps_G(K)$}{
      $\tau \gets \FConcat(\tau, \FSample(H))$ \;
    }
    \KwRet $\tau$ \;
  }
\end{algorithm}
The general approach can be seen in Algorithm~\ref{alg:sampling}.
The recursive function \texttt{sample} takes as input a UCCG $G$ and
produces a topological ordering of the
vertices $\tau$, which represents a uniformly sampled AMO of $G$.
It utilizes the formula
\[
  \hamo(G) = \sum_{\balap{v\in V_T}} \phi(\iota(v), \mathrm{FP}(v, \mathcal{T})) 
  \times \prod_{\balap{H \in \chordalcomps_G(\iota(v))}} \hamo(H).
\]
derived in Proposition~\ref{proposition:fpFormula}. Hence, the
counting is done with respect to a clique tree $\mathcal{T}$ of $G$.

The idea is to
first sample a clique (i.e., a node $v$ of the clique tree), which
is put at the start of the
topological ordering. For this, node $v$ is drawn with
probability proportional to
\[
  \phi(\iota(v), \mathrm{FP}(v, \mathcal{T})) 
  \times \prod_{\balap{H \in \chordalcomps_G(\iota(v))}} \hamo(H),
\]
i.e., the number of AMOs counted at the clique. This will ensure
that every AMO has uniform probability of being drawn. In practice,
it is useful to run the Clique-Picking algorithm 
once as precomputation step, in order not to evaluate the formula
repeatedly. We discuss such implementation details later.
Next, a permutation $\tau$ of chosen clique $K$ is drawn
uniformly from those which do not start with one of the ``forbidden''
prefixes in $\mathrm{FP}(v,\mathcal{T})$. Recall that function $\phi$
counts only such permutations.  Finally, the algorithm recurs, as prescribed by 
the formula above, into the subgraphs in $\chordalcomps_G(K)$, which are considered
independently. The topological orderings sampled for these subgraphs
are appended to $\tau$.

We will start this section by showing that this approach will indeed
sample a uniform AMO. Afterwards, we will discuss possible
implementations of this method.

\begin{theorem}
  For a UCCG $G$, the function \texttt{sample} returns a topological
  ordering representing an AMO chosen with uniform probability.
\end{theorem}

\begin{proof}
  We show the theorem by induction. 
  As base case we consider a single clique $K$. Here, any
  permutation of $K$ represents a unique AMO.
  Because there is only one $v$ to choose and, as
  $\mathrm{FP}(v,\mathcal{T})$ is empty, such a permutation (and hence
  the corresponding AMO) is chosen uniformly.

  In order to make the following arguments more precise, we denote
  with $\mathrm{Pr}(\tau_{\alpha}(G))$ the probability that
  Algorithm~\ref{alg:sampling}
  draws a topological ordering $\tau$ of the vertices in $G$ that
  represents $\alpha$. Our goal is to show, as we just did in the base
  case, that for all $\alpha$:
  \[
    \mathrm{Pr}(\tau_{\alpha}(G)) = 1 / \hamo(G).
  \]

  For UCCG $G$ and clique-tree $\mathcal{T}$, let $v_{\alpha}$ be the
  node in the clique-tree, at which $\alpha$ is ``counted'' and let
  $\pi_{\alpha}$ be the corresponding permutation of the clique $\iota(v)$ in any topological ordering of $\alpha$. The correctness of the proof relies on the fact that both $v_{\alpha}$ and $\pi_{\alpha}$ are unique (for $v_{\alpha}$ this follows from the proof of Proposition~\ref{proposition:fpFormula}). Then:
   \begin{align*}
     \mathrm{Pr}(\tau_{\alpha}(G)) &= \mathrm{Pr}(v_{\alpha})
       \mathrm{Pr}(\pi_{\alpha} \; | \; v_{\alpha})
       \prod_{H \in \chordalcomps_G(\iota(v_\alpha))}
       \mathrm{Pr}(\tau_{\alpha[H]}(H)) \\
     &= \frac{\phi(v_{\alpha}, \mathrm{FP}(v_{\alpha}, \mathcal{T})) \prod_{H \in \chordalcomps_G(\iota(v_{\alpha}))} \hamo(H)
      }{\hamo(G) \cdot \phi(v_{\alpha},
      \mathrm{FP}(v_{\alpha}, \mathcal{T}))}
     \prod_{H \in \chordalcomps_G(\iota(v_\alpha))}
       \mathrm{Pr}(\tau_{\alpha[H]}(H)) \\
     &= \frac{\prod_{H \in \chordalcomps_G(\iota(v_{\alpha}))}
       \hamo(H)}{\hamo(G) \prod_{H \in
       \chordalcomps_G(\iota(v_{\alpha}))}  \hamo(H)} = \frac{1}{\hamo(G)}
   \end{align*}
   In the second step, we insert the definitions of
   $\mathrm{Pr}(v_{\alpha})$ and $\mathrm{Pr}(\pi_{\alpha} \; | \;
   v_{\alpha})$. In the third
   step, we use the induction hypothesis
   \[
     \prod_{H \in \chordalcomps_G(K_\alpha)} \mathrm{Pr}(\tau_{\alpha[H]}(H)) = \frac{1}{\prod_{H \in
         \chordalcomps_G(K_\alpha)} \hamo(H)}
   \]
   to complete the proof.
\end{proof}
We will now discuss how to efficiently implement the proposed sampling algorithm. 
The non-trivial tasks are lines~\ref{line:drawclique} and~\ref{line:drawto}
of Algorithm~\ref{alg:sampling}.

Note that when calling the function \texttt{sample} for an input graph $G$,
it is only necessary to know for each node $v$ in a certain
clique-tree $\mathcal{T}$, the following information: the set
$\mathrm{FP}(v, \mathcal{T})$ and
\[
  \phi(\iota(v), \mathrm{FP}(v, \mathcal{T})) 
  \times \prod_{\balap{H \in \chordalcomps_G(\iota(v))}} \hamo(H).
\]
Moreover, in a single run of the counting algorithm (Algorithm~\ref{alg:cliquepicking})
these terms are computed for $G$ and all possible recursive
subcalls. Hence, in a preprocessing step we perform the counting
algorithm once, storing these information.

For the implementation of line~\ref{line:drawclique},
we hence need to draw from a categorical distribution with known
weights over the nodes of clique-tree $\mathcal{T}$. This is possible
in constant time $\mathcal{O}(1)$ using the Alias Method~\citep{walker1974new,vose1991linear}
assuming that the preprocessing includes the computation of a
Alias Table. As this is possible in linear-time in the
number of categories, there is no computational overhead.

The implementation of line~\ref{line:drawto}
is trickier. In~\citep{wienobst2021counting},
we proposed a routine which performs this step in
$\mathcal{O}(|\iota(v)|^2)$ time. This leads to overall cost of $\mathcal{O}(|V| + |E|)$ of
\texttt{sample}. Moreover, the precomputation is significantly more
complicated, needing time $\mathcal{O}(|\cliques(G)|^2 \cdot |V| \cdot (|V| + |E|)$ and
hence an additional factor $|V|$.

Here, we propose a simple Monte Carlo algorithm for the
implementation of line~\ref{line:drawto}
based on rejection sampling. Due to the combinatorial structure of the
counting function $\phi$, we are able to bound the expected number of draws in
this rejection sampling routine by a constant. This leads to a very
efficient and practical algorithm, as the preprocessing cost are in
the same order as the standard Clique-Picking algorithm, 
i.e., time $\mathcal{O}(|\cliques(G)|^2 \cdot (|V| + |E|))$, and the sampling of
the topological ordering is even possible in time $\mathcal{O}(|V|)$.

\begin{theorem}\label{thm:poly:sampling}
  There is an algorithm that, given a connected chordal graph $G$, uniformly
  samples a topological ordering of an AMO of $G$ in \emph{expected
    time} $\mathcal{O}(|V|)$ after an initial
  $\mathcal{O}(\cliques(G)^2 \cdot (|V| + |E|))$ setup.
\end{theorem}

Theorem~\ref{thm:main:sampling} (announced in the introduction) follows directly  
from the theorem above since to uniformly sample  a DAG in MEC represented 
by a CPDAG $C$ one can uniformly sample a topological ordering of an AMO 
of $G$, independently  for each undirected component $G$ of $C$
and then combine the orderings to obtain a resulting DAG.

\begin{proof}{(of Theorem \ref{thm:poly:sampling})}
  As discussed above, we implement line~\ref{line:drawto}
  in Algorithm~\ref{alg:sampling}
  by rejection sampling, i.e., repeatedly draw random permutations
  until one which is not forbidden is found. 

  We begin by showing that, in expectation, only a constant number of
  draws are necessary (this holds for any input).
  Let $\phi(S, \mathrm{FP})$ be the number of allowed
  permutations. The ratio $\frac{\phi}{\abs{S}!}$ gives the
  probability that a random permutation is allowed.
  We have to find a lower bound for the ratio in order to obtain the statement.
%
  Given a set $S$, the value of $\phi$ reaches its minimum when allowing as
  few prefixes as possible. Consequently, a worst-case collection for
  $S = \{s_1, ..., s_p\}$ is $\mathrm{FP} = \left\{ \{s_1\}, \{s_1,
    s_2\}, ..., \{s_1, s_2, ..., s_{p-1}\}\right\}$.

  In this case, the number of allowed permutations is known as the
  number of irreducible permutations (OEIS A003319~\citep{oeis}),
  which we denote with $\rho(p)$. It is well-known (and a
  special case of Lemma~\ref{lemma:efficientPhi}):
  
  \begin{align*}
    \rho(p) &= p! - \sum_{i=1}^{p-1} i! \cdot \rho(p-i).
  \end{align*}
  
  For our derivation of the lower bound, we start by deriving some simple bounds
  of fractions of binomial coefficients. For $2 \le i \le
  p-2$
  \begin{align*}
    \frac{1}{\binom{p}{i}} \le \frac{1}{\binom{p}{2}} = \frac{2}{p (p-1)}
  \end{align*}
  holds and therefore
  \begin{align*}
    &\sum_{i=1}^{p-1} \frac{1}{\binom{p}{i}} = \frac{2}{p} +
      \sum_{i=2}^{p-2} \frac{1}{\binom{p}{i}} \le \frac{2}{p} +
      \frac{2(p-3)}{p(p-1)} \le \frac{4}{p}.
  \end{align*}
  
  Computing the ratio and using the inputs $S,\mathrm{FP}$ as defined
  above, we have for $p\geq8$
  \begin{align*}
    \frac{\phi(S,\mathrm{FP})}{\abs{S}!} \ge \frac{\rho(p)}{p!}
    &= 1 - \sum_{i=1}^{p-1} i! \cdot \frac{\rho(p-i)}{p!} \\
    &\ge 1 - \sum_{i=1}^{p-1} i! \cdot \frac{(p-i)!}{p!}
    = 1 - \sum_{i=1}^{p-1} \frac{1}{\binom{p}{i}} \\
    &\ge 1 - \frac{4}{p} \geq 1 - \frac{1}{2}.
  \end{align*}
  Hence,
  \[
    \frac{\phi(S,\mathrm{FP})}{\abs{S}!} \geq \frac{1}{2}
  \]
  for $\abs{S} \geq 8$; that the estimate holds for all $\abs{S} < 8$ can be
  checked by hand. In conclusion, it holds
  \begin{equation*}
    \mathbb{E}[\text{number of trials until first success}] \leq \frac{1}{\frac{1}{2}} = 2.
  \end{equation*}

  It remains to analyze the expected run time of this routine. Drawing
  a permutation is possible in linear time in $|S|$. Note that
  $\mathrm{FP}$ can be efficiently represented by only storing the new
  elements of $X_i$ (recall that $X_1 \subsetneq X_2 \subsetneq \dots
  \subsetneq X_p$). Checking whether a permutation is forbidden can
  be done in linear-time as well: For every object $s \in S$, we
  record its first occurrence in $\mathrm{FP}$. If
  it first occurred in set $X_k$, we have $o[s] = \sum_{i=1}^k
  \abs{X_i}$; otherwise, it is in no set of forbidden prefixes and we
  put $o[s] = p+1$. Afterwards, we go through the drawn permutation from
  front to back and memorize the highest $o$-value seen up until this
  step. If at position $i$ the maximal value has been $i$, we
  can conclude that this permutation contains a forbidden prefix.

  We will now discuss the run time of the whole \texttt{sample}
  function: We assume that as precomputation, a modified version of
  the Clique-Picking was performed. Then, using the Alias Method, line~\ref{line:drawclique}
  takes time $\mathcal{O}(1)$. 

  Hence, we have overall expected linear-time for the drawing of a non-forbidden
  permutation. This means, we ``pay'' a constant amount per element in
  the build topological order and therefore this order can even be obtained in
  expected time $\mathcal{O}(|V|)$ after appropriate preprocessing.
  Note that to output the AMO itself, $\Theta(|V| + |E|)$ time is
  needed as this is the size of the output, but in a lot of cases the
  topological ordering might be sufficient.
\end{proof}
We close this section by giving an experimental evaluation of our
algorithm. As there are, to the best of our knowledge, no other
implementations of exact sampling from an MEC, we
will confine ourselves to showing that (i) the overhead of the
preprocessing for sampling compared to the ``standard'' Clique-Picking
algorithm is negligible and (ii) that sampling after preprocessing is
extremely fast. We compare implementations of the algorithms in Julia
and generated chordal graphs as described
in~\citep{wienobst2021counting}, namely using the subtree
intersection method~\citep{SekerHET17} with density parameter $k = \log n$ (the expected
number of neighbors per vertex is proportional to this parameter) and
the algorithm by~\cite{Scheinerman88} for sampling random
interval graphs (interval graphs form a subclass of chordal
graphs). For each input graph, we performed
the counting algorithm without and with preprocessing. The run times are
averages over 100 graphs. Afterwards, we sampled 10 DAGs from each
MEC uniformly, in total forming the average over 1000 sampling steps.

\begin{table}[htbp]
  \centering\small
  \caption{The run times in seconds of the standard Clique-Picking algorithm
    without any precomputations (CP w/o pre.) and the modified one
    which includes precomputations (CP with pre.) for sampling on
    randomly generated chordal graphs
    (using the subtree intersection method as well as random interval
    graphs). For each choice of parameters, the algorithms were run on
  the same 100 graphs. Moreover, we give the average run time of
  sampling (after the preprocessing step), which is calculated as the
  average of 10 samplings per graph.}
  \label{table:experiments}
  \setlength{\tabcolsep}{3.25pt}
  \begin{tabular}{lcccccccccc}
    \toprule
                                                     &  & & & \hss\hbox to 0pt{Number of vertices \ \hss} &   &         &         &         &         \\
                                                     & 16      & 32      & 64      & 128     & 256     & 512     & 1024    & 2048    & 4096    \\
    \cmidrule(rl){1-10}
    %
    %
    \rlap{\it\color{ba.blue} Random subtree intersection ($k=\log_2 n$) \ \hss}    &         &         &         &          &         &         &         &         &         \\
    CP w/o pre.                                         & 0.00076 & 0.00199 & 0.00729 & 0.02718  & 0.09463 & 0.38164 & 1.62875 & 7.53509 & 35.0380 \\
    CP with pre.                                        & 0.00135 & 0.00219 & 0.00774 & 0.02783  & 0.09602 & 0.38530 & 1.63844 & 7.58248 & 35.0759 \\
    Sampling                                              & 0.00001 & 0.00003 & 0.00006 & 0.00013  & 0.00026 & 0.00054 & 0.00118 & 0.00283 & 0.00695 \\[1ex]
    %
    %
    \rlap{\it\color{ba.blue} Random interval graphs \ \hss} &         &         &         &          &         &         &         &         &         \\
    CP w/o pre.                                        & 0.00066 & 0.00211 & 0.00834 & 0.03512  & 0.18089 & 1.14654 & 8.17442 & 66.3541 & 539.270 \\
    CP with pre.                                       & 0.00080 & 0.00233 & 0.00864 & 0.03600  & 0.18278 & 1.15313 & 8.20020 & 66.2455 & 538.496 \\
    Sampling                                             & 0.00002 & 0.00003 & 0.00008 & 0.00025  & 0.00068 & 0.00204 & 0.00691 & 0.02298 & 0.10378 \\
    \bottomrule
  \end{tabular}
\end{table}

First, the run time difference between the standard Clique-Picking algorithm
and the modified one, which includes preprocessing for sampling, is
extremely small. The additional computations do not form the bottleneck of
the approach and have only a small influence on the run time. For the
very large graphs, in particular the dense interval graphs, the run
time difference can hardly be measured, due to the fact that the
additional precomputation effort is independent of the number of
edges, which dominates the run time.\footnote{The execution time
  naturally fluctuates and for the large interval graphs this
  fluctuation influences the result more than the actual
  overhead. Hence, in some cases the precomputation algorithm is
  recorded as faster in the experiments. Clearly, Clique-Picking with
  precomputations does strictly more computations and, thus would,
  without noise, not be faster than normal Clique-Picking.}

Second, it can be clearly seen that sampling (after the initial setup
step) is extremely fast. Even
for large graphs it takes only fractions of a second. We remark that the
sampling algorithm returned the full sampled DAG, which is the
desired output in most cases, but that it would also be possible to
only return the topological ordering, reducing the run time further.

\section{Complexity of Counting Under Background Knowledge}\label{sec:count:background}
As a generalization of the counting problems for MECs, we consider the
problem of counting the number of
DAGs in case of additional background knowledge. The formulation of
the problem will not be different than before, we still want to
compute $\hext(G)$ for a graph $G$, only now we do not make the
assumption that $G$ is a CPDAG (or interventional essential graph), but instead allow for arbitrary input
graphs. This includes two well-known graph classes, the one of PDAGs and
MPDAGs. A PDAG is a partially directed graph without a directed cycle
and an MPDAG is a PDAG, which has been maximally oriented using the
Meek rules~\citep{Meek1995}.

The following theorem shows that Theorem~\ref{theorem:main:restate} is tight
in the sense that counting Markov equivalent
DAGs on the more general input graphs, which encode additional
background knowledge (i.e., PDAGs or MPDAGs) is not in \textsc{P} under standard complexity-theoretic
assumptions. We do this by reduction from the \sharpP-hard problem of
counting the number of topological orderings of a
DAG~\citep{brightwell1991counting}, in the following denoted by $\hto$.

\begin{theorem}
  The problem $\hext$ is
  \sharpP-complete for arbitrary input graphs $G$, and in particular
  for PDAGs and MPDAGs.
\end{theorem}

 \begin{proof}
   We give a  parsimonious reduction which by construction will consist of acyclic
   graphs, hence the hardness follows for PDAGs. The resulting PDAGs
   can moreover be transformed into an equivalent MPDAG (regarding the
   corresponding extensions) in polynomial time~\citep{Meek1995}.
   
   We reduce the \sharpP-hard problem of counting the number of
   topological orderings of a DAG~\citep{brightwell1991counting}
   to counting the number of AMOs of a PDAG.
   
   Given a DAG $G = (V,E)$, we construct the PDAG $G'$ as follows: $G'$ has the
   same set of vertices $V$ as $G$ and we add all edges from $G$ to
   $G'$. We insert an undirected edge for all pairs of remaining nonadjacent
   vertices in $G'$.

   Each extension of $G'$ can be represented by exactly one linear ordering
   of $V$ (because $G'$ is complete) and each topological ordering of
   $G$ is a linear ordering as well.
   We prove in two directions that a linear ordering of $V$ is an AMO of $G'$ if, and only if, it
   is a topological ordering of $G$.
   
   \begin{enumerate}
   \item[$\Rightarrow$)] If a linear ordering $\tau$ represents an AMO
     of $G'$, the edges in $G$ are correctly
     reproduced. Hence, it is a topological ordering of $G$.
   \item[$\Leftarrow$)] If a linear ordering $\tau$ is a topological ordering of $G$,
     the orientation of $G'$ according to it is, by definition, acyclic
     and reproduces the directed edges in $G'$. As $G'$ is complete,
     there can be no v-structures. Hence, $\tau$ represents an AMO of $G'$. \qedheretext
   \end{enumerate}
   \noqed
 \end{proof}
  Notably, the reason Clique-Picking cannot be used to solve these
  counting problems can be directly
  connected to the main idea of the proof as well. Intuitively, the
  problems for PDAGs and MPDAGs can be
  reduced to the setting that, when counting AMOs in UCCGs,
  some edge orientations in the chordal component are
  predetermined by \emph{background knowledge}. Hence, in
  the Clique-Picking algorithm, when counting the number of
  permutations for a clique $K$, we have to count only those consistent with the
  background knowledge. But this is equivalent to the hard
  problem of counting the number of topological orderings of a DAG. We
  formalize this in the following. First, we introduce a modified
  version of counting function~$\phi$.

  \begin{definition}
    For a set $S$, a collection $\cR$ of subsets of $S$ and a partial
    order $\preceq$ over the elements of $S$, we define
    $\phi'(S,\cR,\preceq)$ as the number of all permutations of $S$
    consistent with $\preceq$ that do not have a set $S' \in \cR$ as prefix.
  \end{definition}

  Hence, we generalize the function $\phi$ used in the Clique-Picking algorithm
  to counting only linear orderings (i.e., permutations) consistent with a given
  partial order. If $\cR$ is empty, it coincides with the problem is
  of counting the extensions
  of a partial order. As this is equivalent to the problem $\hto$ (all
  relations can be encoded as directed edges), we will denote by
  $\hto(S, \preceq)$ the number of linear orderings of $S$ consistent
  with $\preceq$.

  \begin{lemma}
    Let $S$ be a set and $\cR = \{X_1, \dots, X_\ell\}$ be a collection
    of subsets of $S$ with $X_1 \subsetneq X_2 \subsetneq \cdots
    \subsetneq X_\ell$. Then, function $\phi'(S,\cR,\preceq)$ can be
    computed by $\mathcal{O}(\ell^2)$ calls to $\hto$.
  \end{lemma}

  \begin{proof}
    We base our approach on the recursive formula derived in Lemma~\ref{lemma:efficientPhi}
    \[
      \phi(S,\cR) = |S|!
      -\sum_{i=1}^{\ell}|S\setminus X_i|! \cdot \phi(X_i,\{X_1,\dots,X_{i-1}\}).    
    \]
    Instead of $|S|!$, compute the number of permutations of $S$
    consistent with $\preceq$. In the sum, check whether the partition
    in $X_i$ (at the beginning of the permutation) and $S \setminus
    X_i$ (at the end of the permutation) violates the partial ordering
    (let indicator function $I(X_i, \preceq)$ denote this and evaluate to 0 if $\preceq$ is
    violated, else to 1). Replace $|S \setminus X_i|$ by the number of
    permutations of this subset of $S$ which conforms to $\preceq$. We
    obtain:
    \[
      \phi'(S,\cR,\preceq) = \hto(S, \preceq)
      -\sum_{i=1}^{\ell}I(X_i, \preceq) \cdot \hto(S\setminus X_i,
      \preceq) \cdot \phi'(X_i,\{X_1,\dots,X_{i-1}\}, \preceq).    
    \]
    Correctness follows as in Lemma~\ref{lemma:efficientPhi}
    and as there are at most $\ell$ recursive calls, we have $\mathcal{O}(\ell^2)$
    calls to $\hto$.
  \end{proof}
  \begin{theorem}
    Counting the number of AMOs can be solved in time $\mathcal{O}(n^4 \cdot T(n))$
    for PDAGs and MPDAGs, where $T(n)$ is the time required to solve an instance of $\hto$. 
  \end{theorem}

  \begin{proof}
    We consider the following algorithm (input is a PDAG or an MPDAG $G$)
    \begin{enumerate}
      \item Compute the CPDAG $C$, which contains all the DAGs represented
        by $G$~\citep{WieobstExtendability2021}. Note that a PDAG or
        an MPDAG represents a subset of an MEC, $C$ is the CPDAG of this class.
      \item Consider the UCCGs of $C$, compute the number of AMOs
        consistent with the edges in $G$ for each, and multiply
        them. This way the number of AMOs of $G$ can be obtained.
        
        We do the computation for each UCCG by calling a modified
        version of \texttt{count} from Algorithm~\ref{alg:cliquepicking}
        with additional parameter $\preceq$ and $\phi$ replaced by
        $\phi'$. We pass this function a UCCG of $C$ and as $\preceq$ we
        choose $\preceq_G$, i.e., the partial ordering over the UCCG
        given by the directed edges of~$G$ (i.e., $u \preceq_G v$ if
        $u \rightarrow v$ in $G$). 
      \end{enumerate}

      The correctness follows immediately, as Algorithm~\ref{alg:cliquepicking}
      considers every AMO once and this modification prunes exactly
      those AMOs not conforming to the background knowledge.

      As $\texttt{count}$ is called at most $n$ times and there are at
      most $n$ maximal cliques, the function $\phi'$ will be called at most
      $n^2$ times. In the worst case, evaluating $\phi'$ needs
      $\mathcal{O}(n^2)$ calls to $\hto$, thus we obtain the overall
      bound of $\mathcal{O}(n^4)$ calls.
  \end{proof}
  In practice, the bound of $\mathcal{O}(n^4)$ oracle calls should be rather pessimistic as the
  parameter $l$ in the computation $\phi'$, i.e., the number of
  forbidden prefixes, is usually rather small.

\section{Conclusion}\label{sec:conclusion}
We presented the first polynomial-time algorithms for counting and
sampling Markov equivalent DAGs. Crucially, our novel Clique-Picking
approach is also extremely fast in practice. This means that
especially the task of computing the size of an MEC does \emph{not}
have to be avoided, as we have argued by
demonstrating the feasibility in two important applications. This
enables researchers to choose more reliable and robust
algorithms.

For the uniform sampling problem, we gave a new and simple linear-time
algorithm after preprocessing with minimal overhead, which performs
very well in practice, in particular, when many DAGs are sampled from
the same MEC. Finally, we completed the theoretical study of the
problem by showing that the more general problem with additional
background knowledge is not solvable in polynomial-time under common
complexity-theoretical assumptions, while also giving a reduction to
classical counting problems.


\section{Missing Proofs}
\label{sec:missing:proofs}

\subsection{Proof of Theorem \ref{thm:I-MEC:enum} in Section \ref{sec:applications}}

 \begin{proof} 
     A natural approach to compute $G'$ for possible interventional values
     represented by $K$, is as follows. We orient the edges in $G$ according to $K$ and
     next apply directly the Meek rules \citep{Meek1995}. It has been
     shown that it is sufficient to only apply the first two Meek rules and
     with an efficient implementation utilizing the special structure of the problem
     this yields time $\mathcal{O}(d \cdot
     m)$~\citep{Teshnizi20} (where $d$ is the maximal degree of the graph). 
     Below we show that using our methods we can compute $G'$ in linear time $\mathcal{O}(n + m)$.

     Due to Proposition~\ref{prop:hauser} we know that, to compute $G'$,
     it is sufficient to orient only $H$ into $H'$ since the remaining UCCGs of
     $G$ remain unchanged. 
     Let $D$ be the set of vertices reachable from $v$ (including $v$ itself) in $H$
     with edges incident to $K$ removed. Let $A = V \setminus \{D \cup K\}$ be
     the remaining vertices without $K$. As we will show in the
     following, (i) the induced subgraph $H[A \cup K]$
     is undirected, (ii) there are no edges between $A$ and $D$, (iii)
     the edges from $K$ to $D$ are oriented outwards from $K$ (iv)
     and the edges in $H[D]$ are given by
     calling Algorithm~\ref{alg:cgk} on $H[D \cup K]$ with clique $K
     \cup \{v\}$.

     We begin with (ii). Assume, for the sake of contradiction
     there is an edge $A \ni a - v \in V$. Then, by definition, $a$
     would be part of $V$.

     For (iii), observe that there is a path in $H[D]$ from
     $v$ to every vertex. For the sake of the argument, let us only
     consider shortest paths. Then, the first Meek rule can be iteratively applied
     along that path (note that the first edge is given by the
     intervention result). Hence, in $H'$, 
     there is a directed path from $v$ to any vertex in
     $H'[D]$. Consequently, every edge between $K$ and $D$ has to be
     oriented from $K$ to $D$ to avoid a directed cycle (every vertex
     in $K$ is a parent of $v$ in $H'$).

     We are now able to show (i). From (ii) and (iii), every edge between $A \cup K$ and $D$
     is oriented from $A \cup K$ to $D$. It follows that the
     chordal induced subgraph $H[A \cup K]$ can be oriented
     independently of the remaining graph as no v-structure nor cycle
     can occur.

     It is left to show (iv). By the intervention result and (iii), we
     know that every edge from the initial clique $K \cup \{v\}$ is
     oriented outwards. It immediately follows from the correctness of
     Algorithm~\ref{alg:cgk} that every implied directed edge is
     correctly detected (as it follows from those initial
     orientations). To see that all undirected edges $a-b$ are indeed
     undirected in $H'$, recall that in the proof of Theorem~\ref{theorem:cgk} it is argued that there exists an AMO with
     $a \rightarrow b$ and one with $a \leftarrow b$. Now note that
     finding an AMO for $H[D \cup K]$ (the orientation of the initial
     clique does not matter, just consider an arbitrary fixed
     orientation), will also yield an AMO for $H$ by
     combining it with an AMO for $H[A \cup K]$. Hence, the same
     argument holds.
 \end{proof}
\null
\subsection{Missing Proofs in Section \ref{sec:cliquepicking}}

\subsubsection{Proof of Lemma  \ref{lemma:startclique}}
 \begin{proof} 
   Consider AMO $\alpha$. We construct one-by-one a topological ordering starting
   with a maximal clique by an adaption of Kahn's algorithm~\citep{kahn1962topological}.
   First, let the start vertex in the ordering be the
   unique source $s$ (recall that an AMO has a unique source vertex)
   and let set $S = \{s\}$ denote the already
   considered vertices. Second, as long as there is a vertex
   adjacent to every $x \in S$, choose such a vertex $v$ which is incident
   to no edge $u \rightarrow v$ in $\alpha$ for $u \in V \setminus
   S$ and add it to $S$. Third, iteratively append the remaining vertices to the
   ordering by repeatedly choosing vertices with no incoming edges from
   unvisited vertices.

   Clearly, the resulting ordering is a topological ordering and starts
   with a maximal clique provided vertex $v$ always exists. Consider
   the set $W = \{ w \; | \; w \in N(u) \text{ for all } u \in S\}$ of common
   neighbors of $S$, which is non-empty in the
   second phase. Assume for a
   contradiction that every vertex in $W$ has an incoming edge from a
   vertex in $V \setminus S$. Note that no
   vertex in $w \in W$ can have an incoming edge from $x \in (V \setminus S)
   \setminus W$ as this would imply a v-structure $y \rightarrow w
   \leftarrow x$ for a $y \in S$ not adjacent to $x$ in the given graph $\alpha$. 
   As the graph $\alpha$ is acyclic (and this property holds for
   taking induced subgraphs, i.e., for $G[W]$ as well) there has to
   be a vertex in $W$ with no incoming edge -- a contradiction.
 \end{proof}
\null
\subsubsection{Proof of Lemma  \ref{lemma:perminvariance}}
 \begin{proof} 
     We prove the statement by showing that, for two arbitrary permutations $\pi(K)$ and
   $\pi'(K)$, the edges in $G^{\pi(K)}$ and $G^{\pi'(K)}$ coincide,
   excluding the edges connecting the vertices in $K$.

   The graph
   $G^{\pi(K)}$ is defined as the union of all AMOs, which can be
   represented by a topological ordering starting with $\pi(K)$. Take
   such an AMO $\alpha$ and, in a corresponding topological ordering $\tau$, replace
   $\pi(K)$ by $\pi'(K)$ obtaining a new topological ordering $\tau'$.
   The orientation $\alpha'$ represented by $\tau'$ is, by definition, acyclic and,
   moreover, moral. For the latter property, assume for a
   contradiction, that there is a v-structure (immorality) $a
   \rightarrow b \leftarrow c$. Because $\alpha$ is moral and only edge
   directions internal in $K$ have been changed in $\alpha'$, it has to
   hold that either
   \begin{enumerate}
   \item two vertices of $a,b,c$ are in $K$ (w.l.o.g.\ assume these are $a$ and
     $b$), but then we have $b \rightarrow c \not\in K$ as $c$ is not in $K$ and
     thus preceded by $b$ in $\tau'$, or
   \item all three vertices are in $K$, but then $a \rightarrow b
     \leftarrow c$ is no induced subgraph as $K$ is a clique.
   \end{enumerate}
   Hence, such a v-structure can not exist and $\alpha'$ is moral as
   well. The reverse direction follows equivalently.

   Therefore, the union of all AMOs, which can be represented by a
   topological ordering $\tau'$ starting with $\pi'(K)$, yields the exact
   same graph as for $G^{\pi(K)}$, excluding the internal edges in $K$. Thus,
   $\chordalcomps_G(\pi(K)) = \chordalcomps_G(\pi'(K))$ and, by definition,
   $\chordalcomps_G(\pi(K)) = \chordalcomps_G(K)$.
 \end{proof}
\null
\subsubsection{Proof of Lemma \ref{lemma:peomao}}
 \begin{proof} 
   For the first direction, assume $\tau$ is a topological ordering representing an
    AMO. By definition of AMOs, there can not be a v-structure and, thus,
    if two vertices $x,y \in N(u)$ precede $u$
    in $\tau$, they need to be neighbors. This implies that the neighbors
    of $u$ preceding $u$ in $\tau$ form a clique. Thus, the reverse
    of $\tau$ is a perfect elimination ordering.

    For the second direction, assume $\rho$ is a perfect elimination ordering and orient the edges
    according to the topological ordering that is the reverse of
    $\rho$. Clearly, the orientation is acyclic. Moreover, there can be
    no v-structure, as two vertices $x,y$ preceding $u$ in the reverse
    of $\rho$ are neighbors. Thus, the reverse of $\rho$ represents an AMO.
  \end{proof}
  \null

\subsubsection{Proof of Corollary \ref{cor:properties:gk}}
 \begin{proof} 
   \begin{enumerate}
   \item Follows immediately from Theorem~\ref{theorem:cgk}.
   \item Shown in the proof of Theorem~\ref{theorem:cgk}.
   \item We show two directions: Let $x \in P_{i(v)}(v)$. Then,
     $x$ is a neighbor of $v$ and output before $v$. By 2. we have $x
     \rightarrow v$. Now, let $x \in  \textit{Pa}_v(G^K)$,
     i.e., $x$ is connected by a directed edge to $v$ in $G^K$. From 1. it follows that $x$
     is not in the same connected component. Then $x$ is visited before
     $v$ is output and consequently in $P_{i(v)}(v)$.
   \item As $a$ and $b$ are output in the same iteration, they both
     have the maximum label, and could both have been picked as vertex
     $x$. However, if $P_{i(a)}(a) \setminus P_{i(b)}(b) \neq
     \emptyset$ or $P_{i(b)}(b) \setminus P_{i(a)}(a) \neq \emptyset$
     the algorithm would not produce the reverse of a PEO
     (after the choice of either $a$ or $b$). A
     contradiction. Hence, the statement follows. 
   \item By 1. the undirected components are chordal induced subgraphs and hence its
     consistent extensions are AMOs. It is left to show that the
     orientations of the connected
     components can be constructed separately, yielding the product
     formula.By combining
     3. and 4., the set of parents is identical for each vertex in the same
     component. Then, the statement follows from this fact analogously to Theorem~4
     and~5 from Lemma~10 in~\citep{He2008}.
   \item By Lemma~\ref{lemma:perminvariance}
     we have that $\hamo(G^{\pi(K)}) = \hamo(G^K)$ for any permutation $\pi$. As there are $|K|!$ many permutations, which all lead to different AMOs, and combined with 5. we arrive at the stated formula.\qedheretext
   \end{enumerate}
   \noqed
 \end{proof}
 \null
\subsubsection{Proof of Lemma  \ref{lemma:minimalSeparator}}
 \begin{proof} 
   Assume by Lemma~\ref{lemma:startclique} that $\tau_1$ starts with the maximal clique $K_1$ and $\tau_2$ with
   the maximal clique $K_2$. Since every AMO of a UCCG has a
   unique source, $\tau_1$ and $\tau_2$ start with the same vertex and,
   hence, $K_1\cap K_2=S\neq\emptyset$.

   We first show that $\tau_1$ and $\tau_2$ have to start with
   $S$. Assume for a contradiction that in $\tau_1$ there is a vertex
   $u \not\in S$ before a $v \in S$. The edge between $u$ and $v$ is
   directed as $u \rightarrow v$ in $\alpha$, but as $v \in K_2$ and
   $u \not\in K_2$, the ordering $\tau_2$ implies $u \leftarrow v$.

   If $K_1=K_2$ then $S\in\cliques(G)$ and we are done. We prove that
   otherwise $S$ is a minimal separator in $G$
   that separates $P_1=K_1\setminus S$ from $P_2=K_2\setminus S$. Note
   that the minimality follows by definition. It remains to show that
   $S$ indeed separates $P_1$ and $P_2$. For a contradiction, let
   $P_1\ni x_1  - x_2 - \dots - x_{k-1} - x_k \in P_2$ be a shortest
   $P_1$-$P_2$-path in $G[V\setminus S]$ with $x_i\not\in K_1\cup K_2$
   for $i \in \{2, \dots, k-1\}$.
   According to $\tau_1$, we have the edge $x_1 \rightarrow x_2$ in
   $\alpha$. Since we consider a shortest path,
   $x_{i-1} - x_{i} - x_{i+1}$ is always an induced subgraph and, thus,
   an iterative application of the first Meek rule implies
   $x_{k-1} \rightarrow x_k$.  However, $\tau_2$ would imply the edge
   $x_{k-1} \leftarrow x_k$ in $\alpha$ -- a contradiction.
 \end{proof}
 \null
 
\subsubsection{Proof of Proposition \ref{proposition:fpFormula}}

 To ensure the property  that  we count every AMO
\emph{exactly} once, 
 we introduce for every AMO~$\alpha$ a partial
 order~$\prec_{\alpha}$ on the maximal cliques. Then we prove that
 there is a unique minimal element with respect to this order, and
 deduce a formula (the one given in
 Proposition~\ref{proposition:fpFormula}) for $\hamo$ that counts~$\alpha$ only ``at this
 minimal element''.
 To get started, we need a technical definition and some auxiliary
 lemmas that give us more control over the rooted clique tree. 

 \begin{definition} 
   An \emph{$S$-flower} for a minimal separator~$S$ is a maximal set
   \[
     F\subseteq\{\,K\mid K\in\cliques(G)\wedge S\subseteq K\,\}
   \]
   such that $\bigcup_{K\in F}K$ is
   connected in $G[V\setminus S]$. The \emph{bouquet}~$\bouquet(S)$ of a
   minimal separator $S$ is the set of all $S$-flowers.
 \end{definition}
 \begin{example}
   The $\{2,3\}$-flowers of the graph from Example~\ref{example:separators} are
   $\{\{1,2,3\}\}$  and  $\{\{2,3,4,5\},$  $\{2,3,5,6\}\}.$
   \exampleqed
 \end{example}
 \begin{lemma}
   \label{lemma:connflowers}
   An $S$-flower $F$ is a connected subtree in a rooted clique tree $(T,r,\iota)$.
 \end{lemma}

 \begin{proof}
   Assume for a contradiction that $F$ is not
   connected in $T$. Then there are cliques $K_1,K_2\in F$ that
   are connected by the unique path $K_1-\tilde K-\dots-K_2$ with $\tilde
   K\not\in F$. Since $\iota^{-1}(S)$ is connected, we
   have $S\subseteq\tilde K$. By the maximality of~$F$, we
   have $K_1\cap\tilde K=S$. But then $S$ separates $K_1\setminus S$
   from $K_2\setminus S$, which contradicts the definition of $S$-flowers.
 \end{proof}
 \begin{lemma}
   \label{lemma:bouquetpart}
   For any minimal separator $S$, the bouquet $\bouquet(S)$ is a
   partition of $\iota^{-1}(S)$.
 \end{lemma}

 \begin{proof}
   For each $x\in\iota^{-1}(S)$, the maximal clique $\iota(x)$ is in
   some $S$-flower by definition. However, no maximal clique can be
   in two $S$-flowers, as these flowers would then be in the same connected
   component in $G[V\setminus S]$.
 \end{proof}
 Since for a $S\in\Delta(G)$ the subtree $\iota^{-1}(S)$ of
 $(T,r,\iota)$ is connected,
 Lemma~\ref{lemma:connflowers} and Lemma~\ref{lemma:bouquetpart} give
 rise to the following order on $S$-flowers $F_1,F_2\in\bouquet(S)$:
 $F_1\prec_T F_2$ if $F_1$ contains a node on the unique path
 from $F_2$ to the root of $T$.

 \begin{lemma}\label{lemma:flowersAreOrdered}
   There is a unique least $S$-flower in $\bouquet(S)$ with respect to $\prec_T$.
 \end{lemma}

 \begin{proof}
   Assume, there is no unique least $S$-flower.
   Then there are two minimal $S$-flowers which are
   incomparable. However, by
   Lemma~\ref{lemma:connflowers} and~\ref{lemma:bouquetpart}, and the definition
   of the partial order, there has to be another $S$-flower closer to
   the root and, thus, lesser given the partial order -- a contradiction.
 \end{proof}
 The lemma states that for every
 AMO~$\alpha$ there is a flower $F$ at which we want to count
 $\alpha$. We have to be sure that this is possible, i.e., that a
 clique in $F$ can be used to generate~$\alpha$.

 \begin{lemma}
   \label{lemma:BSeveryAMO}
   Let $\alpha$ be an AMO such that every clique-starting topological
   ordering that represents $\alpha$ has the minimal separator~$S$ as
   prefix. Then every $F\in\bouquet(S)$ contains a clique $K$ such that
   there is a $\tau\in\mathrm{top}(\alpha)$ starting with $K$.
 \end{lemma}

 \begin{proof}
   Let $\tau$ be a topological ordering representing $\alpha$ that starts with $S$.
   By Lemma~\ref{lemma:startclique}, there is at least one clique $K$ with
   $S\subseteq K$ such that $\tau$ has the form $\tau=(S,K\setminus S,V\setminus K)$. Let $F\in\bouquet(S)$ be
   the flower containing $K$ and $F'\neq F$ be another $S$-flower with
   some $K'\in F'$. Observe that $K\setminus S$ is disconnected from
   $K'\setminus S$ in $G[V\setminus S]$. Therefore, there is a
   topological ordering of the form $(S,K'\setminus S, V\setminus K')$
   that represents $\alpha$ as well.
 \end{proof}
 We use $\prec_T$ to define, for a fixed AMO~$\alpha$, a partial
 order~$\prec_{\alpha}$ on the set of maximal cliques, which are at the
 beginning of some $\tau\in\mathrm{top}(\alpha)$, as follows:
 $K_1\prec_{\alpha}K_2$ if, and only if,
 (i)~$K_1\cap K_2=S\in\separators(G)$, (ii)~$K_1$ and $K_2$ are in
 $S$-flowers $F_1,F_2\in\bouquet(S)$, respectively, and
 (iii)~$F_1\prec_T F_2$.

 Now, we are ready to give:
 \begin{proof}{(of Proposition \ref{proposition:fpFormula})}
   We have to show that every
   AMO~$\alpha$ is counted exactly once. Recall that
   $\mathrm{top}(\alpha)=\{\tau_1,\dots,\tau_{\ell}\}$ is the set of clique-starting topological orderings
   that represent $\alpha$, and that the rooted clique tree $(T,r,\iota)$
   implies a partial order $\prec_T$ on flowers, which in return
   defines partial order $\prec_{\alpha}$ on the set of maximal cliques that are at the
   beginning of some $\tau\in\mathrm{top}(\alpha)$.

   \begin{claim}
     There is a unique least maximal clique $K\in\cliques(G)$ with respect to $\prec_{\alpha}$.
   \end{claim}
   \begin{proof}
     Let $\mathrm{top}'(\alpha)\subseteq\mathrm{top}(\alpha)$ be an arbitrary subset of the
     clique-starting topological orderings that represent~$\alpha$ and let $\mu$ be the
     number of different maximal cliques with which elements in
     $\mathrm{top}'(\alpha)$ start. We prove the claim by induction over
     $\mu$. In the base case, all elements in $\mathrm{top}'(\alpha)$ start with
     the same set $S\in\cliques(G)$ and, of course, this is the unique
     least maximal clique. For $\mu>1$ we observe
     that, by Lemma~\ref{lemma:minimalSeparator}, all
     $\tau\in\mathrm{top}'(\alpha)$ start with the same $S\in\separators(G)$.

     Consider the bouquet $\bouquet(S)$, which is partially ordered
     by~$\prec_T$. Lemma~\ref{lemma:flowersAreOrdered} states that there is a
     unique least $S$-flower $F\in\bouquet(S)$ with respect to
     $\prec_T$, and by the definition of~$\prec_{\alpha}$ the maximal
     cliques occurring in $F$ precede the others. Therefore, we reduce
     $\mathrm{top}'(\alpha)$ to the set $\mathrm{top}''(\alpha)$ of topological orderings that
     start with a maximal clique in $F$. This set is non-empty by
     Lemma~\ref{lemma:BSeveryAMO} and contains, by the induction
     hypothesis, a unique least maximal clique.
   \end{proof}
   We complete the proof by showing that
   the formula counts~$\alpha$ at the term for the unique least
   maximal clique~$K$ from the previous claim. To see this, we need to
   prove that (i) $\alpha$ can be counted at the clique $K$ (i.e., there is no
   set $S\in \mathrm{FP}(\iota^{-1}(K), \mathcal{T})$ preventing $\alpha$
   from being counted), and (ii) that
   $\alpha$ is not counted somewhere else (i.e., there is some set $S\in
   \mathrm{FP}(\iota^{-1}(K'), \mathcal{T})$ for all other $K'\in\cliques(G)$ that can be at the beginning
   of some $\tau\in\mathrm{top}(\alpha)$).

   \begin{claim}
     Let $\alpha$ be an AMO and $K\in\cliques(G)$ be the least
     maximal clique (with respect to $\prec_{\alpha}$) that is a prefix of
     some $\tau\in\mathrm{top}(\alpha)$. Then there is no $S\in\Delta(G)$ with
     $S\in \mathrm{FP}(\iota^{-1}(K), \mathcal{T})$ that is a prefix of $\tau$.
   \end{claim}
   \begin{proof}
     Assume for a contradiction that there would be such a
     $S\in\Delta(G)$ and let $F\in\bouquet(S)$ be the $S$-flower
     containing~$K$. Since $S\in \mathrm{FP}(\iota^{-1}(K), \mathcal{T})$, there is another flower
     $F'\in\bouquet(S)$ with $F'\prec_T
     F$. Lemma~\ref{lemma:BSeveryAMO} tells us that there is another
     clique $K'\in F'$ that is at the beginning of some
     $\tau'\in\mathrm{top}(\alpha)$. However, then we have
     $K'\prec_{\alpha}K$~--~contradicting the minimality of $K$.
   \end{proof}
   \begin{claim}
     Let $\tau_1,\tau_2\in\mathrm{top}(\alpha)$ be two topological orderings
     starting with $K_1,K_2\in\cliques(G)$, respectively. If
     $K_1\prec_{\alpha} K_2$ then $K_1\cap K_2=S\in
     \mathrm{FP}(\iota^{-1}(K_2), \mathcal{T})$.
   \end{claim}
   \begin{proof}
     Since $K_1$ and $K_2$ correspond to $\tau_1,\tau_2\in\mathrm{top}(\alpha)$,
     we have $K_1\cap K_2=S\in\cliques(G)\cup\separators(G)$ by
     Lemma~\ref{lemma:minimalSeparator}~--~in fact, $S$ is a prefix of
     $\tau_1$ and $\tau_2$. As we assume $K_1\prec_{\alpha} K_2$, we have
     $K_1\neq K_2$ and, thus, $S\in\Delta(G)$. Let
     $F_1,F_2\in\bouquet(S)$ be the $S$-flowers containing $K_1$ and
     $K_2$, respectively. The order $K_1\prec_{\alpha}K_2$ implies
     $F_1\prec_T F_2$ (item (iii) in the definition of~$\prec_{\alpha}$), meaning that $F_1$ contains some node of
     $(T,r,\iota)$ that is on the unique path from $F_2$ to the root of
     $T$. But by the definition of $S$-flowers and
     Lemma~\ref{lemma:bouquetpart}, the first edge on this path that
     leads to a node in $F_1$ connects two nodes $x,y$ with
     $\iota(x)\cap\iota(y)=S$. Hence, $S\in \mathrm{FP}(\iota^{-1}(K_2), \mathcal{T})$.
   \end{proof}
   This completes the proof of Proposition~\ref{proposition:fpFormula}.
 \end{proof}
 \null
 
\subsubsection{Proof of Theorem \ref{theorem:cliquepicking}}

 \begin{proof}
   Observe that recursive calls are performed in
   line~\ref{line:mult} if  $\chordalcomps_G(\iota(v))\neq\emptyset$. The only graphs with
   $\chordalcomps_G(S)=\emptyset$ for all $S\in\cliques(G)$ are the
   complete graphs, i.e., the graphs with $|\Pi(G)|=1$.
   We have $|\cliques(H)|<|\cliques(G)|$ for all graphs $G$ and
   $H=G[V\setminus S]$ with $S\in\cliques(G)$. Hence, we may assume by
   induction over $|\cliques(G)|$ that the subproblems are handled
   correctly~--~the base case being given by complete graphs.

   The correctness of the algorithm follows from
   Proposition~\ref{proposition:fpFormula}, as it traverses the clique
   tree with a BFS in order to compute the sets
   $\mathrm{FP}(v, \mathcal{T})$ and evaluates this formula. 
 \end{proof}
\null

\subsubsection{Proof of Proposition \ref{prop:subpbound}}

 Let $V'_{i(x)}$ be the set of all visited vertices by
 Algorithm~\ref{alg:cgk} in the step before $x$ is output (i.e., the set
 $V'$ at this point). Also recall the definition of $P_{i(x)}(x)$ as the already visited neighbors of $x$ at the iteration when $x$ is output. As all other vertices in the same component in $\chordalcomps_G(K)$ as $x$ have the same preceding neighbors, we will define $P_{i(H)}(H) := P_{i(x)}(x)$ for all $H \in \chordalcomps_G(K)$ and all $x \in H$. 

 \begin{lemma}
   \label{lemma:spminsep}
   Let $G$ be a chordal graph and $H \in
   \chordalcomps_G(K)$. Then, $P_{i(H)}(H)$ separates~$H$ from $W =
   \mathcal{X}_{\text{out}}(V) \setminus P_{i(H)}(H)$ and is a
   minimal separator of $G$.
 \end{lemma}
 \begin{proof}
   The set $P_{i(H)}(H)$ is a proper subset of all previously
   visited vertices (as $H$ is not part of the \emph{maximal} clique $K$ Algorithm~\ref{alg:cgk}
   starts with). Since $P_{i(H)}(H)$ contains all visited
   neighbors of $H$, it separates $H$ from $W$. To see this, assume for
   sake of contradiction that there is a path from $v \in V_H$ to $w \in
   W$ without a vertex in $P_{i(H)}(H)$. Consider the shortest such path and let
   $y$ be the first vertex with successor $z$ preceding it in
   the vertex ordering produced by Algorithm~\ref{alg:cgk} : 
   $v - \dots - x - y - z - \dots - w \in W$. Then
   $\{x,z\}\in E_G$, as the ordering is a reverse of a PEO. Hence, the path is
   not the shortest path and, thus, $y$ cannot
   exist. Since there can be no direct edge from $v$ to $w$, the set
   $P_{i(H)}(H)$ is indeed a separator.

   We prove that there is a vertex in $W$, which is a neighbor of
   all vertices in $P_{i(H)}(H)$. Consider the vertex in
   $P_{i(H)}(H)$, which is visited last (denoted by $p$). When
   vertex $p$ is processed, it has to have a neighbor $x \in W$, which was
   previously visited, else $p$ would be part of $H$. This is because the
   preceding neighbors would be identical to the ones of the vertices
   in $H$ (i.e., $\mathcal{P}_{i(H)}(H) \setminus \{p\}$),
   meaning that $p$ would have the same label. It would follow that
   either $p$ and the vertices in $H$ are appended to $L$ when $p$ is
   visited or were already appended to $L$ previously. In both cases, $p$ would
   be in $H$, which is a contradiction.

   Hence, such vertex $x$ has to exist. Moreover, $x$
   has to be connected to all vertices in $P_{i(H)}(H)$ because
   of the PEO property (all preceding neighbors of a vertex form a
   clique).

   From the first part of the proof, we know that $x$ and $y \in H$ are
   separated by $P_{i(H)}(H)$. As both $x$ and $y$ are fully
   connected to $P_{i(H)}(H)$, it follows that this set is also a \emph{minimal} $x-y$ separator.
 \end{proof}
 \begin{lemma} \label{lemma:flspbijection}
   Let $G$ be a chordal graph for which the number of AMOs is computed
   with the function $\texttt{count}$ in
   Algorithm~\ref{alg:cliquepicking}. Let $H$ be any chordal graph for
   which \texttt{count} is called in the recursion (for $H \neq G$).
   Then $H = F \setminus S$ for some $S$-flower $F$ in $G$ with $S \in \separators(G)$.
 \end{lemma}

 \begin{proof}
   Let $S_H$ be the union of all sets $P_{i(\tilde{G})}(\tilde{G})$ for
   $\tilde{G}$ on the recursive call stack from the input graph $G$ to
   currently considered subgraph $H$. We
   define $P_{i(G)}(G) = \emptyset$ for convenience.

   Recall that $H\neq G$. We show by induction that (i) $S_H$ is a minimal
   separator, (ii) $S_H$ is fully connected to $H$, and 
   (iii)~$H = F \setminus S_H$ for some $S_H$-flower $F$.

   In the base case, $H \in \chordalcomps_G(K)$.  By Lemma~\ref{lemma:spminsep},
   $S_H$ is a minimal separator in $G$, which is by definition
   connected to all vertices in $H$. Hence, as $H$ is connected, $H
   \subseteq F \setminus S_H$ holds for an $S_H$-flower $F$. We show the equality by
   contradiction. Assume there is a vertex $v \in F \setminus S$ but
   not in $H$. Then $v$ can neither be a vertex in $W$ nor the neighbor
   of a vertex in $W$, as by the definition of
   flowers this means that there is a path from $W$ to $H$ in $G[V
   \setminus S_H]$~--~this would violate that $H$ is separated from
   $W$ by $S_H$ (Lemma~\ref{lemma:spminsep}). Moreover, $v$ is a neighbor of all vertices in $S_H$.
   Hence, we have $P_{i(v)}(v) = P_{i(H)}(H) = S_H$ and $v \in V_H$. A
   contradiction.

   Assume \texttt{count} is called with a graph $H \in
   \chordalcomps_{G'}(K)$ for some graph $G'$ and $K\in\cliques(G')$. By induction hypothesis,
   we have that $S_{G'}$ is a minimal separator in $G$ and fully
   connected to $G'$. Moreover, $G' = F' \setminus S_{G'}$ for some
   $F'$-flower of $S_{G'}$. Now, $P_{i(H)}(H)$ is by
   Lemma~\ref{lemma:spminsep} a minimal separator in $G'$ for some vertices $x$
   and $y$. As $x$ and $y$ are connected to every vertex in $S_{G'}$,
   it follows that $S_H = S_{G'} \cup P_{i(H)}(H)$ is a minimal
   $x$-$y$ separator in $G$. Furthermore, $S_H$ is fully connected to $H$
   and it can be easily seen that $H \subseteq F \setminus S_H$. To
   show equality, observe that every vertex $v$ in $F \setminus S_H$
   is in $G'$ (if it is not separated from $H$ by $S_H$, it is
   clearly not separated from $H$ in $S_{G'}$). Thus, the same argument
   as in the base case applies and the statement follows.
 \end{proof}
\null

 \begin{proof}{(of Proposition \ref{prop:subpbound})}
   By Lemma~\ref{lemma:flspbijection}, it remains to bound the number
   of flowers in $G$. Each flower is associated with a minimal
   separator $S$ and there are at most $|\cliques(G)|-1$ such
   separators, as they are associated with the edges of the clique
   tree~\citep{Blair1993}.  Let $r$ (which is initially
   $|\cliques(G)|-1$) be an upper bound for the number of remaining
   separators. Now consider separator $S$. If $\bouquet(S)$ has $k$
   flowers, $S$ can be found on at least $k-1$ edges of the clique
   tree, namely the edges between the flowers (by
   Proposition~\ref{lemma:bouquetpart} the flowers partition the
   bouquet and, by the definition of flowers, the intersection of cliques
   from two $S$-flowers has to be a subset of $S$). Thus, we have at
   most $r - (k-1)$ remaining separators. The maximum number of flowers
   is obtained when the quotient $k / (k-1)$ is maximal. This is the
   case for $k = 2$. It follows that there are at most
   $2(|\cliques(G)| - 1)$ flowers.

   When bounding the number of explored UCCGs, we additionally take
   into account the input graph and obtain as bound
   $2(|\cliques(G)| - 1) + 1 = 2|\cliques(G)| - 1$.
 \end{proof}
\null

\vskip 0.2in

\clearpage
\acks{This work was supported by the Deutsche Forschungsgemeinschaft (DFG)
 grant LI634/4-2.

The authors would like to express their gratitude to Paula Arnold for fruitful discussions regarding
uniform sampling of AMOs and her help in the implementations and the
setup of the experiments. The authors also thank the anonymous
reviewers and the action editor for their constructive comments and
suggestions.}

\bibliography{countingmec}

\begin{thebibliography}{49}
\providecommand{\natexlab}[1]{#1}
\providecommand{\url}[1]{\texttt{#1}}
\expandafter\ifx\csname urlstyle\endcsname\relax
  \providecommand{\doi}[1]{doi: #1}\else
  \providecommand{\doi}{doi: \begingroup \urlstyle{rm}\Url}\fi

\bibitem[AhmadiTeshnizi et~al.(2020)AhmadiTeshnizi, Salehkaleybar, and Kiyavash]{Teshnizi20}
Ali AhmadiTeshnizi, Saber Salehkaleybar, and Negar Kiyavash.
\newblock Lazyiter: A fast algorithm for counting {Markov} equivalent {DAGs} and designing experiments.
\newblock In \emph{Proceedings of the 37th International Conference on Machine Learning, {ICML}~'20}, pages 125--133, 2020.

\bibitem[Andersson et~al.(1997)Andersson, Madigan, and Perlman]{Andersson1997}
Steen~A. Andersson, David Madigan, and Michael~D Perlman.
\newblock A characterization of {Markov} equivalence classes for acyclic digraphs.
\newblock \emph{The Annals of Statistics}, 25\penalty0 (2):\penalty0 505--541, 1997.

\bibitem[Berry et~al.(2009)Berry, Krueger, and Simonet]{berry2009maximal}
Anne Berry, Richard Krueger, and Genevieve Simonet.
\newblock Maximal label search algorithms to compute perfect and minimal elimination orderings.
\newblock \emph{SIAM Journal on Discrete Mathematics}, 23\penalty0 (1):\penalty0 428--446, 2009.

\bibitem[Blair and Peyton(1993)]{Blair1993}
Jean~RS Blair and Barry Peyton.
\newblock An introduction to chordal graphs and clique trees.
\newblock In \emph{Graph Theory and Sparse Matrix Computation}, pages 1--29. Springer, 1993.

\bibitem[Brightwell and Winkler(1991)]{brightwell1991counting}
Graham~R. Brightwell and Peter Winkler.
\newblock Counting linear extensions is {\#}{P}-complete.
\newblock In \emph{Proceedings of the 23th Annual {ACM} Symposium on Theory of Computing, {STOC}~'91}, pages 175--181, 1991.

\bibitem[Chickering(2002{\natexlab{a}})]{chickering2002learning}
David~Maxwell Chickering.
\newblock Learning equivalence classes of {Bayesian}-network structures.
\newblock \emph{Journal of Machine Learning Research}, 2:\penalty0 445--498, 2002{\natexlab{a}}.

\bibitem[Chickering(2002{\natexlab{b}})]{chickering2002optimal}
David~Maxwell Chickering.
\newblock Optimal structure identification with greedy search.
\newblock \emph{Journal of Machine Learning Research}, 3:\penalty0 507--554, 2002{\natexlab{b}}.

\bibitem[Corneil and Krueger(2008)]{corneil2008unified}
Derek~G Corneil and Richard~M Krueger.
\newblock A unified view of graph searching.
\newblock \emph{SIAM Journal on Discrete Mathematics}, 22\penalty0 (4):\penalty0 1259--1276, 2008.

\bibitem[Dirac(1961)]{dirac61}
Gabriel~A. Dirac.
\newblock On rigid circuit graphs.
\newblock \emph{Abhandlungen aus dem Mathematischen Seminar der Universit{\"a}t Hamburg}, 25\penalty0 (1):\penalty0 71--76, 1961.

\bibitem[Eberhardt(2008)]{Eberhardt08}
Frederick Eberhardt.
\newblock Almost optimal intervention sets for causal discovery.
\newblock In \emph{Proceedings of the 24th Conference on Uncertainty in Artificial Intelligence, {UAI'08}}, pages 161--168. {AUAI} Press, 2008.

\bibitem[Eberhardt et~al.(2005)Eberhardt, Glymour, and Scheines]{EberhardtGS05}
Frederick Eberhardt, Clark Glymour, and Richard Scheines.
\newblock On the number of experiments sufficient and in the worst case necessary to identify all causal relations among {N} variables.
\newblock In \emph{Proceedings of the 21st Conference on Uncertainty in Artificial Intelligence, {UAI'05}}, pages 178--184. {AUAI} Press, 2005.

\bibitem[Ganian et~al.(2020)Ganian, Hamm, and Talvitie]{Ganian2020}
Robert Ganian, Thekla Hamm, and Topi Talvitie.
\newblock An efficient algorithm for counting {Markov} equivalent {DAG}s.
\newblock In \emph{Proccedings of the 34th AAAI Conference on Artificial Intelligence, {AAAI'20}}, pages 10136--10143. AAAI Press, 2020.

\bibitem[Ganian et~al.(2022)Ganian, Hamm, and Talvitie]{ganian2022efficient}
Robert Ganian, Thekla Hamm, and Topi Talvitie.
\newblock An efficient algorithm for counting markov equivalent dags.
\newblock \emph{Artificial Intelligence}, 304:\penalty0 103648, 2022.

\bibitem[Ghassami et~al.(2018)Ghassami, Salehkaleybar, Kiyavash, and Bareinboim]{ghassami2018budgeted}
AmirEmad Ghassami, Saber Salehkaleybar, Negar Kiyavash, and Elias Bareinboim.
\newblock Budgeted experiment design for causal structure learning.
\newblock In \emph{Proceedings of the 35th International Conference on Machine Learning, {ICML}~'18}, pages 1719--1728, 2018.

\bibitem[Ghassami et~al.(2019)Ghassami, Salehkaleybar, Kiyavash, and Zhang]{Ghassami2019}
AmirEmad Ghassami, Saber Salehkaleybar, Negar Kiyavash, and Kun Zhang.
\newblock Counting and sampling from {Markov} equivalent {DAG}s using clique trees.
\newblock In \emph{Proccedings of the 33th AAAI Conference on Artificial Intelligence, {AAAI'19}}, pages 3664--3671. AAAI Press, 2019.

\bibitem[Gillispie and Perlman(2002)]{Gillispie2002}
Steven~B. Gillispie and Michael~D. Perlman.
\newblock The size distribution for {Markov} equivalence classes of acyclic digraph models.
\newblock \emph{Artificial Intelligence}, 141\penalty0 (1/2):\penalty0 137--155, 2002.

\bibitem[Greenewald et~al.(2019)Greenewald, Katz, Shanmugam, Magliacane, Kocaoglu, Adser{\`{a}}, and Bresler]{GreenewaldKSMKA19}
Kristjan~H. Greenewald, Dmitriy Katz, Karthikeyan Shanmugam, Sara Magliacane, Murat Kocaoglu, Enric~Boix Adser{\`{a}}, and Guy Bresler.
\newblock Sample efficient active learning of causal trees.
\newblock In \emph{Proceedings of the 32nd Conference on Neural Information Processing Systems, {NeurIPS'19}}, pages 14279--14289, 2019.

\bibitem[Hauser and B{\"{u}}hlmann(2012)]{hauser2012characterization}
Alain Hauser and Peter B{\"{u}}hlmann.
\newblock Characterization and greedy learning of interventional {Markov} equivalence classes of directed acyclic graphs.
\newblock \emph{Journal of Machine Learning Research}, 13:\penalty0 2409--2464, 2012.

\bibitem[Hauser and B{\"u}hlmann(2014)]{hauser2014two}
Alain Hauser and Peter B{\"u}hlmann.
\newblock Two optimal strategies for active learning of causal models from interventional data.
\newblock \emph{International Journal of Approximate Reasoning}, 55\penalty0 (4):\penalty0 926--939, 2014.

\bibitem[He and Geng(2008)]{He2008}
Yang-Bo He and Zhi Geng.
\newblock Active learning of causal networks with intervention experiments and optimal designs.
\newblock \emph{Journal of Machine Learning Research}, 9\penalty0 (Nov):\penalty0 2523--2547, 2008.

\bibitem[He and Yu(2016)]{he2016formulas}
Yangbo He and Bin Yu.
\newblock Formulas for counting the sizes of {Markov} equivalence classes of directed acyclic graphs.
\newblock \emph{arXiv}, abs/1610.07921, 2016.
\newblock URL \url{http://arxiv.org/abs/1610.07921}.

\bibitem[He et~al.(2015)He, Jia, and Yu]{He2015}
Yangbo He, Jinzhu Jia, and Bin Yu.
\newblock Counting and exploring sizes of {Markov} equivalence classes of directed acyclic graphs.
\newblock \emph{Journal of Machine Learning Research}, 16\penalty0 (79):\penalty0 2589--2609, 2015.

\bibitem[Heckerman et~al.(1995)Heckerman, Geiger, and Chickering]{heckerman1995learning}
David Heckerman, Dan Geiger, and David~Maxwell Chickering.
\newblock Learning {Bayesian} networks: The combination of knowledge and statistical data.
\newblock \emph{Machine Learning}, 20\penalty0 (3):\penalty0 197--243, 1995.

\bibitem[Kahn(1962)]{kahn1962topological}
Arthur~B Kahn.
\newblock Topological sorting of large networks.
\newblock \emph{Communications of the ACM}, 5\penalty0 (11):\penalty0 558--562, 1962.

\bibitem[Koller and Friedman(2009)]{koller2009probabilistic}
Daphne Koller and Nir Friedman.
\newblock \emph{Probabilistic Graphical Models - Principles and Techniques}.
\newblock {MIT} Press, 2009.
\newblock ISBN 978-0-262-01319-2.

\bibitem[Maathuis et~al.(2009)Maathuis, Kalisch, and B{\"u}hlmann]{maathuis2009estimating}
Marloes~H Maathuis, Markus Kalisch, and Peter B{\"u}hlmann.
\newblock Estimating high-dimensional intervention effects from observational data.
\newblock \emph{The Annals of Statistics}, 37\penalty0 (6A):\penalty0 3133--3164, 2009.

\bibitem[Meek(1995)]{Meek1995}
Christopher Meek.
\newblock Causal inference and causal explanation with background knowledge.
\newblock In \emph{Proceedings of the 11th Conference on Uncertainty in Artificial Intelligence, {UAI'95}}, pages 403--410, 1995.

\bibitem[Meek(1997)]{meek1997graphical}
Christopher Meek.
\newblock \emph{Graphical Models: Selecting Causal and Statistical Models}.
\newblock PhD thesis, Carnegie Mellon University, 1997.

\bibitem[{OEIS Foundation Inc.}(2022)]{oeis}
{OEIS Foundation Inc.}
\newblock The number of irreducible permutations. entry a003319 in the on-line encyclopedia of integer sequences, 2022.
\newblock URL \url{https://oeis.org/A003319}.

\bibitem[Pearl(2009)]{pearl2009causality}
Judea Pearl.
\newblock \emph{Causality}.
\newblock {Cambridge University Press}, 2009.
\newblock ISBN 978-0521895606.

\bibitem[Perkovi\'{c} et~al.(2017)Perkovi\'{c}, Textor, Kalisch, and Maathuis]{perkovic2017complete}
Emilija Perkovi\'{c}, Johannes Textor, Markus Kalisch, and Marloes~H. Maathuis.
\newblock Complete graphical characterization and construction of adjustment sets in {Markov} equivalence classes of ancestral graphs.
\newblock \emph{Journal of Machine Learning Research}, 18:\penalty0 220:1--220:62, 2017.

\bibitem[Rose et~al.(1976)Rose, Tarjan, and Lueker]{Rose1976}
Donald~J. Rose, Robert~Endre Tarjan, and George~S. Lueker.
\newblock Algorithmic aspects of vertex elimination on graphs.
\newblock \emph{{SIAM} Journal on Computing}, 5\penalty0 (2):\penalty0 266--283, 1976.

\bibitem[Scheinerman(1988)]{Scheinerman88}
Edward~R. Scheinerman.
\newblock Random interval graphs.
\newblock \emph{Combinatorica}, 8\penalty0 (4):\penalty0 357--371, 1988.

\bibitem[Seker et~al.(2017)Seker, Heggernes, Ekim, and Taskin]{SekerHET17}
Oylum Seker, Pinar Heggernes, T{\'{\i}}naz Ekim, and Z.~Caner Taskin.
\newblock Linear-time generation of random chordal graphs.
\newblock In \emph{Proccedings of the 10th International Conference on Algorithms and Complexity, {CIAC~17}}, volume 10236, pages 442--453, 2017.

\bibitem[Shanmugam et~al.(2015)Shanmugam, Kocaoglu, Dimakis, and Vishwanath]{shanmugam2015learning}
Karthikeyan Shanmugam, Murat Kocaoglu, Alexandros~G. Dimakis, and Sriram Vishwanath.
\newblock Learning causal graphs with small interventions.
\newblock In \emph{Processing of the 28th Conference on Neural Information Processing Systems, {NeurIPS'15}}, pages 3195--3203, 2015.

\bibitem[Shpitser and Pearl(2006)]{shpitser2006identification}
Ilya Shpitser and Judea Pearl.
\newblock Identification of joint interventional distributions in recursive semi-{Markovian} causal models.
\newblock In \emph{Proceedings of the 21st AAAI Conference on Artificial Intelligence, {AAAI'06}}, volume~21, pages 1219--1226. AAAI Press, 2006.

\bibitem[Shpitser et~al.(2010)Shpitser, VanderWeele, and Robins]{ShpitserVR2010}
Ilya Shpitser, Tyler VanderWeele, and James Robins.
\newblock On the validity of covariate adjustment for estimating causal effects.
\newblock In \emph{Proceedings of the 26th Conference on Uncertainty in Artificial Intelligence, {UAI'10}}, pages 527--536. {AUAI} Press, 2010.

\bibitem[Spirtes et~al.(2000)Spirtes, Glymour, and Scheines]{spirtes2000causation}
Peter Spirtes, Clark Glymour, and Richard Scheines.
\newblock \emph{Causation, Prediction, and Search, Second Edition}.
\newblock {MIT} Press, 2000.
\newblock ISBN 978-0-262-19440-2.

\bibitem[Squires et~al.(2020)Squires, Magliacane, Greenewald, Katz, Kocaoglu, and Shanmugam]{activelearningdct2020}
Chandler Squires, Sara Magliacane, Kristjan Greenewald, Dmitriy Katz, Murat Kocaoglu, and Karthikeyan Shanmugam.
\newblock Active structure learning of causal {DAG}s via directed clique trees.
\newblock In \emph{Proceedings of the 33rd Conference on Neural Information Pressing Systems, {NeurIPS'20}}, volume~33, pages 21500--21511, 2020.

\bibitem[Talvitie and Koivisto(2019)]{Talvitie2019}
Topi Talvitie and Mikko Koivisto.
\newblock Counting and sampling {Markov} equivalent directed acyclic graphs.
\newblock In \emph{Proccedings of the 33th AAAI Conference on Artificial Intelligence, {AAAI'19}}, pages 7984--7991. AAAI Press, 2019.

\bibitem[Tarjan and Yannakakis(1984)]{tarjan1984simple}
Robert~E Tarjan and Mihalis Yannakakis.
\newblock Simple linear-time algorithms to test chordality of graphs, test acyclicity of hypergraphs, and selectively reduce acyclic hypergraphs.
\newblock \emph{SIAM Journal on computing}, 13\penalty0 (3):\penalty0 566--579, 1984.

\bibitem[van~der Zander and Li\'{s}kiewicz(2016)]{van2016separators}
Benito van~der Zander and Maciej Li\'{s}kiewicz.
\newblock Separators and adjustment sets in {Markov} equivalent {DAG}s.
\newblock In \emph{Proceedings of the 30th AAAI Conference on Artificial Intelligence, {AAAI}'16}, pages 3315--3321. AAAI Press, 2016.

\bibitem[van~der Zander et~al.(2019)van~der Zander, Li\'{s}kiewicz, and Textor]{van2019separators}
Benito van~der Zander, Maciej Li\'{s}kiewicz, and Johannes Textor.
\newblock Separators and adjustment sets in causal graphs: Complete criteria and an algorithmic framework.
\newblock \emph{Artificial Intelligence}, 270:\penalty0 1--40, 2019.

\bibitem[Verma and Pearl(1990)]{verma1990equivalence}
Thomas Verma and Judea Pearl.
\newblock Equivalence and synthesis of causal models.
\newblock In \emph{Proceedings of the 6th Conference on Uncertainty in Artificial Intelligence, {UAI'90}}, pages 255--270, 1990.

\bibitem[Verma and Pearl(1992)]{verma1992algorithm}
Thomas Verma and Judea Pearl.
\newblock An algorithm for deciding if a set of observed independencies has a causal explanation.
\newblock In \emph{Proceedings of the 8th Conference on Uncertainty in Artificial Intelligence, {UAI'92}}, pages 323--330, 1992.

\bibitem[Vose(1991)]{vose1991linear}
Michael~D Vose.
\newblock A linear algorithm for generating random numbers with a given distribution.
\newblock \emph{IEEE Transactions on software engineering}, 17\penalty0 (9):\penalty0 972--975, 1991.

\bibitem[Walker(1974)]{walker1974new}
Alastair~J Walker.
\newblock New fast method for generating discrete random numbers with arbitrary frequency distributions.
\newblock \emph{Electronics Letters}, 10\penalty0 (8):\penalty0 127--128, 1974.

\bibitem[Wien{\"o}bst et~al.(2021{\natexlab{a}})Wien{\"o}bst, Bannach, and Li\'{s}kiewicz]{WieobstExtendability2021}
Marcel Wien{\"o}bst, Max Bannach, and Maciej Li\'{s}kiewicz.
\newblock Extendability of causal graphical models: Algorithms and computational complexity.
\newblock In \emph{Proceedings of the 37th Conference on Uncertainty in Artificial Intelligence, {UAI'21}}. {AUAI} Press, 2021{\natexlab{a}}.

\bibitem[Wien{\"o}bst et~al.(2021{\natexlab{b}})Wien{\"o}bst, Bannach, and Li\'{s}kiewicz]{wienobst2021counting}
Marcel Wien{\"o}bst, Max Bannach, and Maciej Li\'{s}kiewicz.
\newblock Polynomial-time algorithms for counting and sampling {Markov} equivalent {DAG}s.
\newblock In \emph{Proceedings of the 35th AAAI Conference on Artificial Intelligence, {AAAI'21}}, pages 12198--12206. {AAAI} Press, 2021{\natexlab{b}}.

\end{thebibliography}

\end{document}